\documentclass{article}

\PassOptionsToPackage{numbers, compress}{natbib}

\usepackage[main, preprint]{neurips_2026}

\usepackage[utf8]{inputenc}
\usepackage[T1]{fontenc}

\usepackage{float}

\usepackage{xcolor}
\usepackage[
    colorlinks=true,
    linkcolor=blue!55!black,
    citecolor=green!45!black,
    urlcolor=purple!60!black,
    filecolor=magenta!60!black,
    bookmarks=true,
    bookmarksopen=true,
    bookmarksnumbered=true,
    unicode=true
]{hyperref}

\usepackage{url}
\usepackage{booktabs}     
\usepackage{amsmath,amssymb,amsfonts}
\usepackage{dsfont}       
\usepackage{amsthm}       
\usepackage{mathtools}    
\usepackage{nicefrac}     
\usepackage{microtype}    

\theoremstyle{plain}
\newtheorem{theorem}{Theorem}[section]
\newtheorem{lemma}[theorem]{Lemma}

\newtheorem{corollary}[theorem]{Corollary}

\theoremstyle{definition}
\newtheorem{definition}[theorem]{Definition}

\theoremstyle{remark}

\usepackage{graphicx}

\usepackage{tabularx}
\usepackage{multirow}

\usepackage[ruled,vlined]{algorithm2e}

\usepackage[acronym, nonumberlist]{glossaries}
\glsdisablehyper

\newcommand{\codeurl}{\url{https://github.com/LouisTier/SOCP}}

\makeglossaries

\newacronym{cp}{CP}{conformal prediction}
\newacronym{scp}{SCP}{split conformal prediction}

\newacronym{som}{SOM}{Self-Organizing Map}
\newacronym{bmu}{BMU}{best-matching unit}
\newacronym{cdf}{CDF}{cumulative distribution function}
\newacronym{ks}{KS}{Kolmogorov-Smirnov}

\newacronym{gbr}{GBR}{Gradient Boosting Regressor}
\newacronym{pca}{PCA}{principal component analysis}
\newacronym{knn}{KNN}{k-nearest neighbors}
\newacronym{sgd}{SGD}{stochastic gradient descent}

\newacronym{lac}{LAC}{least ambiguous set-valued classifier}
\newacronym{aps}{APS}{adaptive prediction sets}
\newacronym{raps}{RAPS}{regularized adaptive prediction sets}
\newacronym{saps}{SAPS}{sorted adaptive prediction sets}

\newacronym{cqr}{CQR}{conformalized quantile regression}
\newacronym{chr}{CHR}{conformal histogram regression}
\newacronym{nfcp}{NFCP}{normalizing-flow conformal prediction}
\newacronym{rcp}{RCP}{rectified conformal prediction}

\newacronym{lcp}{LCP}{localized conformal prediction}
\newacronym{rlcp}{RLCP}{randomly localized conformal prediction}
\newacronym{socp}{SOCP}{Self-Organized Conformal Prediction}

\newacronym{wsc}{WSC}{worst-slab coverage}
\newacronym{ssc}{SSC}{size-stratified coverage}
\newacronym{ert}{ERT}{excess risk of the target coverage}

\title{Self-Organized Conformal Prediction:\\Reducing Regional Coverage Gaps with Unsupervised Group Discovery}

\author{%
  \textbf{Louis Berthier}$^{1,2,}$\thanks{Corresponding author: \texttt{louis.berthier@polytechnique.edu}}\hphantom{$^{*}$}\quad
  \textbf{Ahmed Shokry}$^{1}$\quad
  \textbf{Maxime Moreaud}$^{2}$\quad\and
  \textbf{Guillaume Ramelet}$^{2}$\quad
  \textbf{Aymeric Dieuleveut}$^{1}$ \\[6pt]
  $^{1}$Centre de Mathématiques Appliquées, Ecole Polytechnique, Palaiseau, France\quad \\
  $^{2}$Manufacture Française des Pneumatiques Michelin, Clermont-Ferrand, France
}

\begin{document}

\maketitle
\setcounter{footnote}{0}

\begin{abstract}
  Conformal prediction guarantees marginal coverage, but a pooled calibration quantile can hide systematic undercoverage across heterogeneous regions of the feature space.
  We introduce \glsdisp{socp}{Self-Organized Conformal Prediction (SOCP)}, a calibration scheme that discovers input-space groups with an unsupervised \glsdisp{som}{Self-Organizing Map (SOM)} trained without calibration data.
  At prediction time, the query's \glsdisp{bmu}{best-matching unit (BMU)} retrieves a calibration buffer from one cell, a fixed grid neighborhood, or a prototype enlargement (Regimes~1 to~3).
  When fixed neighborhoods are too sparse, Regime~3 adds cells by prototype distance, using a global budget selected from training cell occupancies and the planned calibration size before any calibration score is observed.
  The predictor and nonconformity score remain unchanged.
  Cell-only retrieval has exact cell-conditional validity, and each fixed union of cells has exact retrieved-set validity.
  Interpreting a neighborhood threshold at its central cell incurs an explicit \glsdisp{ks}{Kolmogorov-Smirnov (KS)} bias term.
  Across ten regression and classification benchmarks, \gls{socp} lowers the mean weighted coverage gap (WCovGap) relative to pooled split conformal prediction (SCP) on all ten datasets.
  The mean relative change is $-14.8\%$, with a mean increase of $3.0\%$ in output size.
  Under fixed-neighborhood retrieval, SO calibration lowers mean WCovGap across ten seeds in $43$ of $50$ score and dataset comparisons, while SO-SCP lowers it for every dataset at all three tested external partition granularities.
  These results provide local group calibration without supervised partitions or predictor retraining, while the diagnostics make the cost and limits of locality explicit.
  The code is available at \codeurl.
\end{abstract}

\section{Introduction}
\label{sec:introduction}

Marginally valid prediction sets can still be unreliable in the regions that matter.
A method may attain its target coverage on average while undercovering rare labels, sparse operating regimes, or subpopulations with larger prediction errors.
Pooled conformal calibration cannot reveal this imbalance because one global quantile averages score distributions across the feature space.
The same correction can be too small in difficult regions and unnecessarily large in easier ones.

\glsdisp{cp}{Conformal prediction (CP)} turns any fixed nonconformity score into prediction sets with distribution-free finite-sample marginal coverage under exchangeability~\citep{vovkAlgorithmicLearningRandom2005a,shaferTutorialConformalPrediction2008,angelopoulosGentleIntroductionConformal2022b}.
For a prediction set $C_\alpha$, this guarantee is
\[
      \mathbb{P}\{Y_{n+1}\in C_\alpha(X_{n+1})\}\geq 1-\alpha .
\]
It controls coverage after averaging over the feature distribution, so despite being model-agnostic, the guarantee remains only marginally valid.
The pointwise conditional target
\[
      \mathbb{P}\{Y_{n+1}\in C_\alpha(X_{n+1})\mid X_{n+1}=x\}\geq 1-\alpha
\]
is stronger, but exact distribution-free coverage of this form is impossible without essentially uninformative sets~\citep{vovkConditionalValidityInductive2012,leiDistributionfreePredictionBands2014,barberLimitsDistributionfreeConditional2020a}.
A practical method must choose a locality that is coarse enough to calibrate with finite data and fine enough to expose and reduce regional failures.

We introduce \emph{\glsdisp{socp}{Self-Organized Conformal Prediction (SOCP)}}, a conformal calibration scheme that learns this locality from the input geometry.
A \glsdisp{som}{Self-Organizing Map (SOM)}~\citep{kohonenSelfOrganizingMaps2001a} maps a fixed input representation to a finite grid of prototype cells.
The map is trained on training inputs before calibration, and every calibration point is assigned to its \glsdisp{bmu}{best-matching unit (BMU)}.
For a test point, \gls{socp} retrieves the score calibration buffer from
(i) the query's \gls{bmu} cell (Regime~1),
(ii) a fixed neighborhood on the \gls{som} grid (Regime~2), or
(iii) a fixed prototype enlargement of that neighborhood (Regime~3).
The conformal threshold is computed from this buffer rather than from the full calibration set.
Only calibration is localized, so the predictor and the score remain unchanged.

The three regimes expose the bias and variance of locality.
Retrieval from one cell gives exact cell-conditional validity, but a small cell can force an infinite threshold.
A fixed neighborhood pools more calibration scores and is exactly valid for each fixed retrieved union of cells.
When that threshold is interpreted at the central query cell, its coverage loss is bounded by a \glsdisp{ks}{Kolmogorov-Smirnov (KS)} bias between the score distributions of the central cell and the retrieved cells.
When a fixed neighborhood is too sparse, Regime~3 augments it with cells ordered by prototype distance.
The method automatically selects a single global enlargement budget from the training \gls{bmu} occupancies and the planned calibration size, then freezes it before observing any calibration score.
The regime itself remains a user choice.

\textbf{Our contributions are as follows.}
\begin{enumerate}
      \item We introduce \gls{socp}, a model- and score-agnostic local calibration procedure built on a \gls{som}, which discovers groups from unlabeled input geometry and, in all three regimes, assigns each query the calibration buffer of a cell set fixed before calibration.
            With $K$ \gls{som} cells and fixed representation dimension, localization costs $O(K)$ per query, versus $O(n)$ for kernel localizers with $n$ calibration points.
      \item We prove finite-sample guarantees for the three retrieval regimes: cell-only, fixed-neighborhood, and prototype-enlarged.
            Exact validity holds on the group used for calibration, while validity at the central cell pays an explicit \gls{ks} bias.
            Also, we formalize both the training-routed enlargement used in the experiments and an optional split-routed target-buffer extension.
      \item We benchmark \gls{socp} on five regression and five classification datasets, pairing SO calibration with five nonconformity scores per task and comparing it with up to six alternative calibration rules.
            On external $k$-means partitions unseen during calibration, SO-SCP lowers the mean weighted coverage gap (WCovGap) on all ten datasets, averaging a $14.8\%$ reduction for a $3.0\%$ increase in output size.
            Under fixed-neighborhood retrieval, SO calibration lowers mean WCovGap in $43$ of $50$ score and dataset comparisons, and SO-SCP lowers it in all $30$ external audits with $10$, $25$, and $50$ $k$-means clusters.
      \item We provide cellwise diagnostics for coverage, size, occupancy, calibration buffers, and score distributions.
\end{enumerate}

\section{Related work}
\label{sec:related_work}

\paragraph{From marginal validity to local reliability.}
\glsdisp{cp}{Conformal prediction (CP)} constructs prediction sets with distribution-free coverage guarantees under exchangeability~\citep{vovkAlgorithmicLearningRandom2005a,shaferTutorialConformalPrediction2008,angelopoulosGentleIntroductionConformal2022b}.
Its split form separates predictor fitting from calibration and works with modern black-box models~\citep{papadopoulosInductiveConfidenceMachines2002a,leiDistributionFreePredictiveInference2017}.
A pooled calibration quantile gives finite-sample marginal coverage, but coverage can still vary substantially across the feature space.
Exact pointwise conditional coverage would provide stronger local control, but cannot be obtained distribution-free without uninformative sets~\citep{vovkConditionalValidityInductive2012,leiDistributionfreePredictionBands2014,barberLimitsDistributionfreeConditional2020a}.
Methods between these endpoints typically modify
(i) the score,
(ii) the calibration distribution, or
(iii) the groups on which calibration is performed.
\gls{socp} follows the third route.
We compose it with scores from the first route and compare it against calibration baselines from the second and third routes.

\paragraph{Changing the score.}
Changing the nonconformity score can make the resulting prediction sets adaptive to the input.
For regression, \glsdisp{cqr}{conformalized quantile regression (CQR)} uses lower and upper conditional quantile estimates to adapt interval location and width~\citep{romanoConformalizedQuantileRegression2019b}, while \gls{chr} estimates the conditional response distribution with a histogram~\citep{sesiaConformalHistogramRegression2021}.
\Gls{nfcp} learns a transformation of the absolute residual~\citep{colomboNormalizingFlowsConformal2024}, while \gls{rcp} rescales a base score using an estimate of its conditional $(1-\alpha)$-quantile~\citep{plassierRectifyingConformityScores2025}.
For classification, \gls{lac} ranks labels through $1-\hat p_y(x)$, following the least-ambiguity principle~\citep{sadinleLeastAmbiguousSetvalued2019}.
\Gls{aps} accumulate ranked class probabilities~\citep{romanoClassificationValidAdaptive2020}, \gls{raps} add a penalty beyond a chosen label rank~\citep{angelopoulosUncertaintySetsImage2021}, and \gls{saps} retain only the largest class probability together with label rank~\citep{huangConformalPredictionDeep2024}.
The rank-only ablation removes probability values and retains only randomized label rank~\citep{huangConformalPredictionDeep2024}.
These scores are complementary to \gls{socp}, which keeps the chosen score fixed and changes only the calibration buffer used to estimate its threshold.

\paragraph{Changing the calibration distribution.}
Localized conformal methods weight calibration observations according to their proximity to the query.
\citet{guanLocalizedConformalPrediction2022a} proposes \gls{lcp}, which recalibrates the level of a kernel-weighted score distribution to preserve finite-sample marginal validity, with additional local guarantees under regularity conditions.
This flexibility requires a bandwidth and $O(n)$ kernel evaluations for each query, with $n$ calibration points.
\citet{horeConformalPredictionLocal2023} explain why naive kernel weighting around the query can fail to preserve marginal coverage.
Their \gls{rlcp} construction samples a random anchor around the query, which provides finite-sample marginal validity and relaxed guarantees for locality and covariate shift.
\citet{hanSplitLocalizedConformal2023} adjust each score by subtracting a kernel estimate of its conditional quantile, then apply split-conformal calibration.
In contrast, \gls{socp} uses no kernel weights at test time.
It retrieves with a cost of $O(K)$ from a finite map whose cells, topological neighborhoods, and prototype order are fixed before calibration.

\paragraph{Changing the groups.}
Mondrian \gls{cp} calibrates separately within a fixed partition and gives finite-sample validity conditional on each populated group~\citep{vovkMondrianConfidenceMachine2003}.
Its main difficulty is granularity, since small groups yield noisy or infinite cutoffs.
\citet{dingClassConditionalConformalPrediction2023a} address this problem when the desired groups are classes.
They cluster classes with similar score behavior, obtain exact validity at cluster level, and control the transfer to individual classes through differences in score distributions within clusters.
We apply an analogous \gls{ks} bridge to cells learned from inputs by an unsupervised \gls{som}.
Other work formulates conditional coverage through covariate shifts or richer collections of groups~\citep{gibbsConformalPredictionConditional2024a,bairaktariKandinskyConformalPrediction2025}, or studies covariate shifts determined by groups~\citep{bhattacharyyaGroupweightedConformalPrediction2026}.
\gls{socp} makes a different restriction: an unsupervised map discovers a finite partition of the input space, and calibration borrows only along neighborhoods and prototype orders fixed by that map.

\paragraph{Evaluating conditional reliability.}
Conditional coverage diagnostics measure behavior that marginal coverage alone cannot identify.
The coverage gap (CovGap)~\citep{dingClassConditionalConformalPrediction2023a} and WCovGap~\citep{braunConditionalCoverageDiagnostics2025} summarize absolute coverage deviations over a partition, with WCovGap weighting each group by test mass.
They are diagnostics rather than validity guarantees.
Also, evaluation on a method's own partition could favor it by design.
As a result, we evaluate every calibrator on the same external partition of $25$ $k$-means clusters, fitted on training data, and also compare \gls{socp} with pooled calibration at $10$ and $50$ clusters to test granularity.
Additionally, we report three diagnostics that need no feature-space partition: \gls{wsc} error~\citep{cauchoisKnowingWhatYou2021}, the \gls{ssc} gap~\citep{angelopoulosUncertaintySetsImage2021}, and \gls{ert}~\citep{braunConditionalCoverageDiagnostics2025}.

\paragraph{\gls{som} and local exchangeability.}
\glspl{som} map high-dimensional observations to a low-dimensional grid of prototype cells~\citep{kohonenSelforganizedFormationTopologically1982a,kohonenSelforganizingMap1990,kohonenSelfOrganizingMaps2001a}.
The \gls{bmu} map gives a hard partition, while the grid provides a fixed neighborhood relation for borrowing nearby scores.
This geometry resembles the motivation behind local exchangeability~\citep{campbellLocalExchangeability2022}, but our guarantees do not assume that observations are locally exchangeable.
Instead, we obtain exact validity for each fixed retrieved neighborhood.
The corresponding central-cell guarantee includes a \gls{ks} term that bounds the difference between the central cell's score distribution and the neighborhood mixture.

\section{Self-Organized Conformal Prediction}
\label{sec:methodology}

\glsdisp{socp}{Self-Organized Conformal Prediction (SOCP)} is a group-based calibration procedure for a pre-trained predictor and a fixed nonconformity score.
An unsupervised \gls{som}, trained on training inputs in the fixed representation, partitions the feature space into $K$ prototype cells.
For a test point, its \gls{bmu} retrieves a calibration buffer from
(i) one cell,
(ii) a fixed grid neighborhood, or
(iii) a prototype enlargement based on training data.
The predictor and score remain unchanged.
The conformal threshold is computed from the retrieved buffer rather than the full calibration set.
Figure~\ref{fig:socp-pipeline} summarizes the workflow, and the rest of this section states what each retrieval regime guarantees.

\begin{figure}[ht]
  \centering
  \includegraphics[width=0.99\textwidth]{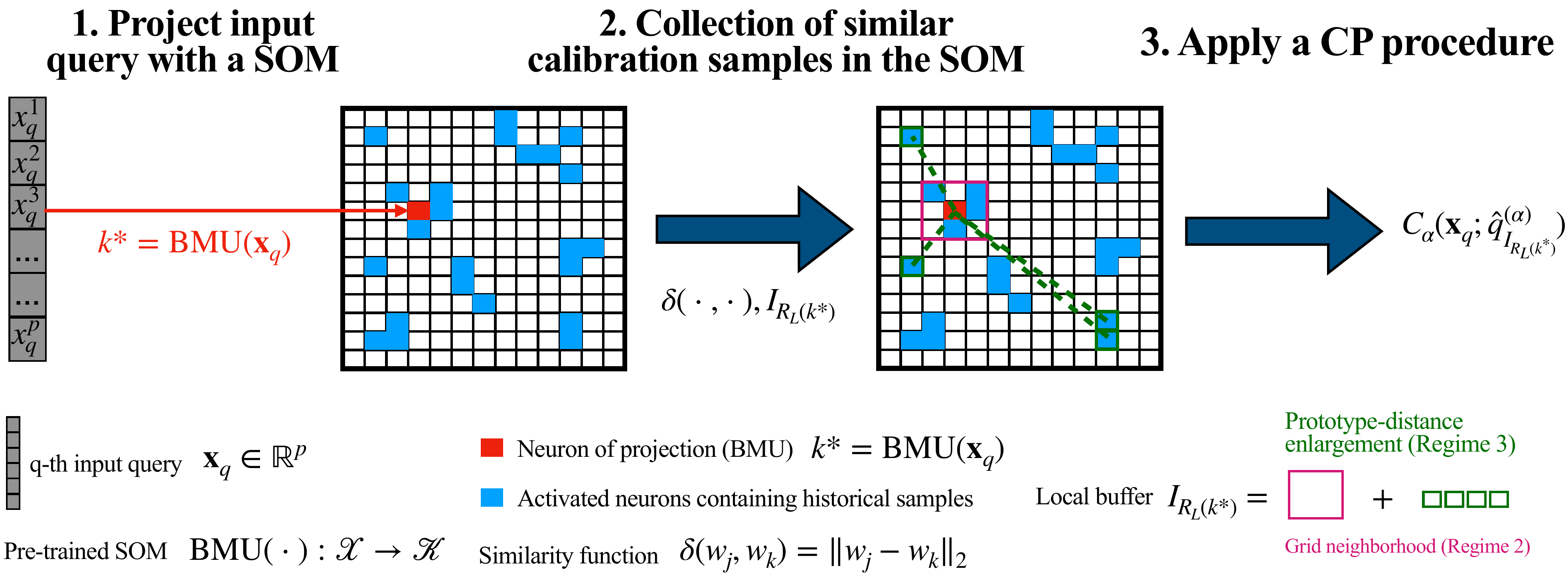}
  \caption{
    Overview of the \gls{socp} pipeline.
    A pre-trained \gls{som} maps a query $x_q$ to its \gls{bmu} $k^\ast$.
    From this cell, \gls{socp} forms a calibration buffer from the \gls{bmu} cell (Regime~1), a fixed grid neighborhood (Regime~2), or a training-routed enlargement that adds cells in prototype-distance order (Regime~3).
    The split-conformal quantile of the retrieved scores determines the prediction set.
  }
  \label{fig:socp-pipeline}
\end{figure}

\subsection{Setup and standing assumptions}
\label{sec:methodology:setup}

Let $\mathcal{X} \subseteq \mathbb{R}^p$ denote the representation space used by the score and the \gls{som}, and let $\mathcal{Y}$ be the response space.
Let $\mathcal{D}_{\mathrm{tr}}$ and $\mathcal{D}_{\mathrm{sel}}$ denote the disjoint training and selection splits, and write $\mathcal{D}_{\mathrm{pre}}\coloneqq(\mathcal{D}_{\mathrm{tr}},\mathcal{D}_{\mathrm{sel}})$.
In the experiments, $\mathcal{D}_{\mathrm{sel}}$ serves as a validation set whenever a score or baseline requires hyperparameter selection.
\gls{socp} itself does not require this split.
Before calibration, we fix the predictor, representation, nonconformity score $s:\mathcal{X}\times\mathcal{Y}\to\mathbb{R}$, and all retrieval choices.
Any required fitting or selection uses only $\mathcal{D}_{\mathrm{pre}}$.
In particular, the inputs in $\mathcal{D}_{\mathrm{tr}}$ train the \gls{som}, whose prototypes $\{w_k\}_{k=1}^K$ are indexed by $\mathcal{K}\coloneqq\{1,\dots,K\}$.
Let $\mathcal{D}_{\mathrm{cal}}\coloneqq\{(X_i,Y_i)\}_{i=1}^n$ be the calibration split, and let $(X_{n+1},Y_{n+1})$ be one test point.
Conditional on $\mathcal{D}_{\mathrm{pre}}$, we assume that the calibration points and the test point are i.i.d.\ from a common distribution $P$ on $\mathcal{X}\times\mathcal{Y}$.
All probabilities below are conditional on $\mathcal{D}_{\mathrm{pre}}$, so the predictor, score, map, and retrieval choices are fixed throughout the analysis.
Fix a miscoverage level $\alpha\in(0,1)$.

We introduce and recall the following definitions:
\begin{definition}
  \label{def:conformal-q}
  Following \citet{papadopoulosInductiveConfidenceMachines2002a, leiDistributionFreePredictiveInference2017}, let $S$ be a finite calibration index set, with $m\coloneqq|S|$ and $\ell_m\coloneqq\lceil(1-\alpha)(m+1)\rceil$.
  If $\ell_m\leq m$, define $\hat q^{(\alpha)}_S$ as the $\ell_m$-th order statistic of the scores $\{s(X_i,Y_i):i\in S\}$.
  If $\ell_m>m$, define $\hat q^{(\alpha)}_S\coloneqq+\infty$.
  Equivalently, this happens when $m<\lceil1/\alpha\rceil-1$.
  The associated prediction set is
  \[
    C_\alpha(x;q)\coloneqq\{y\in\mathcal{Y}:s(x,y)\leq q\}.
  \]
\end{definition}

\begin{definition}
  \label{def:bmu}
  Following \citet{kohonenSelforganizingMap1990,kohonenSelfOrganizingMaps2001a}, the best-matching unit map is defined as
  \[
    \mathrm{BMU}(x)\coloneqq\arg\min_{k\in\mathcal{K}}\|x-w_k\|_2.
  \]
  Ties are resolved by a deterministic rule fixed with the \gls{som}, so $\mathrm{BMU}$ is a measurable map fixed before calibration.
  Let $\gamma:\mathcal{K}\to\mathbb{Z}^2$ denote the fixed map from cell index to grid coordinate, and define the grid metric $d_{\mathrm{grid}}(k,j)\coloneqq\|\gamma(k)-\gamma(j)\|_\infty$ on a rectangular grid (Chebyshev distance).
  For a fixed integer radius $r\geq0$, we define the neighborhood of cell $k$ as
  \[
    N(k,r)\coloneqq\{j\in\mathcal{K}: d_{\mathrm{grid}}(k,j)\leq r\}.
  \]
\end{definition}

Throughout the regime statements, $k^\ast\coloneqq\mathrm{BMU}(X_{n+1})$ denotes the test point's BMU.
The BMU map is deterministic, but $k^\ast$ is random because $X_{n+1}$ is random.
Conditioning on $\{k^\ast=k\}$ or $\{k^\ast\in A\}$ restricts the query to a specific cell or cell set.
The deployed threshold uses $k^\ast$ to choose a buffer, so exact rank arguments are stated for fixed reference thresholds and fixed retrieved cell sets.
Then, central-cell statements specialize these thresholds to $\{k^\ast=k\}$.
Throughout, $\hat Q$ denotes a fixed reference threshold (parameterized by a fixed cell or cell set), and $\hat q$ the deployed threshold evaluated at $k^\ast$.

With those neighborhoods being defined, we now define cell masses, calibration buffers, and conditional cell weights:
\begin{definition}
  \label{def:cell-mass}
  For each cell $k\in\mathcal{K}$, define its population mass and calibration index set by
  \[
    \pi_k\coloneqq\mathbb{P}(\mathrm{BMU}(X)=k),
    \qquad
    \mathcal{I}_k\coloneqq\{i\in\{1,\dots,n\}: \mathrm{BMU}(X_i)=k\}.
  \]
  The active-cell set is defined as
  \[
    \mathcal{K}_+\coloneqq\{k\in\mathcal{K}:\pi_k>0\}.
  \]
  Moreover, a cell is called active when it belongs to $\mathcal{K}_+$.

  For any non-empty cell set $A\subseteq\mathcal{K}$, define
  \[
    \pi_A\coloneqq\mathbb{P}(\mathrm{BMU}(X)\in A)=\sum_{j\in A}\pi_j,
    \qquad
    \mathcal{I}_A\coloneqq\bigsqcup_{j\in A}\mathcal{I}_j.
  \]
  The union defining $\mathcal{I}_A$ collects the calibration sample indices assigned to cells in $A$ and is disjoint because the \gls{bmu} map assigns each sample to exactly one cell.
  Whenever $\pi_A>0$, for each $j\in A$, define
  \[
    \tilde\pi_j^A\coloneqq\frac{\pi_j}{\pi_A}
    =
    \mathbb{P}(\mathrm{BMU}(X)=j\mid \mathrm{BMU}(X)\in A).
  \]
  Thus $\tilde\pi_j^A$ is the conditional mass of cell $j$ within the cell set $A$.
  We set $\tilde\pi_j^A=0$ for inactive $j\in A$.
\end{definition}

Combining Definitions~\ref{def:bmu} and~\ref{def:cell-mass}, we introduce the score distribution within each cell and retrieved set:
\begin{definition}
  \label{def:cell-cdf}
  We write the calibration and test scores as
  \[
    s_i\coloneqq s(X_i,Y_i),\quad s_{n+1}\coloneqq s(X_{n+1},Y_{n+1}).
  \]
  For each active cell $k\in\mathcal{K}_+$, we define the conditional score \gls{cdf} as
  \[
    F^{(k)}(t)\coloneqq\mathbb{P}(s(X,Y)\leq t\mid \mathrm{BMU}(X)=k),
    \qquad t\in\mathbb{R},
  \]
  and more generally, for a cell set $A$ with $\pi_A>0$, we define the mixture \gls{cdf} as
  \[
    F^A(t)\coloneqq\sum_{j\in A\cap\mathcal{K}_+}\tilde\pi_j^A F^{(j)}(t)
    =\mathbb{P}(s(X,Y)\leq t\mid \mathrm{BMU}(X)\in A).
  \]
\end{definition}

Finally, we introduce the following \gls{ks} pseudometric, which measures the discrepancy between score distributions across cells:
\begin{definition}
  \label{def:ks-bias}
  The pairwise pseudometric is the Kolmogorov-Smirnov distance between the conditional score \glspl{cdf} of two active cells $k$ and $j$:
  \[
    d_{\mathrm{KS}}(k,j)\coloneqq\sup_{t\in\mathbb{R}}\left|F^{(k)}(t)-F^{(j)}(t)\right|.
  \]

  For an active cell $k\in A$, define the \gls{ks} bias as the weighted sum of the pairwise pseudometrics over all active cells in $A$:
  \[
    \varepsilon_k(A)\coloneqq\sum_{j\in A\cap\mathcal{K}_+}\tilde\pi_j^A d_{\mathrm{KS}}(k,j).
  \]
  By abuse of notation, when no cell set $A$ is specified, $\varepsilon_k$ denotes the neighborhood bias $\varepsilon_k(N(k,r))$.
\end{definition}

Every guarantee below applies to any measurable score fixed before calibration.

The proofs use two standard ingredients from Appendix~\ref{app:deferred_methodology}: a Mondrian split-conformal rank argument for fixed-reference cells (Lemma~\ref{lem:mondrian}) and a \gls{ks} bridge from retrieved-set to central-cell coverage (Lemma~\ref{lem:bridge}).
A tower-property lemma (Lemma~\ref{lem:tower-transfer}) combines them for every proof of approximate validity conditional on a cell.
Proofs of all lemmas, theorems, and corollaries are given in Appendix~\ref{app:theory_proofs}.

\subsection{Regime 1: BMU cell only}
\label{sec:methodology:regime1}

Regime~1 retrieves only calibration points in the query's \gls{bmu} cell.
Its test-point threshold is
\[
  \hat q^{(1)}\coloneqq\hat q^{(\alpha)}_{\mathcal{I}_{k^\ast}}.
\]

\begin{theorem}[Finite-sample cell-conditional validity, Regime~1]
  \label{thm:regime1}
  Under the standing assumptions, for every active cell $k\in\mathcal{K}_+$,
  \[
    \mathbb{P}\!\left(Y_{n+1}\in C_\alpha(X_{n+1};\hat q^{(1)})\mid \mathrm{BMU}(X_{n+1})=k\right)
    \geq 1-\alpha.
  \]
\end{theorem}

Regime~1 has no distributional bias because the calibration and test scores are conditioned on the same \gls{bmu} cell.
It can nevertheless be uninformative when $|\mathcal{I}_{k^\ast}|<\lceil1/\alpha\rceil-1$, because then $\hat q^{(1)}=+\infty$.

\subsection{Regime 2: Fixed topological neighborhood}
\label{sec:methodology:regime2}

Regime~2 retrieves all calibration points whose \gls{bmu} lies in the fixed grid neighborhood of the query cell.
For a fixed-reference cell $k$, define the neighborhood threshold
\[
  \hat Q_k^{(2)}\coloneqq\hat q^{(\alpha)}_{\mathcal{I}_{N(k,r)}}.
\]
For the test point, Regime~2 uses
\[
  \hat q^{(2)}\coloneqq\hat Q_{k^\ast}^{(2)}.
\]

\begin{theorem}[Fixed-reference neighborhood-conditional validity]
  \label{thm:regime2a}
  Under the standing assumptions, fix a cell $k$ with $\pi_{N(k,r)}>0$.
  Then
  \[
    \mathbb{P}\!\left(s_{n+1}\leq \hat Q_k^{(2)}\mid \mathrm{BMU}(X_{n+1})\in N(k,r)\right)
    \geq 1-\alpha.
  \]
\end{theorem}

Theorem~\ref{thm:regime2a} treats each fixed neighborhood as a single Mondrian group.
For $r>0$, the family $\{N(k,r):k\in\mathcal{K}\}$ is generally an overlapping cover.
Every coverage statement is indexed by a fixed reference cell, so overlap between buffers does not enter that statement.
At deployment, the procedure is queried at $k^\ast=\mathrm{BMU}(X_{n+1})$, and the resulting threshold is computed from the buffer $\mathcal{I}_{N(k^\ast,r)}$.
Theorem~\ref{thm:regime2a} alone is not a coverage statement for the random map $x\mapsto \hat Q_{\mathrm{BMU}(x)}^{(2)}$.
The deployed central-cell statement is Theorem~\ref{thm:regime2b}, which pays the \gls{ks} bias $\varepsilon_k$.

\begin{theorem}[Approximate cell-conditional validity, Regime~2]
  \label{thm:regime2b}
  Under the standing assumptions, for every active cell $k\in\mathcal{K}_+$,
  \[
    \mathbb{P}\!\left(Y_{n+1}\in C_\alpha(X_{n+1};\hat q^{(2)})\mid \mathrm{BMU}(X_{n+1})=k\right)
    \geq 1-\alpha-\varepsilon_k.
  \]
\end{theorem}

\begin{corollary}[Marginal coverage bound for Regime~2]
  \label{cor:regime2c}
  Under the standing assumptions, it holds that
  \[
    \mathbb{P}\!\left(Y_{n+1}\in C_\alpha(X_{n+1};\hat q^{(2)})\right)
    \geq 1-\alpha-\bar\varepsilon,
    \qquad
    \bar\varepsilon\coloneqq\sum_{k\in\mathcal{K}_+}\pi_k\varepsilon_k.
  \]
\end{corollary}

The shorthand $\varepsilon_k$ from Definition~\ref{def:ks-bias} is the bias paid for replacing the cell score distribution by the score mixture over neighboring cells.
It is small when adjacent \gls{som} cells have similar conditional score distributions.

\subsection{Regime 3: Training-routed prototype enlargement}
\label{sec:methodology:summary}
\label{sec:methodology:regime3}
\label{sec:methodology:target-buffer}

Regime~3 enlarges $N(k,r)$ by adding the cells outside it in increasing prototype distance to the reference cell.
For $L\in\{0,\dots,K\}$, let $R_L(k)$ consist of $N(k,r)$ and the first $\min\{L,K-|N(k,r)|\}$ cells in this order.
Hence $R_L(k)=\mathcal{K}$ once $L\geq K-|N(k,r)|$.

Write $n_{\mathrm{tr}}\coloneqq|\mathcal{D}_{\mathrm{tr}}|>0$ for the number of training examples, and let $X_i^{\mathrm{tr}}$ denote their inputs.
For each cell $j\in\mathcal{K}$, let $T_j$ count the training inputs assigned to that cell,
\[
  T_j\coloneqq\sum_{i=1}^{n_{\mathrm{tr}}}
  \mathds{1}\{\mathrm{BMU}(X_i^{\mathrm{tr}})=j\},
  \qquad
  \sum_{j\in\mathcal{K}}T_j=n_{\mathrm{tr}}.
\]
Using the training proportion $T_j/n_{\mathrm{tr}}$ to estimate the mass of cell $j$, we project the size of the calibration buffer at the planned calibration size $n$ as
\[
  \widehat m_k(L)
  \coloneqq
  \frac{n}{n_{\mathrm{tr}}}\sum_{j\in R_L(k)}T_j.
\]
Given a fixed projected-size target $b\leq n$, the automatic rule selects one global budget
\begin{equation}
  L^\star
  \coloneqq
  \min\left\{L\in\{0,\dots,K\}:\min_{k\in\mathcal{K}}\widehat m_k(L)\geq b\right\}.
  \label{eq:training-routed-budget}
\end{equation}
For each reference cell, the order by prototype distance contains every cell outside the base neighborhood, including cells with $T_j=0$.
An empty cell may be one of the first $L$ additions, but it contributes zero to $\widehat m_k(L)$.

The trained \gls{som}, prototype orders, and counts $T_j$ are determined by $\mathcal{D}_{\mathrm{tr}}$ and the fixed training procedure, while $n$ and $b$ are fixed design values.
Thus $L^\star$ is selected without using any calibration observation.
Conditional on $\mathcal{D}_{\mathrm{pre}}$, it is a fixed integer, so the fixed-$L$ guarantees in Appendix~\ref{app:deferred_methodology} apply with $L=L^\star$.
The validity result does not require the projected sizes to equal the realized buffer sizes.
Projection error can affect efficiency, but not the rank argument, because the retrieved sets remain fixed before calibration.

Appendix~\ref{app:deferred_methodology} gives a split-routed target-buffer variant, in which an independent routing split $\mathcal{D}_{\mathrm{route}}$ selects an enlargement per cell to reach a target routing count.
Table~\ref{tab:socp-guarantees} summarizes the guarantees.
Exact bounds condition on the retrieved cell set, while central-cell bounds pay the \gls{ks} bias from pooling.

\begin{table}[ht]
  \centering
  \caption{
  Coverage guarantees of \gls{socp}.
  Retrieved-set rows evaluate a fixed reference threshold conditional on $k^\ast$ lying in the retrieved set.
  Cell and central rows evaluate the deployed threshold conditional on $k^\ast=k$.
  The marginal row averages the Regime~2 central bounds over $k^\ast$.
  Exact rows guarantee coverage of at least $1-\alpha$, while approximate rows include the stated \gls{ks} bias.
  Regime~3 and target-buffer quantities, including $\mathcal{D}_{\mathrm{route}}$, are defined in Appendix~\ref{app:deferred_methodology}.
  }
  \label{tab:socp-guarantees}
  \resizebox{\textwidth}{!}{%
    \begin{tabular}{@{}llllll@{}}
      \toprule
      Retrieval                                              & Calibration buffer             & Threshold                   & Conditioning event                                           & Lower bound                                                        & Status               \\
      \midrule
      Regime~1, cell (Thm.~\ref{thm:regime1})                & $\mathcal{I}_k$                & $\hat q^{(1)}$              & $k^\ast=k$                                                   & $1-\alpha$                                                         & exact                \\
      Regime~2, retrieved (Thm.~\ref{thm:regime2a})          & $\mathcal{I}_{N(k,r)}$         & $\hat Q_k^{(2)}$            & $k^\ast\in N(k,r)$                                           & $1-\alpha$                                                         & exact                \\
      Regime~2, central (Thm.~\ref{thm:regime2b})            & $\mathcal{I}_{N(k,r)}$         & $\hat q^{(2)}$              & $k^\ast=k$                                                   & $1-\alpha-\varepsilon_k$                                           & approximate          \\
      Regime~2, marginal (Cor.~\ref{cor:regime2c})           & $\mathcal{I}_{N(k^\ast,r)}$    & $\hat q^{(2)}$              & none                                                         & $1-\alpha-\bar\varepsilon$                                         & approximate          \\
      Regime~3, retrieved (Cor.~\ref{cor:regime3-training})  & $\mathcal{I}_{R_{L^\star}(k)}$ & $\hat Q_k^{(3,L^\star)}$    & $k^\ast\in R_{L^\star}(k)$                                   & $1-\alpha$                                                         & exact                \\
      Regime~3, central (Cor.~\ref{cor:regime3-training})    & $\mathcal{I}_{R_{L^\star}(k)}$ & $\hat q^{(3,L^\star)}$      & $k^\ast=k$                                                   & $1-\alpha-\varepsilon_k(R_{L^\star}(k))$                           & approximate          \\
      Target buffer, retrieved (Thm.~\ref{thm:split-target}) & $\mathcal{I}_{\hat R_b(k)}$    & $\hat Q_k^{(\mathrm{tar})}$ & $k^\ast\in\hat R_b(k)$, given $\mathcal{D}_{\mathrm{route}}$ & $1-\alpha$                                                         & exact, given routing \\
      Target buffer, central (Thm.~\ref{thm:split-target})   & $\mathcal{I}_{\hat R_b(k)}$    & $\hat q^{(\mathrm{tar})}$   & $k^\ast=k$                                                   & $1-\alpha-\mathbb{E}_{\mathrm{route}}[\varepsilon_k(\hat R_b(k))]$ & approximate          \\
      \bottomrule
    \end{tabular}%
  }
\end{table}

\section{Experiments}
\label{sec:experiments}

\subsection{Experimental protocol}

We evaluate \gls{socp} at $\alpha=0.1$ over ten seeds and ten datasets.
The regression suite contains Bio (CASP protein)~\citep{ranaPhysicochemicalPropertiesProtein2015}, Bike Sharing~\citep{fanaee-tBikeSharing2013}, California Housing~\citep{kelleypaceSparseSpatialAutoregressions1997}, Concrete Compressive Strength~\citep{i-chengyehConcreteCompressiveStrength1998}, and Auto MPG~\citep{r.quinlanAutoMPG1993}.
The classification suite contains CIFAR-10 and CIFAR-100~\citep{krizhevskyLearningMultipleLayers2012}, MNIST~\citep{lecunGradientbasedLearningApplied1998}, Covertype~\citep{blackardCovertype1998}, and ImageNet-1k~\citep{russakovskyImageNetLargeScale2015}.

The headline comparison uses the absolute residual score for regression and \gls{lac}~\citep{sadinleLeastAmbiguousSetvalued2019} for classification, written as SO-SCP against pooled \gls{scp}.
It uses fixed \gls{bmu} neighborhoods, or Regime~2, except on Concrete and Auto MPG.
Their calibration splits contain only $154$ and $58$ observations, so we use training-routed Regime~3 with the fixed projected-size target $b=19$.
The broader evaluation composes SO calibration with \gls{cqr}~\citep{romanoConformalizedQuantileRegression2019b}, \gls{chr}~\citep{sesiaConformalHistogramRegression2021}, \gls{nfcp}~\citep{colomboNormalizingFlowsConformal2024}, and \gls{rcp}~\citep{plassierRectifyingConformityScores2025} for regression, and with \gls{aps}~\citep{romanoClassificationValidAdaptive2020}, \gls{raps}~\citep{angelopoulosUncertaintySetsImage2021}, \gls{saps}~\citep{huangConformalPredictionDeep2024}, and a rank-only score~\citep{huangConformalPredictionDeep2024} for classification.
We also compare pooled calibration, \gls{lcp}~\citep{guanLocalizedConformalPrediction2022a}, \gls{rlcp}~\citep{horeConformalPredictionLocal2023}, and Mondrian calibration on $k$-means groups fitted on training data~\citep{vovkMondrianConfidenceMachine2003}.
Classification additionally includes Classwise and Clustered calibration~\citep{dingClassConditionalConformalPrediction2023a}.
Appendix~\ref{app:experimental_details} gives every split, predictor, map, score, and selection rule, and the code is available at \codeurl.

We measure regional coverage with weighted coverage gap (WCovGap)~\citep{braunConditionalCoverageDiagnostics2025}.
Let $G$ be the groups of a fixed audit partition that contain test points, with empirical coverage $c_g$, test count $n_g$, and total test size $n_{\mathrm{test}}$.
Then
\[
  \mathrm{WCovGap}
  =
  100\sum_{g\in G}\frac{n_g}{n_{\mathrm{test}}}|c_g-(1-\alpha)|.
\]
We also report marginal coverage and output size, meaning interval width for regression and set cardinality for classification.
Auditing on the learned \gls{som} alone could favor \gls{socp} by construction, so the primary audit uses an external $k$-means partition with $K=25$ clusters, fitted on the training representation and shared by every method within a seed.
Partitions with $K\in\{10,50\}$ test granularity, while held-out \gls{wsc} target error, the \gls{ssc} gap, and $L_1$ \gls{ert} provide diagnostics that need no feature-space partition.
For each dataset, headline significance uses a two-sided Wilcoxon signed-rank test on the ten paired WCovGap differences, followed by Holm correction across the ten datasets.

\subsection{Results}

Figure~\ref{fig:cr_tradeoff} gives the primary trade-off by seed, while Tables~\ref{tab:cr_main_regression} and~\ref{tab:cr_main_classification} report the corresponding coverage, WCovGap, and output size.

\begin{figure}[ht]
  \centering
  \includegraphics[width=0.99\textwidth]{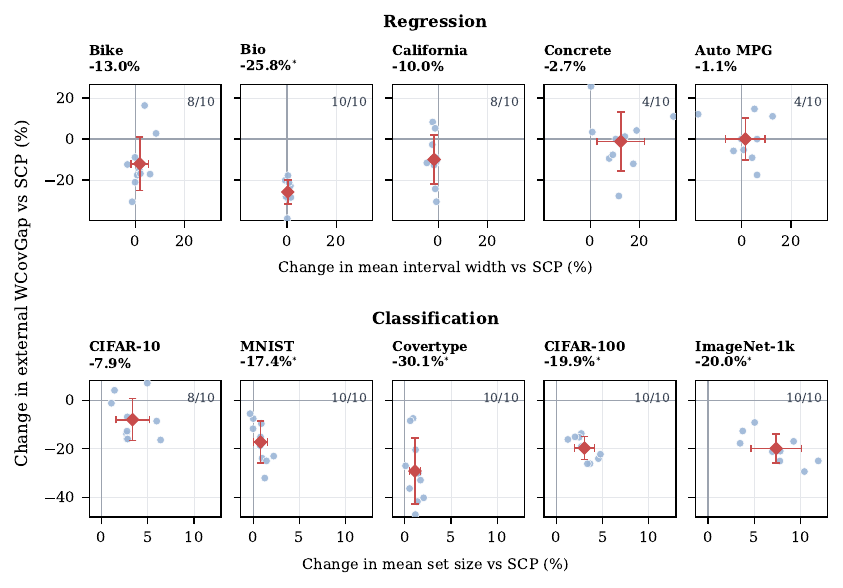}
  \caption{
    Relative changes from \gls{scp} to SO-SCP in external WCovGap ($K=25$) and output size, using the absolute residual score for regression (top) and \gls{lac} for classification (bottom).
    Dots show individual seeds, while diamonds and bars show the mean and one standard deviation of their relative changes.
    Panel titles give percentage changes between the means in Tables~\ref{tab:cr_main_regression} and~\ref{tab:cr_main_classification}.
    Stars mark paired Wilcoxon $p<0.05$ after Holm adjustment, and counts give seeds with lower WCovGap.
    Concrete and Auto MPG use Regime~3.
  }
  \label{fig:cr_tradeoff}
\end{figure}

\paragraph{Does SO-SCP lower WCovGap without materially increasing output size?}
SO-SCP lowers mean WCovGap on all eight Regime~2 datasets, by $18.0\%$ on average for a $2.0\%$ mean increase in output size.
With Regime~3 on Concrete and Auto MPG, mean WCovGap is lower on all ten datasets, although the reductions on these two datasets are small, $2.7\%$ and $1.1\%$.
The overall mean changes are $-14.8\%$ in WCovGap and $+3.0\%$ in output size.
Bio, MNIST, Covertype, CIFAR-100, and ImageNet-1k remain significant after Holm correction.
On the eight Regime~2 datasets, every marginal coverage mean across ten seeds lies within $0.008$ of the $0.90$ target for both methods.

On these eight datasets, output sizes remain close to those of \gls{scp}.
On four classification datasets, SO-SCP returns fewer than two labels on average even though their label spaces contain between $7$ and $1{,}000$ classes, and it adds at most $0.11$ labels over \gls{scp}.
On CIFAR-100, the mean set size moves from $14.25$ to $14.67$ labels out of $100$.
For Bike Sharing, Bio, and California Housing, mean interval width changes by $+1.8\%$, $+0.4\%$, and $-2.2\%$.
Together, these results show that SO-SCP reduces regional coverage gaps while preserving the output scale of pooled calibration.

\paragraph{Does SO calibration compose well with other scores?}
The headline comparison uses the base score for each task.
Now, we ask whether SO calibration remains useful with other score constructions.
For a common Regime~2 retrieval rule, SO calibration lowers mean WCovGap in $43$ of $50$ score and dataset comparisons, $22$ of $25$ in regression and $21$ of $25$ in classification (Appendix Tables~\ref{tab:cr_factorial_regression} and~\ref{tab:cr_factorial_classification}).
The $43$ reductions average $13.7\%$ ($0.66$ percentage points), while the seven increases average $5.4\%$ ($0.09$ points).
The increases occur only on Bike Sharing and MNIST, for scores with the lowest pooled WCovGap on each respective dataset.
Averaged over the five datasets of each task, SO calibration lowers WCovGap for every score.
Both base scores improve on every dataset.
Thus, \gls{socp} composes across score families without changing the predictor or score.

\paragraph{How does SO-SCP compare with other calibrators?}
The composition study compares pooled and SO calibration across scores.
Now, we hold the base score fixed and compare the calibration rules directly in Appendix Tables~\ref{tab:cr_calibrator_tradeoff_regression} and~\ref{tab:cr_calibrator_tradeoff_classification} and Figures~\ref{fig:cr_calibrator_tradeoff_regression} and~\ref{fig:cr_calibrator_tradeoff_classification}.
On Bike Sharing, Bio, and California Housing, only SO-SCP and pooled \gls{scp} keep all $115{,}910$ recorded intervals finite.
SO-SCP reduces WCovGap by $0.35$ to $0.81$ points there, with mean width changes between $-2.2\%$ and $+1.8\%$.
By comparison, \gls{lcp} and \gls{rlcp} leave between $0.77\%$ and $5.50\%$ of the predictions unbounded on the datasets where they return infinite intervals.
For classification, SO-SCP gains $0.39$ to $1.44$ WCovGap points, at least five times the gain of \gls{lcp} on every dataset, with increases of $1.1\%$ to $7.4\%$ in set size.
On ImageNet-1k, \gls{rlcp} averages $139.33$ of the $1{,}000$ labels, against $1.48$ for \gls{scp} and $1.59$ for SO-SCP, while Classwise increases the set size by $699\%$.
On the same dataset, Clustered gains $0.90$ WCovGap points at a $22.3\%$ size increase, compared with $1.13$ points at $7.4\%$ for SO-SCP.
Mondrian $k$-means lowers WCovGap more than SO-SCP on all five classification datasets, but enlarges mean sets more on four of them, reaching $+19.6\%$ on ImageNet-1k against $+7.4\%$ for SO-SCP.
The audit groups are independently fitted $k$-means partitions of the same representation, which may favor Mondrian $k$-means.
On the two smallest regression datasets, Mondrian $k$-means also returns unbounded intervals at rates of $16.58\%$ and $28.50\%$, against $0.19\%$ and $10.17\%$ for Regime~2 SO-SCP.
On the eight adequately supported datasets, SO-SCP lowers WCovGap relative to \gls{scp} on each one, with output-size changes between $-2.2\%$ and $+7.4\%$, and without the extreme enlargement or unbounded outputs observed for several alternatives.

\paragraph{Do the improvements persist beyond the headline audit partition?}
The headline result uses one external $K=25$ partition.
We test whether the gains persist across other external granularities, the shared \gls{som} partition, and three diagnostics that need no feature-space partition.
With Regime~2 on every dataset, SO-SCP lowers WCovGap in all $30$ external dataset and granularity comparisons with $K\in\{10,25,50\}$ (Appendix Figures~\ref{fig:cr_transfer_regression} and~\ref{fig:cr_transfer_classification}).
On the shared \gls{som} partition, SO-SCP lowers WCovGap on all eight datasets that use Regime~2 in the headline comparison, by $10.6\%$ on average, with a mean change of $+2.0\%$ in output size.
Among the three diagnostics, held-out \gls{wsc} target error and $L_1$ \gls{ert} improve most often.
Their mean relative changes across datasets are $-29.6\%$ and $-11.1\%$, and they improve on all nine datasets with a recorded \gls{wsc} value and on nine of ten for \gls{ert}.
The \gls{ssc} gap decreases on six datasets and matches \gls{scp} on MNIST (Appendix Tables~\ref{tab:cr_conditional_regression} and~\ref{tab:cr_conditional_classification}).
These audits show that the gain is not tied to one partition.

\paragraph{How sensitive is \gls{socp} to the map configuration and small calibration buffers?}
The preceding comparisons use one map configuration per dataset.
On California Housing, $15$ of the $18$ grid and radius configurations with no infinite cutoff have a lower mean WCovGap than the pooled \gls{scp} value of $3.50$ points on the same external audit, one matches it, and the remaining two exceed it by at most $0.02$ points (Appendix Figure~\ref{fig:cr_sensitivity}).
The headline $7\times8$ grid with radius $2$ gives $3.15$ points rather than the sweep minimum of $2.32$, so the reduction does not rely on selecting the best configuration.
On Concrete and Auto MPG, reducing the calibration set to $20$ observations with the map and Regime~2 radius held fixed produces infinite cutoffs for $90.1\%$ and $83.2\%$ of queries.
Training-routed Regime~3 removes them in all seven recorded dataset and calibration-size settings and returns finite mean widths throughout (Appendix Figure~\ref{fig:cr_regime3}), with mean marginal coverage between $0.897$ and $0.926$.
At $20$ observations it reduces to pooled calibration.
As the planned calibration size grows, the training-routed budget shrinks toward local calibration, and Regime~3 lowers mean WCovGap relative to pooling at every larger recorded size.

\paragraph{Where do the aggregate gains come from?}
The Covertype map shows how SO calibration redistributes output size within one fitted map (Appendix Figure~\ref{fig:cr_worked_example} and Table~\ref{tab:cr_worked_example}).
For seed $42$, WCovGap on the shared \gls{som} partition falls from $3.53$ to $3.04$ points, while marginal coverage and mean set size barely change.
One undercovered cell moves from $0.836$ to $0.914$ coverage as its mean set size rises from $1.237$ to $1.612$ labels.
In an overcovered cell, coverage moves from $0.943$ to $0.898$ as mean set size falls from $1.193$ to $1.061$ labels.
The ten-seed Covertype result follows the same pattern, with external WCovGap at $K=25$ falling by $30.1\%$ for a $1.7\%$ increase in mean set size and every seed favoring SO-SCP.

\section{Conclusion}
\label{sec:conclusion}

\glsdisp{socp}{Self-Organized Conformal Prediction} uses an unsupervised \gls{som} fitted before calibration to organize calibration scores into local buffers without changing the predictor or nonconformity score.
The analysis separates exact retrieved-set validity from central-cell bounds that pay a \gls{ks} bias.
This isolates the conformal rank guarantee from the mismatch introduced by neighborhood pooling.
Regime~1 gives exact cell-conditional validity, while Regimes~2 and~3 give exact retrieved-set validity.
Training-routed Regime~3 preserves these guarantees because its enlargement budget is fixed from training occupancies, the planned calibration size, and the target $b$ before calibration.

Across ten datasets, SO-SCP lowers mean external WCovGap on all ten, averaging a $14.8\%$ reduction for a $3.0\%$ increase in output size.
With fixed neighborhoods, SO calibration improves $43$ of $50$ score and dataset comparisons, and SO-SCP improves all $30$ external audits.
These fixed-neighborhood predictions also improve held-out \gls{wsc} target error on all nine recorded datasets and $L_1$ \gls{ert} on nine of ten.
On Concrete and Auto MPG, the datasets with the two smallest calibration splits, Regime~3 removes every infinite cutoff in the seven recorded dataset and calibration-size settings.
Among localized calibrators, only SO-SCP keeps all recorded intervals finite across Bike Sharing, Bio, and California Housing.
The \gls{som} diagnostics in Appendix~\ref{app:spatial_diagnostics} indicate that the WCovGap gains reflect targeted regional adjustment rather than uniform widening.

As a result, \gls{socp} is a local calibration layer built on a fixed \gls{som}.
Its guarantees hold on the cells used for calibration rather than pointwise.
Relative to pooled calibration, it lowers regional coverage gaps at a small average cost in output size.
Other calibrators reach lower gaps on some datasets, but in nearly all cases at the cost of larger outputs or unbounded intervals, which can hinder their real-world utility.
Locality still depends on the representation, map resolution, and retrieval rule.
Future work should assess how reliably the cell and \gls{ks} diagnostics can be estimated from finite samples and study learned neighborhoods under distribution shift.
Overlapping assignments and multiple map resolutions may soften grid boundaries and connect \gls{som} routing to the weighted localization of \gls{lcp} and \gls{rlcp}~\citep{guanLocalizedConformalPrediction2022a,horeConformalPredictionLocal2023}, but require guarantees tailored to the resulting retrieval rules.

\clearpage
\begin{ack}
  The work of Aymeric Dieuleveut is supported by French State aid managed by the Agence Nationale de la Recherche (ANR) under the France 2030 program with the reference ANR-23-IACL-0005 (Hi!\ PARIS).
  The authors declare no competing interests.
\end{ack}

\bibliographystyle{abbrvnat}
\bibliography{references}

\newpage
\appendix
\section{Appendix}
\label{sec:appendix}

\subsection{Deferred methodological details}
\label{app:deferred_methodology}

This appendix collects the proof tools for Mondrian validity, the \gls{ks} bridge, and the tower-property transfer.
It also presents retrieval by prototype \gls{knn} and the split-routed target-buffer variant, which selects a separate enlargement for each cell to reach a target routing count.

\begin{lemma}
    \label{lem:mondrian}
    Let $G:\mathcal{X}\to\mathcal{G}$ be a fixed measurable grouping into a countable set.
    For $g\in\mathcal{G}$, define
    \[
        S_g\coloneqq\{i\in\{1,\dots,n\}:G(X_i)=g\},
        \qquad
        \hat q_g\coloneqq\hat q_{S_g}^{(\alpha)}.
    \]
    Under the standing assumptions, for every group with $\mathbb{P}(G(X_{n+1})=g)>0$,
    \[
        \mathbb{P}\!\left(Y_{n+1}\in C_\alpha(X_{n+1};\hat q_g)\mid G(X_{n+1})=g\right)
        \geq 1-\alpha.
    \]
\end{lemma}

\begin{lemma}
    \label{lem:bridge}
    Let $k\in\mathcal{K}_+$ and let $A\subseteq\mathcal{K}$ satisfy $k\in A$ and $\pi_A>0$.
    For every $q\in\mathbb{R}\cup\{+\infty\}$,
    \[
        \left|F^{(k)}(q)-F^A(q)\right|\leq \varepsilon_k(A),
    \]
    with the convention $F^{(k)}(+\infty)=F^A(+\infty)=1$.
    The bound is uniform in $q$, hence pathwise valid for any $\sigma(\mathcal{D}_{\mathrm{cal}})$-measurable threshold.
\end{lemma}

\begin{lemma}
    \label{lem:tower-transfer}
    Let $A\subseteq\mathcal{K}$ with $\pi_A>0$, $k\in A\cap\mathcal{K}_+$, and let $\hat Q$ be any $\sigma(\mathcal{D}_{\mathrm{cal}})$-measurable threshold.
    Write $\mathbb{E}_{\mathrm{cal}}[Z]\coloneqq\mathbb{E}[Z]$ for any $\sigma(\mathcal{D}_{\mathrm{cal}})$-measurable $Z$.
    Under the standing assumptions,
    \[
        \mathbb{P}\!\left(s_{n+1}\leq \hat Q\mid k^\ast=k\right)
        =
        \mathbb{E}_{\mathrm{cal}}\!\left[F^{(k)}(\hat Q)\right],
        \qquad
        \mathbb{P}\!\left(s_{n+1}\leq \hat Q\mid k^\ast\in A\right)
        =
        \mathbb{E}_{\mathrm{cal}}\!\left[F^A(\hat Q)\right],
    \]
    and
    \[
        \mathbb{P}\!\left(s_{n+1}\leq \hat Q\mid k^\ast=k\right)
        \geq
        \mathbb{P}\!\left(s_{n+1}\leq \hat Q\mid k^\ast\in A\right)-\varepsilon_k(A).
    \]
    In particular, $\mathbb{P}(s_{n+1}\leq \hat Q\mid k^\ast\in A)\geq1-\alpha$ implies $\mathbb{P}(s_{n+1}\leq \hat Q\mid k^\ast=k)\geq1-\alpha-\varepsilon_k(A)$.
\end{lemma}

\subsubsection{Fixed prototype-KNN retrieval}
\label{app:fixed_prototype_knn}

For an integer $L\in\{0,\dots,K\}$, let $K_L(k)$ collect the first $\min\{L,K-|N(k,r)|\}$ cells outside $N(k,r)$ in increasing prototype distance $\|w_j-w_k\|_2$.
Ties are resolved by the deterministic rule fixed with the \gls{som}.
Define
\[
    R_L(k)\coloneqq N(k,r)\cup K_L(k),
    \qquad
    \hat Q_k^{(3,L)}\coloneqq\hat q^{(\alpha)}_{\mathcal{I}_{R_L(k)}},
    \qquad
    \hat q^{(3,L)}\coloneqq\hat Q_{k^\ast}^{(3,L)}.
\]
The set $R_L(k)$ depends only on the pre-trained \gls{som}, so it is calibration-free.
The case $L=0$ recovers Regime~2, while $L=K$ retrieves the full map for every reference cell.

\begin{theorem}[Fixed-$L$ KNN validity, Regime~3]
    \label{thm:regime3a}
    Fix $L\in\{0,\dots,K\}$.
    Under the standing assumptions, for every active cell $k\in\mathcal{K}_+$,
    \begin{align}
        \mathbb{P}\!\left(s_{n+1}\leq \hat Q_k^{(3,L)}\mid \mathrm{BMU}(X_{n+1})\in R_L(k)\right)
         & \geq 1-\alpha, \label{eq:r3a-nbhd}                       \\
        \mathbb{P}\!\left(Y_{n+1}\in C_\alpha(X_{n+1};\hat q^{(3,L)})\mid \mathrm{BMU}(X_{n+1})=k\right)
         & \geq 1-\alpha-\varepsilon_k(R_L(k)). \label{eq:r3a-cell}
    \end{align}
\end{theorem}

\begin{corollary}[Marginal coverage bound for fixed-$L$ Regime~3]
    \label{cor:regime3a-marg}
    Fix $L\in\{0,\dots,K\}$.
    Under the standing assumptions,
    \[
        \mathbb{P}\!\left(Y_{n+1}\in C_\alpha(X_{n+1};\hat q^{(3,L)})\right)
        \geq
        1-\alpha-\sum_{k\in\mathcal{K}_+}\pi_k\varepsilon_k(R_L(k)).
    \]
\end{corollary}

\begin{corollary}[Training-routed Regime~3]
    \label{cor:regime3-training}
    Let $L^\star$ be selected by~\eqref{eq:training-routed-budget} from the training \gls{bmu} counts, the planned calibration size, and a fixed projected-size target $b\leq n$.
    Conditional on $\mathcal{D}_{\mathrm{pre}}$, the two bounds in Theorem~\ref{thm:regime3a} and the marginal bound in Corollary~\ref{cor:regime3a-marg} hold with $L=L^\star$.
\end{corollary}

\subsubsection{Split-routed target buffer}
\label{app:split_routed_target_buffer}

Reusing $\mathcal{D}_{\mathrm{cal}}$ both to enlarge the retrieved set and to compute the quantile breaks the calibration-free property of the retrieved set.
Data splitting restores it.
Let
\[
    \mathcal{D}_{\mathrm{route}}\coloneqq\{X_a^{\mathrm{route}}\}_{a=1}^{n_{\mathrm{route}}}
\]
be independent of $\mathcal{D}_{\mathrm{cal}}$ and of the test point.
For each cell $k$, set
\[
    A_h(k)\coloneqq N(k,r)\cup K_h(k),
    \qquad h=0,\dots,K-|N(k,r)|,
\]
where $K_h(k)$ collects the $h$ closest cells outside $N(k,r)$ in prototype distance $\|w_j-w_k\|_2$.
For a target routing count $b\in\{0,\dots,n_{\mathrm{route}}\}$,
\[
    \hat L_b(k)
    \coloneqq
    \min\!\left\{h:\sum_{a=1}^{n_{\mathrm{route}}}\mathds{1}\{\mathrm{BMU}(X_a^{\mathrm{route}})\in A_h(k)\}\geq b\right\}.
\]
The minimum is well-defined: at the largest admissible $h=K-|N(k,r)|$, $K_h(k)=\mathcal{K}\setminus N(k,r)$ and hence $A_h(k)=\mathcal{K}$, so every route sample's \gls{bmu} lies in $A_h(k)$ and the indicator sum equals $n_{\mathrm{route}}\geq b$.
We then set
\[
    \hat R_b(k)\coloneqq A_{\hat L_b(k)}(k),
    \qquad
    \hat Q_k^{(\mathrm{tar})}\coloneqq\hat q^{(\alpha)}_{\mathcal{I}_{\hat R_b(k)}},
    \qquad
    \hat q^{(\mathrm{tar})}\coloneqq\hat Q_{k^\ast}^{(\mathrm{tar})}.
\]

Conditional on $\mathcal{D}_{\mathrm{route}}$, $\hat R_b(k)$ is a deterministic union of \gls{som} cells, and $(\mathcal{D}_{\mathrm{cal}},X_{n+1},Y_{n+1})$ remain i.i.d.\ from $P$ by the independence of the routing split.

\begin{theorem}[Split-routed target-buffer validity]
    \label{thm:split-target}
    Under the standing assumptions and the independent routing split above, for every active cell $k\in\mathcal{K}_+$, $\mathcal{D}_{\mathrm{route}}$-almost surely,
    \begin{align}
        \mathbb{P}\!\left(s_{n+1}\leq \hat Q_k^{(\mathrm{tar})}\mid \mathrm{BMU}(X_{n+1})\in \hat R_b(k),\mathcal{D}_{\mathrm{route}}\right)
         & \geq 1-\alpha, \label{eq:target-retrieved}                       \\
        \mathbb{P}\!\left(Y_{n+1}\in C_\alpha(X_{n+1};\hat q^{(\mathrm{tar})})\mid \mathrm{BMU}(X_{n+1})=k,\mathcal{D}_{\mathrm{route}}\right)
         & \geq 1-\alpha-\varepsilon_k(\hat R_b(k)). \label{eq:target-cell}
    \end{align}
    Marginalizing over $\mathcal{D}_{\mathrm{route}}$,
    \[
        \mathbb{P}\!\left(Y_{n+1}\in C_\alpha(X_{n+1};\hat q^{(\mathrm{tar})})\mid \mathrm{BMU}(X_{n+1})=k\right)
        \geq
        1-\alpha-\mathbb{E}_{\mathrm{route}}\!\left[\varepsilon_k(\hat R_b(k))\right].
    \]
\end{theorem}

\subsection{Technical proofs}
\label{app:theory_proofs}

\paragraph{Proof of Lemma~\ref{lem:mondrian}.}
\begin{proof}
    Fix $g$ with $\mathbb{P}(G(X_{n+1})=g)>0$.

    \emph{Step~1 (slice).}
    For deterministic $J\subseteq\{1,\dots,n\}$, set $E_J\coloneqq\{S_g=J,\ G(X_{n+1})=g\}$.
    By i.i.d.\ data and measurability of $G$, $E_J$ is a product of independent events, and conditioning on $E_J$ leaves the in-group coordinates $\{(X_i,Y_i):i\in J\}\cup\{(X_{n+1},Y_{n+1})\}$ i.i.d.\ from $P(\cdot\mid G(X)=g)$.

    \emph{Step~2 (rank uniformity).}
    Set $m=|J|$.
    Step~1 makes the score sample exchangeable under $\mathbb{P}(\cdot\mid E_J)$.
    Break ties with an independent uniform variable.
    The rank $R$ of $s_{n+1}$ among the $m+1$ scores after tie breaking satisfies $\mathbb{P}(R\leq\ell\mid E_J)=\ell/(m+1)$ for $\ell\in\{1,\dots,m+1\}$.
    Refinement through tie breaking gives $\{R\leq\ell\}\subseteq\{s_{n+1}\leq s_{(\ell:m)}\}$.

    \emph{Step~3 (rank-quantile bound).}
    Set $\ell=\lceil(1-\alpha)(m+1)\rceil$.
    If $\ell>m$, $\hat q_g=+\infty$ gives $\mathbb{P}(s_{n+1}\leq\hat q_g\mid E_J)=1$.
    Otherwise $\hat q_g=s_{(\ell:m)}$ and
    \[
        \mathbb{P}(s_{n+1}\leq\hat q_g\mid E_J)
        \geq
        \mathbb{P}(R\leq\ell\mid E_J)
        =
        \frac{\ell}{m+1}
        \geq 1-\alpha.
    \]

    \emph{Step~4 (slice averaging).}
    The slices $\{E_J\}$ partition $\{G(X_{n+1})=g\}$, so $\mathbb{P}(s_{n+1}\leq\hat q_g\mid G(X_{n+1})=g)$ is a convex combination of slice-conditional probabilities, each at least $1-\alpha$.
    By Definition~\ref{def:conformal-q}, $\{s_{n+1}\leq\hat q_g\}=\{Y_{n+1}\in C_\alpha(X_{n+1};\hat q_g)\}$.
\end{proof}

\paragraph{Proof of Lemma~\ref{lem:bridge}.}
\begin{proof}
    At $q=+\infty$ both \glspl{cdf} equal one, so the bound holds trivially.
    Fix $q\in\mathbb{R}$.
    Definition~\ref{def:cell-cdf} gives the mixture identity $F^A(q)=\sum_{j\in A\cap\mathcal{K}_+}\tilde\pi_j^A F^{(j)}(q)$, hence
    \[
        F^{(k)}(q)-F^A(q)
        =
        \sum_{j\in A\cap\mathcal{K}_+}\tilde\pi_j^A\bigl(F^{(k)}(q)-F^{(j)}(q)\bigr).
    \]
    The triangle inequality combined with $|F^{(k)}(q)-F^{(j)}(q)|\leq d_{\mathrm{KS}}(k,j)$ uniformly in $q$ yields
    \[
        \bigl|F^{(k)}(q)-F^A(q)\bigr|
        \leq
        \sum_{j\in A\cap\mathcal{K}_+}\tilde\pi_j^A d_{\mathrm{KS}}(k,j)
        =
        \varepsilon_k(A).
    \]
    The right-hand side is independent of $q$, so the bound transfers pathwise to any calibration-measurable threshold.
\end{proof}

\paragraph{Proof of Lemma~\ref{lem:tower-transfer}.}
\begin{proof}
    \emph{Step~1 (fixed-set tower identity).}
    Let $B\subseteq\mathcal{K}$ with $\pi_B>0$, and set $E_B\coloneqq\{k^\ast\in B\}$.
    Conditional on $\mathcal{D}_{\mathrm{cal}}$, $\hat Q$ is a deterministic scalar and $(X_{n+1},Y_{n+1})$ is an independent draw from $P$.
    Integrating against $P$ and applying Definition~\ref{def:cell-cdf},
    \[
        \mathbb{P}(s_{n+1}\leq\hat Q,\,E_B\mid\mathcal{D}_{\mathrm{cal}})
        =
        \pi_B\,F^B(\hat Q).
    \]
    The tower property and division by $\pi_B$ give
    \begin{equation}\label{eq:tower-fixed-set}
        \mathbb{P}(s_{n+1}\leq\hat Q\mid k^\ast\in B)
        =
        \mathbb{E}_{\mathrm{cal}}\!\left[F^B(\hat Q)\right].
    \end{equation}

    \emph{Step~2 (specialize).}
    Equation~\eqref{eq:tower-fixed-set} with $B=\{k\}$ and $B=A$ yields the two displayed identities of the lemma.

    \emph{Step~3 (bridge).}
    Lemma~\ref{lem:bridge} is uniform in $q$ and $\hat Q$ is calibration-measurable, so pathwise $F^{(k)}(\hat Q)\geq F^A(\hat Q)-\varepsilon_k(A)$.
    Taking $\mathbb{E}_{\mathrm{cal}}$ of both sides and using Step~2 gives the transfer bound.
\end{proof}

\paragraph{Proof of Theorem~\ref{thm:regime1}.}
\begin{proof}
    Apply Lemma~\ref{lem:mondrian} to the calibration-free grouping $G=\mathrm{BMU}$ at $g=k$:
    \[
        \mathbb{P}\!\left(s_{n+1}\leq\hat q^{(\alpha)}_{\mathcal{I}_k}\mid \mathrm{BMU}(X_{n+1})=k\right)
        \geq 1-\alpha.
    \]
    On $\{\mathrm{BMU}(X_{n+1})=k\}$, $\hat q^{(1)}=\hat q^{(\alpha)}_{\mathcal{I}_k}$, and the score event coincides with the coverage event.
\end{proof}

\paragraph{Proof of Theorem~\ref{thm:regime2a}.}
\begin{proof}
    The binary grouping $G_k(x)=\mathds{1}\{\mathrm{BMU}(x)\in N(k,r)\}$ is calibration-free, with group-$1$ buffer $\mathcal{I}_{N(k,r)}$.
    Lemma~\ref{lem:mondrian} at $g=1$ gives
    \[
        \mathbb{P}\!\left(s_{n+1}\leq\hat q^{(\alpha)}_{\mathcal{I}_{N(k,r)}}\mid \mathrm{BMU}(X_{n+1})\in N(k,r)\right)
        \geq 1-\alpha,
    \]
    which is the claim, since $\hat Q_k^{(2)}=\hat q^{(\alpha)}_{\mathcal{I}_{N(k,r)}}$.
\end{proof}

\paragraph{Proof of Theorem~\ref{thm:regime2b}.}
\begin{proof}
    Set $A=N(k,r)$, so $\hat Q_k^{(2)}=\hat q^{(\alpha)}_{\mathcal{I}_A}$ and $\pi_A>0$ since $k\in A\cap\mathcal{K}_+$.
    Theorem~\ref{thm:regime2a} gives $\mathbb{P}(s_{n+1}\leq \hat Q_k^{(2)}\mid k^\ast\in A)\geq 1-\alpha$, which Lemma~\ref{lem:tower-transfer} transfers to
    \[
        \mathbb{P}(s_{n+1}\leq \hat Q_k^{(2)}\mid k^\ast=k)\geq 1-\alpha-\varepsilon_k(A).
    \]
    On $\{k^\ast=k\}$, $\hat q^{(2)}=\hat Q_k^{(2)}$ and $\varepsilon_k(A)=\varepsilon_k$.
\end{proof}

\paragraph{Proof of Corollary~\ref{cor:regime2c}.}
\begin{proof}
    Let $E=\{Y_{n+1}\in C_\alpha(X_{n+1};\hat q^{(2)})\}$.
    Total probability over the partition $\{\mathrm{BMU}(X_{n+1})=k\}_{k\in\mathcal{K}_+}$ combined with Theorem~\ref{thm:regime2b} gives
    \[
        \mathbb{P}(E)
        \geq
        \sum_{k\in\mathcal{K}_+}\pi_k(1-\alpha-\varepsilon_k)
        =
        1-\alpha-\bar\varepsilon,
    \]
    using $\sum_{k\in\mathcal{K}_+}\pi_k=1$ and the definition $\bar\varepsilon=\sum_{k\in\mathcal{K}_+}\pi_k\varepsilon_k$.
\end{proof}

\paragraph{Proof of Theorem~\ref{thm:regime3a}.}
\begin{proof}
    $R_L(k)$ depends only on the pre-trained \gls{som}, so it is calibration-free, and $k\in R_L(k)\cap\mathcal{K}_+$ gives $\pi_{R_L(k)}>0$.

    \emph{Retrieved-set part.}
    The binary grouping $G_{k,L}(x)=\mathds{1}\{\mathrm{BMU}(x)\in R_L(k)\}$ is calibration-free with buffer $\mathcal{I}_{R_L(k)}$.
    Lemma~\ref{lem:mondrian} at $g=1$ yields~\eqref{eq:r3a-nbhd}.

    \emph{Cell-conditional part.}
    Set $A=R_L(k)$ and $\hat Q=\hat Q_k^{(3,L)}$.
    Lemma~\ref{lem:tower-transfer} transfers~\eqref{eq:r3a-nbhd} to
    \[
        \mathbb{P}\!\left(s_{n+1}\leq \hat Q_k^{(3,L)}\mid k^\ast=k\right)
        \geq
        1-\alpha-\varepsilon_k(R_L(k)),
    \]
    and on $\{k^\ast=k\}$, $\hat q^{(3,L)}=\hat Q_k^{(3,L)}$, giving~\eqref{eq:r3a-cell}.
\end{proof}

\paragraph{Proof of Corollary~\ref{cor:regime3a-marg}.}
\begin{proof}
    Let $E=\{Y_{n+1}\in C_\alpha(X_{n+1};\hat q^{(3,L)})\}$.
    Total probability over $\{\mathrm{BMU}(X_{n+1})=k\}_{k\in\mathcal{K}_+}$ combined with~\eqref{eq:r3a-cell} gives
    \[
        \mathbb{P}(E)
        \geq
        \sum_{k\in\mathcal{K}_+}\pi_k\bigl(1-\alpha-\varepsilon_k(R_L(k))\bigr)
        =
        1-\alpha-\sum_{k\in\mathcal{K}_+}\pi_k\varepsilon_k(R_L(k)),
    \]
    using $\sum_{k\in\mathcal{K}_+}\pi_k=1$.
\end{proof}

\paragraph{Proof of Corollary~\ref{cor:regime3-training}.}
\begin{proof}
    Conditional on $\mathcal{D}_{\mathrm{pre}}$, the trained \gls{som}, the orders by prototype distance, and the training counts are fixed.
    Since the planned calibration size $n$ and the target $b$ are fixed design values, $L^\star$ is also fixed.
    Consequently, the entire family $\{R_{L^\star}(k):k\in\mathcal{K}\}$ is fixed before calibration.
    At prediction time, the test point selects one member of this family through its \gls{bmu} $k^\ast$.
    Apply Theorem~\ref{thm:regime3a} and Corollary~\ref{cor:regime3a-marg} with $L=L^\star$.
\end{proof}

\paragraph{Proof of Theorem~\ref{thm:split-target}.}
\begin{proof}
    Condition throughout on $\mathcal{D}_{\mathrm{route}}$.
    Then $\hat R_b(k)$ is a deterministic union of \gls{som} cells, $(\mathcal{D}_{\mathrm{cal}},X_{n+1},Y_{n+1})$ remain i.i.d.\ from $P$, and $\pi_{\hat R_b(k)}>0$ since $k\in\hat R_b(k)\cap\mathcal{K}_+$.

    \emph{Retrieved-set part.}
    The binary grouping $G_{k,b}(x)=\mathds{1}\{\mathrm{BMU}(x)\in\hat R_b(k)\}$ is deterministic conditional on $\mathcal{D}_{\mathrm{route}}$.
    Lemma~\ref{lem:mondrian} under $\mathbb{P}(\cdot\mid\mathcal{D}_{\mathrm{route}})$ at $g=1$ yields~\eqref{eq:target-retrieved}.

    \emph{Cell-conditional part.}
    Set $A=\hat R_b(k)$.
    All hypotheses of Lemma~\ref{lem:tower-transfer} hold under $\mathbb{P}(\cdot\mid\mathcal{D}_{\mathrm{route}})$, giving
    \[
        \mathbb{P}\!\left(s_{n+1}\leq \hat Q_k^{(\mathrm{tar})}\mid k^\ast=k,\mathcal{D}_{\mathrm{route}}\right)
        \geq
        1-\alpha-\varepsilon_k(\hat R_b(k)),
    \]
    which is~\eqref{eq:target-cell} on $\{k^\ast=k\}$ where $\hat q^{(\mathrm{tar})}=\hat Q_k^{(\mathrm{tar})}$.

    \emph{Marginalize.}
    Independence of $\mathcal{D}_{\mathrm{route}}$ and $(X_{n+1},Y_{n+1})$ implies
    \[
        \mathbb{P}(\,\cdot\,\mid\mathrm{BMU}(X_{n+1})=k)
        =
        \mathbb{E}_{\mathrm{route}}\!\left[\mathbb{P}(\,\cdot\,\mid\mathrm{BMU}(X_{n+1})=k,\mathcal{D}_{\mathrm{route}})\right].
    \]
    Applying this identity to $\{Y_{n+1}\in C_\alpha(X_{n+1};\hat q^{(\mathrm{tar})})\}$ and using~\eqref{eq:target-cell} pathwise yields the final bound.
\end{proof}

\clearpage
\subsection{Experimental details}
\label{app:experimental_details}

\subsubsection{Datasets, splits, and predictors}
\label{app:datasets_splits_predictors}

Table~\ref{tab:datasets} gives the data splits and predictor configurations used in every run.\footnote{The experiments ran on Intel Xeon Gold 6134 CPUs with $32$ logical cores.}
The ten seeds are $42,123,288,327,456,555,690,761,832,999$, and each seed fixes the data split and every seeded fit, including predictor training, \gls{som} initialization, and the $k$-means audit partitions.
Training, selection, calibration, and test rows are disjoint within each dataset and seed.

\begin{table}[ht]
    \centering
    \scriptsize
    \caption{
        Prepared split sizes and predictor configurations.
        Dim. is the dimension of the representation used by the \gls{som} and the other localizers.
        For \gls{gbr}, the predictor column reports the number of estimators, maximum depth, and learning rate.
        For histogram gradient boosting, it reports the maximum number of iterations, maximum leaves, and learning rate.
        For ResNet-18, it reports epochs, batch size, and learning rate.
        ImageNet-1k uses the pretrained torchvision ResNet-50 V2 weights.
    }
    \label{tab:datasets}
    \resizebox{\textwidth}{!}{%
        \begin{tabular}{lrrrrcl}
            \toprule
            Dataset            & Train      & Select    & Calibrate  & Test       & Dim. & Predictor                  \\
            \midrule
            Bike Sharing       & $6{,}531$  & $1{,}088$ & $1{,}632$  & $1{,}635$  & $18$ & \gls{gbr} $500/7/0.1$      \\
            CASP (Bio)         & $27{,}438$ & $4{,}573$ & $6{,}859$  & $6{,}860$  & $9$  & \gls{gbr} $800/7/0.1$      \\
            California Housing & $12{,}384$ & $2{,}064$ & $3{,}096$  & $3{,}096$  & $8$  & \gls{gbr} $500/7/0.1$      \\
            Concrete strength  & $618$      & $103$     & $154$      & $155$      & $8$  & \gls{gbr} $200/3/0.05$     \\
            Auto MPG           & $235$      & $39$      & $58$       & $60$       & $7$  & \gls{gbr} $200/2/0.05$     \\
            CIFAR-10           & $35{,}000$ & $5{,}000$ & $10{,}000$ & $10{,}000$ & $32$ & ResNet-18 $15/256/10^{-3}$ \\
            MNIST              & $42{,}000$ & $6{,}000$ & $12{,}000$ & $10{,}000$ & $32$ & ResNet-18 $15/256/10^{-3}$ \\
            Covertype          & $30{,}000$ & $5{,}000$ & $7{,}500$  & $7{,}500$  & $54$ & HistGB $300/31/0.1$        \\
            CIFAR-100          & $35{,}000$ & $5{,}000$ & $10{,}000$ & $10{,}000$ & $32$ & ResNet-18 $15/256/10^{-3}$ \\
            ImageNet-1k        & $8{,}000$  & $2{,}000$ & $20{,}000$ & $20{,}000$ & $32$ & ResNet-50 V2, pretrained   \\
            \bottomrule
        \end{tabular}%
    }
\end{table}

Except for the fixed ImageNet-1k pretrained network, predictors are fitted on the training split.
CIFAR-10, MNIST, and CIFAR-100 use ImageNet-pretrained ResNet-18 models fine-tuned on their training rows for $15$ epochs~\citep{heDeepResidualLearning2015,russakovskyImageNetLargeScale2015}.
ImageNet-1k uses the fixed ResNet-50 V2 weights.

For image datasets, a training-fitted \gls{pca} projection followed by standardization reduces the penultimate representation to $32$ dimensions.
All tabular inputs, including the regression inputs, are standardized using training statistics.
Regression targets are not standardized, so prediction errors and interval widths retain the response units of each dataset.
Covertype is subsampled to $50{,}000$ rows before splitting to keep kernel calibration for each query tractable.

The selection split is held-out pre-calibration data used to select the \gls{nfcp} training duration and the \gls{raps} and \gls{saps} score parameters.
It is also used to select the \gls{lcp} bandwidth and the Mondrian $k$-means group count, as described below.

Table~\ref{tab:setup_quality} records a preparation snapshot for seed~$42$ with the predictor and representation diagnostics.

\begin{table}[ht]
    \centering
    \scriptsize
    \caption{
        Preparation diagnostics for seed~$42$.
        Regression columns report test RMSE, $R^2$, and the coverage of the fitted $0.05$ to $0.95$ quantile band before conformal calibration.
        Classification reports test top-1 accuracy.
        Empty cells have no training \gls{bmu} assignments.
    }
    \label{tab:setup_quality}
    \resizebox{\textwidth}{!}{%
        \begin{tabular}{llclc}
            \toprule
            Dataset            & Predictor test performance & Raw quantile coverage & Localization representation        & Empty training cells \\
            \midrule
            Bike Sharing       & RMSE $36.67$, $R^2=0.958$  & $0.795$               & standardized features              & $11/63$              \\
            CASP (Bio)         & RMSE $3.73$, $R^2=0.630$   & $0.870$               & standardized features              & $0/90$               \\
            California Housing & RMSE $0.46$, $R^2=0.840$   & $0.847$               & standardized features              & $0/56$               \\
            Concrete strength  & RMSE $5.26$, $R^2=0.900$   & $0.839$               & standardized features              & $1/30$               \\
            Auto MPG           & RMSE $2.84$, $R^2=0.854$   & $0.750$               & standardized features              & $1/20$               \\
            CIFAR-10           & Top-1 $0.812$              & n/a                   & train-fitted \gls{pca}, $32$ dims. & $6/156$              \\
            MNIST              & Top-1 $0.994$              & n/a                   & train-fitted \gls{pca}, $32$ dims. & $26/400$             \\
            Covertype          & Top-1 $0.813$              & n/a                   & standardized features              & $55/130$             \\
            CIFAR-100          & Top-1 $0.491$              & n/a                   & train-fitted \gls{pca}, $32$ dims. & $0/156$              \\
            ImageNet-1k        & Top-1 $0.808$              & n/a                   & train-fitted \gls{pca}, $32$ dims. & $1/156$              \\
            \bottomrule
        \end{tabular}%
    }
\end{table}

\subsubsection{Localization and calibration protocol}
\label{app:localization_calibration_protocol}

The \gls{som} and every learned localization transform are fitted on training inputs.
Table~\ref{tab:som_config} separates the map configuration from the predictor choices above.
Regime~2 uses the listed radius.
For Concrete and Auto MPG, Regime~3 starts from the same radius and adds the nearest cells in prototype space according to Equation~\eqref{eq:training-routed-budget}.

\begin{table}[ht]
    \centering
    \scriptsize
    \caption{
        \gls{som} and retrieval configurations.
        The map columns give grid size, epochs $E$, batch size $B$, initial learning rate $\eta_0$, initial neighborhood scale $\sigma_0$, and grid neighborhood radius $r$.
        The final column gives the retrieval regime used in the headline comparison.
    }
    \label{tab:som_config}
    \resizebox{\textwidth}{!}{%
        \begin{tabular}{lcrrrrrl}
            \toprule
            Dataset            & Grid         & $E$   & $B$   & $\eta_0$ & $\sigma_0$ & $r$ & Retrieval \\
            \midrule
            Bike Sharing       & $7\times9$   & $100$ & $128$ & $0.95$   & $1.10$     & $2$ & Regime~2  \\
            CASP (Bio)         & $9\times10$  & $100$ & $256$ & $0.95$   & $1.40$     & $2$ & Regime~2  \\
            California Housing & $7\times8$   & $100$ & $128$ & $0.95$   & $1.10$     & $2$ & Regime~2  \\
            Concrete strength  & $5\times6$   & $50$  & $64$  & $0.85$   & $1.00$     & $1$ & Regime~3  \\
            Auto MPG           & $4\times5$   & $50$  & $32$  & $0.85$   & $0.90$     & $1$ & Regime~3  \\
            CIFAR-10           & $12\times13$ & $100$ & $128$ & $0.95$   & $2.30$     & $3$ & Regime~2  \\
            MNIST              & $20\times20$ & $100$ & $128$ & $0.95$   & $3.75$     & $4$ & Regime~2  \\
            Covertype          & $10\times13$ & $100$ & $256$ & $0.95$   & $1.80$     & $2$ & Regime~2  \\
            CIFAR-100          & $12\times13$ & $100$ & $128$ & $0.95$   & $2.30$     & $3$ & Regime~2  \\
            ImageNet-1k        & $12\times13$ & $100$ & $128$ & $0.95$   & $2.30$     & $3$ & Regime~2  \\
            \bottomrule
        \end{tabular}%
    }
\end{table}

Concrete and Auto MPG are retained under both Regime~2 and Regime~3.
The headline comparison uses Regime~3, while analyses that require a common retrieval rule across all datasets use Regime~2.
At $\alpha=0.1$, nine scores suffice for a finite cutoff, while $19$ is the smallest buffer size whose conformal quantile is not the maximum score, since $\ell_{19}=18<19$.
Thus $b=19$ is a projected efficiency target, and the realized buffer may contain fewer than $19$ calibration samples.

Every \gls{som} uses \texttt{torchsom}~\citep{berthierTorchsomReferencePyTorch2025}, with seeded random initialization, a Gaussian neighborhood function, Euclidean prototype distance, inverse decay of the learning rate and neighborhood scale, rectangular topology, and no periodic boundary condition.
The \gls{lcp} baseline uses the corrected localized level of~\citet{guanLocalizedConformalPrediction2022a} with a Laplace kernel.
Its bandwidth for each score is selected on the selection split using the authors' grid of $20$ values and resampling objective.
For classification, this objective uses prediction set size computed with the same score and the full label set in place of interval length and the real line.
The \gls{rlcp} baseline uses a Gaussian random anchor and a bandwidth rule, based on training data and the planned calibration size, that targets an effective calibration sample size of $50$.
Mondrian $k$-means fits every candidate group count in $\{5,10,15,20,25,30\}$ on the training inputs and selects the one with the highest silhouette score on the selection inputs.
Classwise calibration estimates a separate threshold for each label from the calibration examples carrying that label.
Following~\citet{dingClassConditionalConformalPrediction2023a}, Clustered calibration groups labels on an auxiliary calibration subset, then estimates one threshold per group on the remaining calibration data.

\subsubsection{Nonconformity scores and methods}
\label{app:nonconformity_scores_methods}

For regression, the absolute residual, \gls{nfcp}, and \gls{rcp} scores use the same fitted point predictor $\hat\mu$.
The absolute residual score is
\[
    s_{\mathrm{abs}}(x,y)
    \coloneqq
    |y-\hat\mu(x)|.
\]
The other two scores transform this residual as
\begin{align}
    s_{\mathrm{nfcp}}(x,y)
     & \coloneqq
    \frac{s_{\mathrm{abs}}(x,y)+\delta}{\gamma+\hat g(x)},
    \\
    s_{\mathrm{rcp}}(x,y)
     & \coloneqq
    \frac{s_{\mathrm{abs}}(x,y)}{\hat\tau(x)}.
\end{align}
Here, $\hat g(x)\geq0$ is the \gls{nfcp} scale model and $\delta,\gamma>0$ are numerical stabilizers fixed before calibration~\citep{colomboNormalizingFlowsConformal2024}.
For \gls{rcp}, $\hat\tau(x)>0$ estimates the conditional $(1-\alpha)$-quantile of the absolute residual~\citep{plassierRectifyingConformityScores2025}.
Both $\hat g$ and $\hat\tau$ are fitted from out-of-fold training residuals.
The number of \gls{nfcp} training updates is selected on the selection split.

The remaining regression scores use fitted distributional models.
For \gls{cqr},
\[
    s_{\mathrm{cqr}}(x,y)
    \coloneqq
    \max\{\hat\ell(x)-y,\ y-\hat u(x)\},
\]
where $\hat\ell$ and $\hat u$ are the fitted $0.05$ and $0.95$ conditional quantiles~\citep{romanoConformalizedQuantileRegression2019b}.
For \gls{chr}, a fitted conditional histogram defines randomized nested intervals $\{I_t(x,U)\}_{t=0}^{T}$~\citep{sesiaConformalHistogramRegression2021}.
The score is the first level containing the response,
\[
    s_{\mathrm{chr}}(x,y,U)
    \coloneqq
    \min\bigl\{t\in\{0,\dots,T\}:y\in I_t(x,U)\bigr\}.
\]

For classification, let $C=|\mathcal{Y}|$ and write $\hat p_c(x)$ for the predicted probability of class $c$.
For the scores that rank labels (\gls{aps}, \gls{raps}, \gls{saps}, and rank), let $\pi_x(1),\dots,\pi_x(C)$ order the classes by decreasing predicted probability, and let $r_x(y)$ be the resulting rank of label $y$.
The randomized classification scores below use the same kind of auxiliary draw $U$, shared across all candidate labels of an observation.
The classification scores are
\begin{align}
    s_{\mathrm{lac}}(x,y)
     & \coloneqq 1-\hat p_y(x),
    \\
    s_{\mathrm{aps}}(x,y,U)
     & \coloneqq
    \sum_{j<r_x(y)}\hat p_{\pi_x(j)}(x)+U\hat p_y(x),
    \\
    s_{\mathrm{raps}}(x,y,U)
     & \coloneqq
    s_{\mathrm{aps}}(x,y,U)
    +\lambda_{\mathrm R}\bigl(r_x(y)-k_{\mathrm R}\bigr)_+.
\end{align}
These definitions recover \gls{lac}, \gls{aps}, and \gls{raps}~\citep{sadinleLeastAmbiguousSetvalued2019,romanoClassificationValidAdaptive2020,angelopoulosUncertaintySetsImage2021}.
The \gls{aps} score is the cumulative probability before $y$, plus a uniformly random fraction of $\hat p_y(x)$.
For \gls{saps},
\begin{align}
    s_{\mathrm{saps}}(x,y,U)
     & \coloneqq
    \begin{cases}
        U\hat p_{\pi_x(1)}(x),
         & r_x(y)=1,    \\
        \hat p_{\pi_x(1)}(x)
        +\bigl(r_x(y)-2+U\bigr)\lambda_{\mathrm S},
         & r_x(y)\geq2,
    \end{cases}
    \\
    s_{\mathrm{rank}}(x,y,U)
     & \coloneqq r_x(y)-1+U.
\end{align}
The parameters $(\lambda_{\mathrm R},k_{\mathrm R})$ and $\lambda_{\mathrm S}$ are selected on the selection split~\citep{angelopoulosUncertaintySetsImage2021,huangConformalPredictionDeep2024}.

Every score is fixed before calibration, and larger values indicate greater nonconformity.
Given a threshold $q$, the prediction set retains every candidate $y$ such that $s(x,y)\leq q$.
Pooled calibration computes $q$ from all calibration scores.
SO calibration computes the same split-conformal quantile from the buffer retrieved by the query's \gls{bmu}.
The composition study evaluates each score with both pooled and SO calibration.
In addition to the composition study, we compare calibrators using the absolute residual score for regression and \gls{lac} for classification.
For classification, Classwise uses one cutoff per label and Clustered shares cutoffs within learned label clusters.
This design changes the calibration distribution or grouping while leaving the underlying score unchanged.

We report wall time after preparing the data, predictors, and learned score models.
For each dataset and seed, the timed evaluation covers all $25$ score and calibrator configurations in regression or all $35$ in classification, together with their diagnostics.
On California Housing, Concrete, and Auto MPG, it also includes the sensitivity analyses of Appendix~\ref{app:sensitivity_small_buffers}.
Five datasets and ten seeds give $50$ timed evaluations for each task.
Regression evaluations averaged $31$~min and ranged from less than one minute to approximately $2$~h.
Classification evaluations averaged approximately $6$~h $30$~min and ranged from about $1$~h $30$~min to $10$~h $30$~min.

\clearpage
\subsection{Detailed empirical results}
\label{app:detailed_results}

We report ten-seed results at target coverage $0.90$ under the protocol in Section~\ref{app:experimental_details}.
Unless stated otherwise, SO calibration uses Regime~2 and WCovGap uses the shared external $k$-means audit with $K=25$.

The headline comparison reports marginal coverage, WCovGap, and output size (Section~\ref{app:headline_comparison}).
Then, we evaluate SO calibration across scores (Section~\ref{app:score_composition}) and compare calibrators at fixed base scores (Section~\ref{app:calibration_tradeoffs}).
Robustness checks vary the audit partition (Section~\ref{app:partition_transfer_diagnostics}) and use diagnostics that need no feature-space partition (Section~\ref{app:partition_free_diagnostics}).
Finally, we examine map resolution, retrieval radius, and calibration size, including Regime~3 for small calibration buffers (Section~\ref{app:sensitivity_small_buffers}).

\subsubsection{Headline comparison}
\label{app:headline_comparison}

Tables~\ref{tab:cr_main_regression} and~\ref{tab:cr_main_classification} report the mean metrics behind Figure~\ref{fig:cr_tradeoff}.
Two-sided Wilcoxon signed-rank tests use paired WCovGap differences, with Holm correction across ten datasets.

\paragraph{Regression.}
Table~\ref{tab:cr_main_regression} uses the absolute residual score.
Concrete and Auto MPG use training-routed Regime~3 with $b=19$ because their calibration splits contain only $154$ and $58$ observations.
The other datasets use Regime~2.

\begin{table}[ht]
    \centering
    \caption{
        Regression headline results with the absolute residual score, averaged across ten seeds.
        WCovGap uses the external $K=25$ audit and is in percentage points.
        $\Delta$ gives the relative change between the two means (\%), and Holm $p$ gives the adjusted Wilcoxon $p$-value.
        Daggers mark Regime~3 with $b=19$, while other rows use Regime~2.
    }
    \label{tab:cr_main_regression}
    \resizebox{\textwidth}{!}{%
        \begin{tabular}{lccccccccc}
\toprule
 & \multicolumn{2}{c}{Coverage} & \multicolumn{3}{c}{Width} & \multicolumn{3}{c}{External WCovGap ($\downarrow$)} & Holm $p$ \\
 & SCP & SO-SCP & SCP & SO-SCP & $\Delta$ (\%) & SCP & SO-SCP & $\Delta$ (\%) &  \\
\midrule
Bike & $0.902$ & $0.902$ & $115.36$ & $117.44$ & $+1.8$ & $5.46$ & $\mathbf{4.75}$ & $-13.0$ & $0.098$ \\
Bio & $0.902$ & $0.906$ & $12.74$ & $12.79$ & $+0.4$ & $3.14$ & $\mathbf{2.33}$ & $-25.8$ & $0.020$ \\
California Housing & $0.898$ & $0.899$ & $1.39$ & $1.36$ & $-2.2$ & $3.50$ & $\mathbf{3.15}$ & $-10.0$ & $0.098$ \\
Concrete$^{\dagger}$ & $0.908$ & $0.914$ & $16.57$ & $18.68$ & $+12.7$ & $10.40$ & $\mathbf{10.12}$ & $-2.7$ & $1.000$ \\
Auto MPG$^{\dagger}$ & $0.928$ & $0.922$ & $10.25$ & $10.39$ & $+1.4$ & $12.87$ & $\mathbf{12.73}$ & $-1.1$ & $1.000$ \\
\bottomrule
\end{tabular}
    }
\end{table}

On Bike Sharing, Bio, and California Housing, SO-SCP reduces WCovGap by $10.0\%$ to $25.8\%$, with mean width changes between $-2.2\%$ and $+1.8\%$.
Figure~\ref{fig:cr_tradeoff} shows reductions in $26$ of the $30$ corresponding seed comparisons, and Bio remains significant after Holm correction.
Concrete and Auto MPG have smaller mean WCovGap reductions, $2.7\%$ and $1.1\%$, with improvements in only $4$ of $10$ seeds each.
Regime~3 produces finite intervals on both datasets, but Concrete's mean width rises from $16.57$ to $18.68$ target units.

\paragraph{Classification.}
Table~\ref{tab:cr_main_classification} reports the \gls{lac} comparison under Regime~2.

\begin{table}[ht]
    \centering
    \caption{
        Classification headline means across ten seeds with \gls{lac} and Regime~2.
        WCovGap is in percentage points at $K=25$, and set size counts retained labels.
        $\Delta$ and Holm $p$ follow Table~\ref{tab:cr_main_regression}.
    }
    \label{tab:cr_main_classification}
    \resizebox{\textwidth}{!}{%
        \begin{tabular}{lccccccccc}
\toprule
 & \multicolumn{2}{c}{Coverage} & \multicolumn{3}{c}{Set size} & \multicolumn{3}{c}{External WCovGap ($\downarrow$)} & Holm $p$ \\
 & SCP & SO-SCP & SCP & SO-SCP & $\Delta$ (\%) & SCP & SO-SCP & $\Delta$ (\%) &  \\
\midrule
CIFAR-10 & $0.898$ & $0.901$ & $1.42$ & $1.46$ & $+2.8$ & $4.93$ & $\mathbf{4.54}$ & $-7.9$ & $0.098$ \\
MNIST & $0.901$ & $0.908$ & $0.90$ & $0.91$ & $+1.1$ & $8.28$ & $\mathbf{6.84}$ & $-17.4$ & $0.020$ \\
Covertype & $0.898$ & $0.899$ & $1.20$ & $1.22$ & $+1.7$ & $2.79$ & $\mathbf{1.95}$ & $-30.1$ & $0.020$ \\
CIFAR-100 & $0.901$ & $0.904$ & $14.25$ & $14.67$ & $+2.9$ & $3.22$ & $\mathbf{2.58}$ & $-19.9$ & $0.020$ \\
ImageNet-1k & $0.899$ & $0.898$ & $1.48$ & $1.59$ & $+7.4$ & $5.70$ & $\mathbf{4.56}$ & $-20.0$ & $0.020$ \\
\bottomrule
\end{tabular}
    }
\end{table}

SO-SCP lowers WCovGap by $19.1\%$ on average, with $3.2\%$ larger sets.
Figure~\ref{fig:cr_tradeoff} records reductions in $48$ of $50$ seed comparisons, and every dataset except CIFAR-10 remains significant after Holm correction.
Four mean set sizes stay below two labels, with at most $0.11$ additional labels.
On CIFAR-100, the mean set size increases from $14.25$ to $14.67$ labels.

Across Tables~\ref{tab:cr_main_regression} and~\ref{tab:cr_main_classification}, SO-SCP reduces WCovGap by $14.8\%$ on average, with a mean output-size increase of $3.0\%$.
WCovGap decreases on all eight Regime~2 datasets, by $18.0\%$ on average for a $2.0\%$ mean increase in output size.
These gains do not come at the cost of marginal coverage: both methods' mean marginal coverages remain within $0.008$ of $0.90$ on these eight datasets.

\subsubsection{Composition across scores}
\label{app:score_composition}

We evaluate each score from Section~\ref{app:nonconformity_scores_methods} with both pooled and SO calibration.
Each comparison holds the predictor and score fixed and uses the same Regime~2 retrieval rule.
Tables~\ref{tab:cr_factorial_regression} and~\ref{tab:cr_factorial_classification} report relative changes between pooled and SO mean WCovGap values.

\paragraph{Regression.}
Figure~\ref{fig:cr_composition_regression} shows how SO calibration changes WCovGap for the five regression scores.
Table~\ref{tab:cr_factorial_regression} gives the corresponding values and relative changes.

\begin{figure}[ht]
    \centering
    \includegraphics[width=0.99\textwidth]{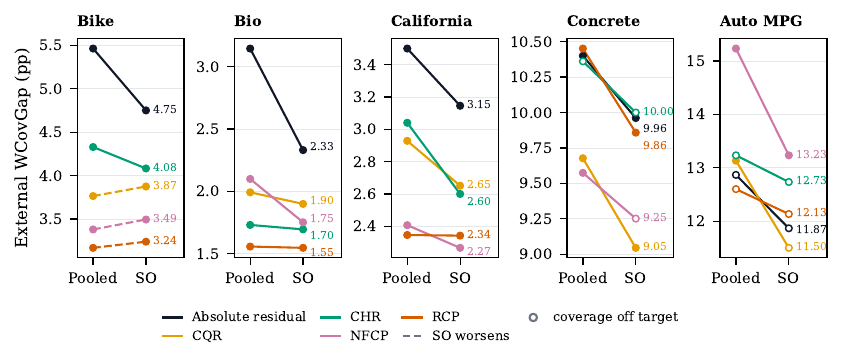}
    \caption{
        Pooled-to-SO changes in mean external WCovGap across ten seeds for regression scores, using Regime~2 and $K=25$.
        Dashed segments mark increases.
        Open endpoints mark mean marginal coverage more than $0.02$ from $0.90$.
    }
    \label{fig:cr_composition_regression}
\end{figure}

\begin{table}[ht]
    \centering
    \caption{
        Mean external WCovGap in percentage points across ten seeds for regression scores, using Regime~2 and $K=25$.
        Brackets give the relative change from pooled calibration (\%), and the lower value in each pair is bold.
        Daggers mark mean marginal coverage more than $0.02$ from $0.90$.
    }
    \label{tab:cr_factorial_regression}
    \setlength{\tabcolsep}{4pt}
    \resizebox{\textwidth}{!}{%
        \begin{tabular}{lcccccccccc}
\toprule
 & \multicolumn{2}{c}{Absolute residual} & \multicolumn{2}{c}{CQR} & \multicolumn{2}{c}{CHR} & \multicolumn{2}{c}{NFCP} & \multicolumn{2}{c}{RCP} \\
 & Pooled & SO & Pooled & SO & Pooled & SO & Pooled & SO & Pooled & SO \\
\midrule
Bike & $5.46$ & $\mathbf{4.75}$ $[-13.0\%]$ & $\mathbf{3.76}$ & $3.87$ $[+2.9\%]$ & $4.33$ & $\mathbf{4.08}$ $[-5.8\%]$ & $\mathbf{3.38}$ & $3.49$ $[+3.3\%]$ & $\mathbf{3.17}$ & $3.24$ $[+2.2\%]$ \\
Bio & $3.14$ & $\mathbf{2.33}$ $[-25.8\%]$ & $1.99$ & $\mathbf{1.90}$ $[-4.5\%]$ & $1.73$ & $\mathbf{1.70}$ $[-1.7\%]$ & $2.10$ & $\mathbf{1.75}$ $[-16.7\%]$ & $1.56$ & $\mathbf{1.55}$ $[-0.6\%]$ \\
California Housing & $3.50$ & $\mathbf{3.15}$ $[-10.0\%]$ & $2.93$ & $\mathbf{2.65}$ $[-9.6\%]$ & $3.04$ & $\mathbf{2.60}$ $[-14.5\%]$ & $2.41$ & $\mathbf{2.27}$ $[-5.8\%]$ & $2.35$ & $\mathbf{2.34}$ $[-0.4\%]$ \\
Concrete & $10.40$ & $\mathbf{9.96}$ $[-4.2\%]$ & $9.68$ & $\mathbf{9.05}$ $[-6.5\%]$ & $10.36$$^{\dagger}$ & $\mathbf{10.00}$$^{\dagger}$ $[-3.5\%]$ & $9.57$ & $\mathbf{9.25}$$^{\dagger}$ $[-3.3\%]$ & $10.45$ & $\mathbf{9.86}$ $[-5.6\%]$ \\
Auto MPG & $12.87$$^{\dagger}$ & $\mathbf{11.87}$$^{\dagger}$ $[-7.8\%]$ & $13.13$ & $\mathbf{11.50}$$^{\dagger}$ $[-12.4\%]$ & $13.23$$^{\dagger}$ & $\mathbf{12.73}$$^{\dagger}$ $[-3.8\%]$ & $15.23$ & $\mathbf{13.23}$ $[-13.1\%]$ & $12.60$$^{\dagger}$ & $\mathbf{12.13}$$^{\dagger}$ $[-3.7\%]$ \\
\bottomrule
\end{tabular}
    }
\end{table}

SO calibration lowers WCovGap in $22$ of the $25$ regression comparisons in Table~\ref{tab:cr_factorial_regression}.
The absolute residual improves on every dataset, with relative reductions of $4.2\%$ to $25.8\%$.
The largest absolute reduction is $2.00$ points ($13.1\%$) for \gls{nfcp} on Auto MPG.
The three increases occur on Bike Sharing for \gls{cqr}, \gls{nfcp}, and \gls{rcp}, which have the three lowest pooled WCovGap values on that dataset.
They amount to $0.07$ to $0.11$ points, or $2.2\%$ to $3.3\%$.

All ten daggered entries concern Concrete or Auto MPG, with four under pooled calibration and six under SO calibration.
Each indicates overcoverage, with mean marginal coverage above $0.92$ for a target of $0.90$.
Small buffers can make conformal cutoffs conservative, and the local Regime~2 buffers contain fewer scores than the full calibration split.
The score comparison retains Regime~2 on these two datasets, where the headline Table~\ref{tab:cr_main_regression} uses Regime~3.

\clearpage
\paragraph{Classification.}
Figure~\ref{fig:cr_composition_classification} and Table~\ref{tab:cr_factorial_classification} report the same comparison for the five classification scores.

\begin{figure}[ht]
    \centering
    \includegraphics[width=0.99\textwidth]{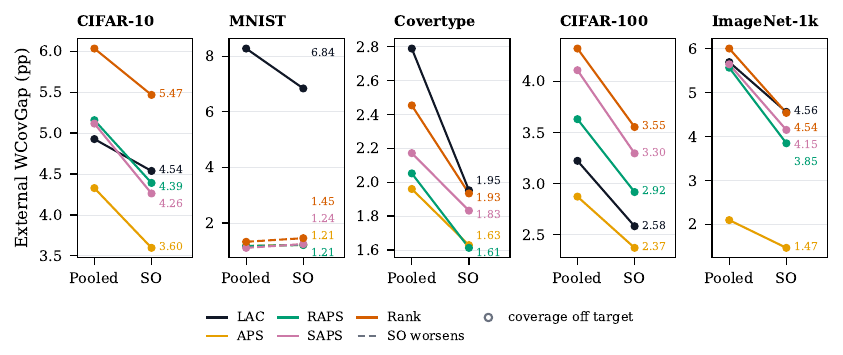}
    \caption{
        Pooled-to-SO changes in mean external WCovGap across ten seeds for classification scores, using Regime~2 and $K=25$.
        Dashed segments mark increases.
    }
    \label{fig:cr_composition_classification}
\end{figure}

\begin{table}[ht]
    \centering
    \caption{
        Mean external WCovGap in percentage points across ten seeds for classification scores, using Regime~2 and $K=25$.
        Brackets give the relative change from pooled calibration (\%), and the lower value in each pair is bold.
    }
    \label{tab:cr_factorial_classification}
    \setlength{\tabcolsep}{4pt}
    \resizebox{\textwidth}{!}{%
        \begin{tabular}{lcccccccccc}
\toprule
 & \multicolumn{2}{c}{LAC} & \multicolumn{2}{c}{APS} & \multicolumn{2}{c}{RAPS} & \multicolumn{2}{c}{SAPS} & \multicolumn{2}{c}{Rank} \\
 & Pooled & SO & Pooled & SO & Pooled & SO & Pooled & SO & Pooled & SO \\
\midrule
CIFAR-10 & $4.93$ & $\mathbf{4.54}$ $[-7.9\%]$ & $4.33$ & $\mathbf{3.60}$ $[-16.9\%]$ & $5.16$ & $\mathbf{4.39}$ $[-14.9\%]$ & $5.12$ & $\mathbf{4.26}$ $[-16.8\%]$ & $6.03$ & $\mathbf{5.47}$ $[-9.3\%]$ \\
MNIST & $8.28$ & $\mathbf{6.84}$ $[-17.4\%]$ & $\mathbf{1.17}$ & $1.21$ $[+3.4\%]$ & $\mathbf{1.17}$ & $1.21$ $[+3.4\%]$ & $\mathbf{1.10}$ & $1.24$ $[+12.7\%]$ & $\mathbf{1.32}$ & $1.45$ $[+9.8\%]$ \\
Covertype & $2.79$ & $\mathbf{1.95}$ $[-30.1\%]$ & $1.96$ & $\mathbf{1.63}$ $[-16.8\%]$ & $2.05$ & $\mathbf{1.61}$ $[-21.5\%]$ & $2.17$ & $\mathbf{1.83}$ $[-15.7\%]$ & $2.45$ & $\mathbf{1.93}$ $[-21.2\%]$ \\
CIFAR-100 & $3.22$ & $\mathbf{2.58}$ $[-19.9\%]$ & $2.87$ & $\mathbf{2.37}$ $[-17.4\%]$ & $3.63$ & $\mathbf{2.92}$ $[-19.6\%]$ & $4.11$ & $\mathbf{3.30}$ $[-19.7\%]$ & $4.32$ & $\mathbf{3.55}$ $[-17.8\%]$ \\
ImageNet-1k & $5.70$ & $\mathbf{4.56}$ $[-20.0\%]$ & $2.10$ & $\mathbf{1.47}$ $[-30.0\%]$ & $5.57$ & $\mathbf{3.85}$ $[-30.9\%]$ & $5.65$ & $\mathbf{4.15}$ $[-26.5\%]$ & $6.01$ & $\mathbf{4.54}$ $[-24.5\%]$ \\
\bottomrule
\end{tabular}
    }
\end{table}

SO calibration lowers WCovGap in $21$ of the $25$ classification comparisons in Table~\ref{tab:cr_factorial_classification}.
\gls{lac} improves on every dataset, with relative reductions of $7.9\%$ to $30.1\%$.
The largest absolute reduction is $1.72$ points ($30.9\%$) for \gls{raps} on ImageNet-1k.
All four increases occur on MNIST for the scores with the lowest pooled gaps.
They reach $12.7\%$ in relative terms but remain at most $0.14$ percentage points, while \gls{lac} improves by $17.4\%$ ($1.44$ points).
Every classification entry has mean marginal coverage within $0.02$ of the target.

Across Tables~\ref{tab:cr_factorial_regression} and~\ref{tab:cr_factorial_classification}, the $43$ reductions average $13.7\%$ ($0.66$ percentage points), while the seven increases average $5.4\%$ ($0.09$ points).
Every score improves when its WCovGap is averaged over the five datasets.
Overall, reductions are more frequent than increases and larger on average in both relative and absolute terms, while the small increases concern the scores with the lowest pooled WCovGap on each affected dataset.

\subsubsection{Comparison of calibrators}
\label{app:calibration_tradeoffs}

The composition study in Section~\ref{app:score_composition} varies the score.
Now, we hold the base score fixed, using the absolute residual for regression and \gls{lac} for classification, and compare calibration rules under the protocol in Section~\ref{app:localization_calibration_protocol}.
SO-SCP uses Regime~2 on every dataset.
Gain denotes the reduction in mean WCovGap relative to \gls{scp}, in percentage points.
Output changes compare the ten-seed mean widths or set sizes.
Larger gains and smaller outputs are preferable.

\paragraph{Regression.}
Table~\ref{tab:cr_calibrator_tradeoff_regression} reports WCovGap gains, interval widths, and the frequency of unbounded intervals.

\begin{table}[ht]
    \centering
    \caption{
        Regression calibrator comparison with the absolute residual score, Regime~2 SO-SCP, and external $K=25$ audit.
        Gain is the WCovGap reduction from \gls{scp} in percentage points, averaged over ten seeds.
        Mean widths average ten seeds in dataset target units, with relative changes from \gls{scp} in brackets.
        Median widths and counts of infinite-width intervals are computed from all test predictions pooled across seeds.
    }
    \label{tab:cr_calibrator_tradeoff_regression}
    \resizebox{\textwidth}{!}{%
        \begin{tabular}{llcccc}
\toprule
 &  & Gain & Mean width [$\Delta\%$] & Median width & Infinite widths \\
\midrule
\multirow[t]{5}{*}{Bike} & SCP & $0.00$ & $115.36\;[0.0\%]$ & $115.83$ & $0/16{,}350\;(0.00\%)$ \\
 & LCP & $+0.60$ & $111.30\;[-3.5\%]$ & $111.93$ & $0/16{,}350\;(0.00\%)$ \\
 & RLCP & $+2.01$ & $\infty$ & $112.43$ & $556/16{,}350\;(3.40\%)$ \\
 & Mondrian k-means & $+1.21$ & $119.37\;[+3.5\%]$ & $115.87$ & $0/16{,}350\;(0.00\%)$ \\
 & \textbf{SO-SCP} & $+0.71$ & $117.44\;[+1.8\%]$ & $115.87$ & $0/16{,}350\;(0.00\%)$ \\
\cline{1-6}
\multirow[t]{5}{*}{Bio} & SCP & $0.00$ & $12.74\;[0.0\%]$ & $12.73$ & $0/68{,}600\;(0.00\%)$ \\
 & LCP & $+0.38$ & $\infty$ & $12.42$ & $531/68{,}600\;(0.77\%)$ \\
 & RLCP & $+1.08$ & $\infty$ & $12.85$ & $3{,}774/68{,}600\;(5.50\%)$ \\
 & Mondrian k-means & $+0.31$ & $\infty$ & $13.22$ & $37/68{,}600\;(0.05\%)$ \\
 & \textbf{SO-SCP} & $+0.81$ & $12.79\;[+0.4\%]$ & $12.74$ & $0/68{,}600\;(0.00\%)$ \\
\cline{1-6}
\multirow[t]{5}{*}{California Housing} & SCP & $0.00$ & $1.39\;[0.0\%]$ & $1.38$ & $0/30{,}960\;(0.00\%)$ \\
 & LCP & $+0.26$ & $\infty$ & $1.33$ & $455/30{,}960\;(1.47\%)$ \\
 & RLCP & $+0.88$ & $\infty$ & $1.36$ & $1{,}151/30{,}960\;(3.72\%)$ \\
 & Mondrian k-means & $+0.07$ & $\infty$ & $1.38$ & $12/30{,}960\;(0.04\%)$ \\
 & \textbf{SO-SCP} & $+0.35$ & $1.36\;[-2.2\%]$ & $1.36$ & $0/30{,}960\;(0.00\%)$ \\
\cline{1-6}
\multirow[t]{5}{*}{Concrete} & SCP & $0.00$ & $16.57\;[0.0\%]$ & $16.26$ & $0/1{,}550\;(0.00\%)$ \\
 & LCP & $-0.55$ & $\infty$ & $16.18$ & $12/1{,}550\;(0.77\%)$ \\
 & RLCP & $+0.32$ & $\infty$ & $17.05$ & $30/1{,}550\;(1.94\%)$ \\
 & Mondrian k-means & $+0.76$ & $\infty$ & $18.40$ & $257/1{,}550\;(16.58\%)$ \\
 & \textbf{SO-SCP} & $+0.44$ & $\infty$ & $17.52$ & $3/1{,}550\;(0.19\%)$ \\
\cline{1-6}
\multirow[t]{5}{*}{Auto MPG} & SCP & $0.00$ & $10.25\;[0.0\%]$ & $9.89$ & $0/600\;(0.00\%)$ \\
 & LCP & $-0.17$ & $10.02\;[-2.2\%]$ & $9.87$ & $0/600\;(0.00\%)$ \\
 & RLCP & $0.00$ & $10.25\;[0.0\%]$ & $9.89$ & $0/600\;(0.00\%)$ \\
 & Mondrian k-means & $+0.77$ & $\infty$ & $11.63$ & $171/600\;(28.50\%)$ \\
 & \textbf{SO-SCP} & $+1.00$ & $\infty$ & $10.49$ & $61/600\;(10.17\%)$ \\
\cline{1-6}
\bottomrule
\end{tabular}
    }
\end{table}

Among the localized calibrators in Table~\ref{tab:cr_calibrator_tradeoff_regression}, only SO-SCP keeps all $115{,}910$ recorded intervals finite across Bike Sharing, Bio, and California Housing.
Its WCovGap gains of $0.35$ to $0.81$ points come with mean width changes between $-2.2\%$ and $+1.8\%$, preserving the interval scale of \gls{scp}.
\gls{rlcp} achieves larger gains on these datasets, but returns unbounded intervals at rates of $3.40\%$ to $5.50\%$.
Those intervals provide no finite bound on the response, a cost neither \gls{scp} nor SO-SCP incurs on these datasets.
Both \gls{lcp} and Mondrian $k$-means also return unbounded intervals on Bio and California Housing.
A finite median can hide this behavior, as illustrated by \gls{rlcp} on Bio, where the median is $12.85$ but $5.50\%$ of intervals are unbounded.

Figure~\ref{fig:cr_calibrator_tradeoff_regression} plots the gain against the relative change in mean width.

\begin{figure}[ht]
    \centering
    \includegraphics[width=0.99\textwidth]{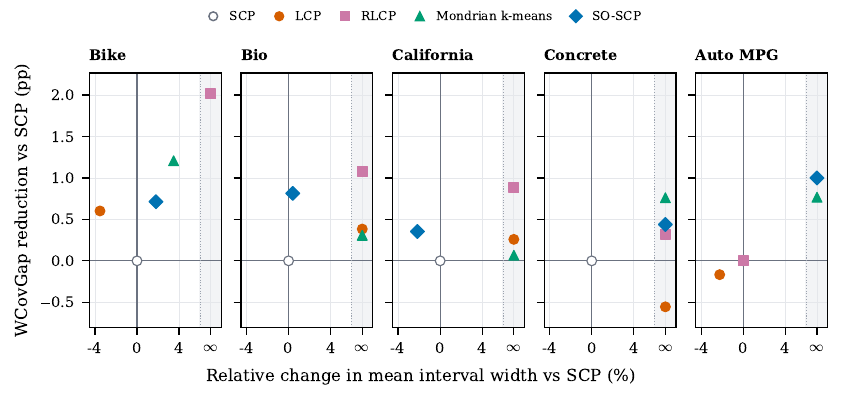}
    \caption{
        WCovGap gain versus relative change in mean width for Table~\ref{tab:cr_calibrator_tradeoff_regression}.
        \gls{scp} is at the origin and SO-SCP is marked by a diamond.
        The shaded rail contains calibrators with infinite mean width.
    }
    \label{fig:cr_calibrator_tradeoff_regression}
\end{figure}

Bike Sharing offers competitive finite alternatives.
\gls{lcp} gains $0.60$ points with $3.5\%$ narrower intervals, compared with SO-SCP's $0.71$ points at a $1.8\%$ width increase.
Mondrian $k$-means gains more, $1.21$ points, with a larger width increase of $3.5\%$.
SO-SCP offers an intermediate trade-off on Bike Sharing and avoids the unbounded outputs of the other localized methods on Bio and California Housing.

Concrete and Auto MPG, with $154$ and $58$ calibration observations, expose the limitation of fixed Regime~2 buffers.
SO-SCP returns unbounded intervals for $0.19\%$ and $10.17\%$ of queries, respectively, compared with $16.58\%$ and $28.50\%$ for Mondrian $k$-means.
\gls{lcp} worsens WCovGap on both datasets, by $0.55$ and $0.17$ points, respectively.
\gls{rlcp} gains only $0.32$ points on Concrete, where both kernel methods return unbounded intervals.
On Auto MPG, both kernel methods remain finite, but neither improves WCovGap, and \gls{rlcp} matches \gls{scp} on every reported quantity.
Regime~3, used for these two datasets in Table~\ref{tab:cr_main_regression}, removes the observed infinite cutoffs throughout the calibration-size experiment in Figure~\ref{fig:cr_regime3}.

\paragraph{Classification.}
Table~\ref{tab:cr_calibrator_tradeoff_classification} and Figure~\ref{fig:cr_calibrator_tradeoff_classification} compare WCovGap gains and prediction set sizes with the \gls{lac} score.

\begin{table}[ht]
    \centering
    \caption{
        Classification calibrator comparison with \gls{lac}, Regime~2 SO-SCP, and external $K=25$ audit, averaged across ten seeds.
        Gain is the WCovGap reduction from \gls{scp} in percentage points.
        Output is mean set size, with its relative change from \gls{scp} in brackets.
    }
    \label{tab:cr_calibrator_tradeoff_classification}
    \resizebox{\textwidth}{!}{%
        \begin{tabular}{lcccccccccc}
\toprule
 & \multicolumn{2}{c}{CIFAR-10} & \multicolumn{2}{c}{MNIST} & \multicolumn{2}{c}{Covertype} & \multicolumn{2}{c}{CIFAR-100} & \multicolumn{2}{c}{ImageNet-1k} \\
 & Gain & Output [$\Delta\%$] & Gain & Output [$\Delta\%$] & Gain & Output [$\Delta\%$] & Gain & Output [$\Delta\%$] & Gain & Output [$\Delta\%$] \\
\midrule
SCP & $0.00$ & $1.42\;[0.0\%]$ & $0.00$ & $0.90\;[0.0\%]$ & $0.00$ & $1.20\;[0.0\%]$ & $0.00$ & $14.25\;[0.0\%]$ & $0.00$ & $1.48\;[0.0\%]$ \\
Classwise & $+0.21$ & $1.46\;[+2.8\%]$ & $+1.78$ & $0.90\;[0.0\%]$ & $+0.56$ & $1.68\;[+40.0\%]$ & $+1.07$ & $15.88\;[+11.4\%]$ & $+2.41$ & $11.82\;[+699\%]$ \\
Clustered & $0.00$ & $1.42\;[0.0\%]$ & $-0.01$ & $0.90\;[0.0\%]$ & $+0.13$ & $1.21\;[+0.8\%]$ & $+0.76$ & $14.91\;[+4.6\%]$ & $+0.90$ & $1.81\;[+22.3\%]$ \\
LCP & $+0.05$ & $1.42\;[0.0\%]$ & $+0.04$ & $0.90\;[0.0\%]$ & $+0.15$ & $1.22\;[+1.7\%]$ & $+0.05$ & $14.23\;[-0.1\%]$ & $+0.05$ & $1.48\;[0.0\%]$ \\
RLCP & $+1.94$ & $2.31\;[+62.7\%]$ & $+3.26$ & $1.64\;[+82.2\%]$ & $+1.21$ & $1.64\;[+36.7\%]$ & $+0.90$ & $21.41\;[+50.2\%]$ & $+2.66$ & $139.33\;[+9314\%]$ \\
Mondrian k-means & $+3.17$ & $1.63\;[+14.8\%]$ & $+5.80$ & $0.91\;[+1.1\%]$ & $+1.13$ & $1.25\;[+4.2\%]$ & $+1.57$ & $15.32\;[+7.5\%]$ & $+1.29$ & $1.77\;[+19.6\%]$ \\
\textbf{SO-SCP} & $+0.39$ & $1.46\;[+2.8\%]$ & $+1.44$ & $0.91\;[+1.1\%]$ & $+0.84$ & $1.22\;[+1.7\%]$ & $+0.64$ & $14.67\;[+2.9\%]$ & $+1.13$ & $1.59\;[+7.4\%]$ \\
\bottomrule
\end{tabular}
    }
\end{table}

\begin{figure}[ht]
    \centering
    \includegraphics[width=0.99\textwidth]{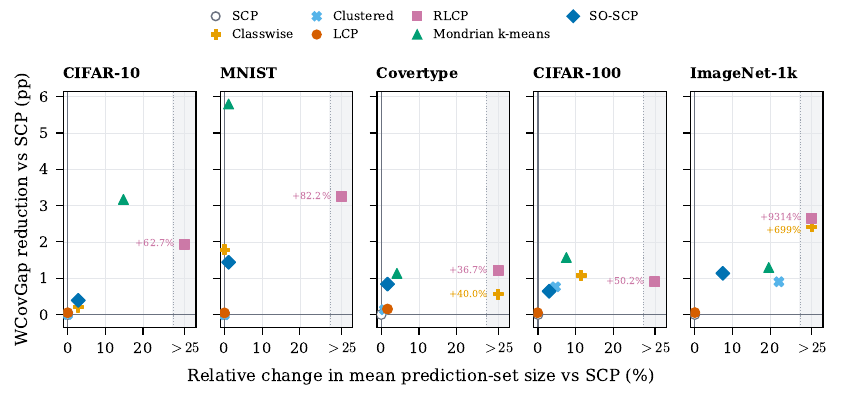}
    \caption{
        WCovGap gain versus relative change in mean set size for Table~\ref{tab:cr_calibrator_tradeoff_classification}.
        \gls{scp} is at the origin and SO-SCP is marked by a diamond.
        The shaded rail contains size increases above $25\%$, with their values printed alongside.
    }
    \label{fig:cr_calibrator_tradeoff_classification}
\end{figure}

SO-SCP gains $0.39$ to $1.44$ WCovGap points across the five datasets in Table~\ref{tab:cr_calibrator_tradeoff_classification}, with set-size increases of $1.1\%$ to $7.4\%$.
This adds at most $0.11$ labels on the four datasets whose mean sets remain below two labels, and $0.42$ labels on CIFAR-100 out of $100$ classes.
\gls{lcp} keeps set sizes nearly unchanged but gains only $0.04$ to $0.15$ points, less than one fifth of SO-SCP's gain on every dataset.
\gls{rlcp} gains more than SO-SCP on every dataset, at substantially larger size increases of $36.7\%$ to $9{,}314\%$.
On ImageNet-1k, its sets average $139.33$ labels, compared with $1.48$ for \gls{scp} and $1.59$ for SO-SCP.
Retaining roughly $14\%$ of the label space on average makes these sets far less selective and less informative about the true label than SO-SCP's.

Classwise and Clustered calibration show trade-offs between WCovGap reduction and prediction set size that vary widely across datasets.
Classwise gains $1.78$ points on MNIST while retaining a displayed mean set size of $0.90$ labels.
However, on Covertype and ImageNet-1k, Classwise enlarges sets by $40.0\%$ and $699\%$, respectively.
Clustered leaves WCovGap essentially unchanged on CIFAR-10 and MNIST, although its gain of $0.76$ points for $4.6\%$ larger sets on CIFAR-100 is competitive.
Nevertheless, on ImageNet-1k, it gains $0.90$ points at a $22.3\%$ size increase, while SO-SCP gains more, $1.13$ points, at $7.4\%$, roughly one third of the enlargement.

Mondrian $k$-means achieves larger WCovGap gains than SO-SCP on all five classification datasets, with larger sets on four.
On MNIST, Mondrian gains $5.80$ points with the same displayed mean set size as SO-SCP, $0.91$ labels.
The audit groups and Mondrian calibration groups are fitted separately but both use $k$-means on the same representation.
Using the same clustering geometry for calibration and auditing could favor Mondrian $k$-means relative to the other calibrators.
The regression comparison in Table~\ref{tab:cr_calibrator_tradeoff_regression} nevertheless shows a different trade-off, with smaller gains than SO-SCP on Bio and California Housing and substantially more unbounded intervals on the two smallest datasets.

Across Tables~\ref{tab:cr_calibrator_tradeoff_regression} and~\ref{tab:cr_calibrator_tradeoff_classification}, a lower WCovGap is useful only if the predictions remain informative.
Some alternatives achieve larger gains with unbounded intervals or heavily enlarged sets, while \gls{lcp} provides little improvement in classification and worsens WCovGap on the two smallest regression datasets.
SO-SCP preserves the scale of pooled outputs on the eight datasets with adequate buffers, and Regime~3 removes its observed unbounded intervals on the two smallest splits.

\subsubsection{Transfer across audit partitions}
\label{app:partition_transfer_diagnostics}

We extend the $K=25$ audit in Section~\ref{app:calibration_tradeoffs} by evaluating the same Regime~2 predictions from \gls{scp} and SO-SCP on a shared \gls{som} and external $k$-means partitions with $K\in\{10,25,50\}$.
The shared \gls{som} uses SO-SCP's training representation and map settings.

\paragraph{Regression.}
Figure~\ref{fig:cr_transfer_regression} compares WCovGap changes across the four audits.
Table~\ref{tab:cr_transfer_regression} gives the means and relative changes, including interval width, which is unaffected by the audit partition.

\begin{figure}[ht]
    \centering
    \includegraphics[width=0.85\textwidth]{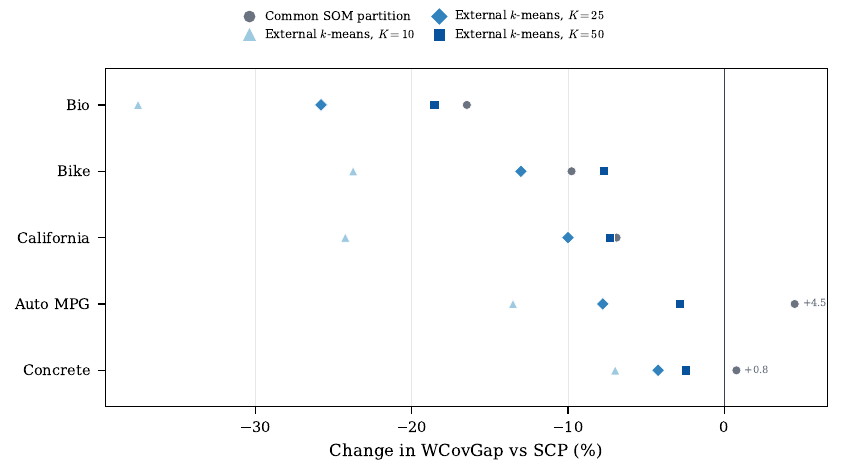}
    \caption{
        Relative WCovGap changes across four audits in Table~\ref{tab:cr_transfer_regression}, with the absolute residual score and Regime~2.
        Markers identify the shared \gls{som} and external $k$-means partitions.
        Negative values indicate reductions.
    }
    \label{fig:cr_transfer_regression}
\end{figure}

\begin{table}[ht]
    \centering
    \caption{
        Regression audit means across ten seeds under Regime~2.
        Each cell shows \gls{scp} $\to$ SO-SCP, with the relative change between the two means in brackets (\%).
        Audit columns report WCovGap in percentage points, and Output reports interval width in target units.
        An infinite mean width indicates that some intervals are unbounded.
    }
    \label{tab:cr_transfer_regression}
    \resizebox{\textwidth}{!}{%
        \begin{tabular}{lccccc}
\toprule
 & $K=10$ & $K=25$ & $K=50$ & SOM & Output \\
\midrule
Bike & \shortstack{$4.17\to 3.18$\\$[-23.7\%]$} & \shortstack{$5.46\to 4.75$\\$[-13.0\%]$} & \shortstack{$6.74\to 6.22$\\$[-7.7\%]$} & \shortstack{$6.04\to 5.45$\\$[-9.8\%]$} & \shortstack{$115.36\to 117.44$\\$[+1.8\%]$} \\
Bio & \shortstack{$1.92\to 1.20$\\$[-37.5\%]$} & \shortstack{$3.14\to 2.33$\\$[-25.8\%]$} & \shortstack{$3.72\to 3.03$\\$[-18.5\%]$} & \shortstack{$4.19\to 3.50$\\$[-16.5\%]$} & \shortstack{$12.74\to 12.79$\\$[+0.4\%]$} \\
California Housing & \shortstack{$2.31\to 1.75$\\$[-24.2\%]$} & \shortstack{$3.50\to 3.15$\\$[-10.0\%]$} & \shortstack{$4.49\to 4.16$\\$[-7.3\%]$} & \shortstack{$4.50\to 4.19$\\$[-6.9\%]$} & \shortstack{$1.39\to 1.36$\\$[-2.2\%]$} \\
Concrete & \shortstack{$7.74\to 7.20$\\$[-7.0\%]$} & \shortstack{$10.40\to 9.96$\\$[-4.2\%]$} & \shortstack{$12.72\to 12.41$\\$[-2.4\%]$} & \shortstack{$10.28\to 10.36$\\$[+0.8\%]$} & \shortstack{$16.57\to \infty$\\$[+\infty]$} \\
Auto MPG & \shortstack{$11.10\to 9.60$\\$[-13.5\%]$} & \shortstack{$12.87\to 11.87$\\$[-7.8\%]$} & \shortstack{$14.20\to 13.80$\\$[-2.8\%]$} & \shortstack{$11.10\to 11.60$\\$[+4.5\%]$} & \shortstack{$10.25\to \infty$\\$[+\infty]$} \\
\bottomrule
\end{tabular}
    }
\end{table}

All fifteen external-audit comparisons in Table~\ref{tab:cr_transfer_regression} favor SO-SCP.
The relative reductions become smaller as the audit is refined, ranging from $7.0\%$ to $37.5\%$ at $K=10$ and from $2.4\%$ to $18.5\%$ at $K=50$.
On the shared \gls{som}, Bike Sharing, Bio, and California Housing also improve, by $9.8\%$, $16.5\%$, and $6.9\%$, respectively.

Concrete and Auto MPG are the exceptions on the shared map, with WCovGap increases of $0.8\%$ and $4.5\%$.
Their external-audit reductions also accompany the unbounded Regime~2 intervals documented in Table~\ref{tab:cr_calibrator_tradeoff_regression}.
The switch to Regime~3 in Table~\ref{tab:cr_main_regression} removes the observed unbounded intervals on these two datasets, at the cost of smaller WCovGap reductions at $K=25$.

\paragraph{Classification.}
Figure~\ref{fig:cr_transfer_classification} and Table~\ref{tab:cr_transfer_classification} repeat the four audits with \gls{lac} and report the accompanying change in set size.

\begin{figure}[ht]
    \centering
    \includegraphics[width=0.82\textwidth]{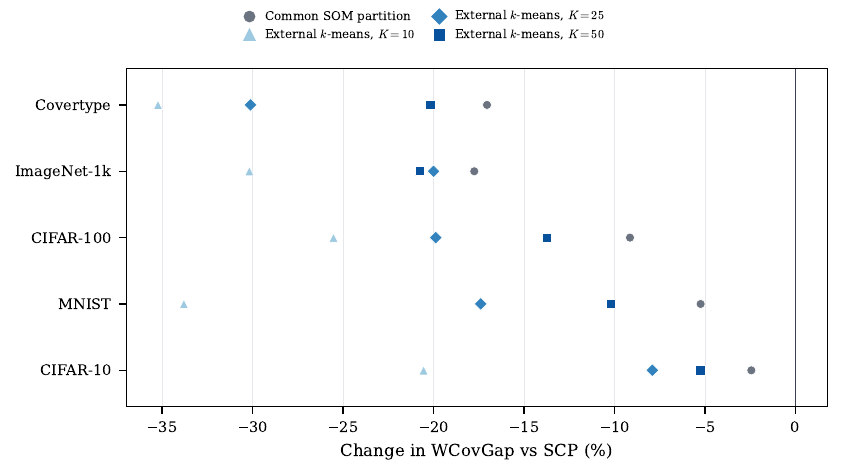}
    \caption{
        Relative WCovGap changes in Table~\ref{tab:cr_transfer_classification}, with \gls{lac} and Regime~2.
        Markers identify the shared \gls{som} and external $k$-means partitions.
        Negative values indicate reductions.
    }
    \label{fig:cr_transfer_classification}
\end{figure}

\begin{table}[ht]
    \centering
    \caption{
        Classification audit means across ten seeds under Regime~2.
        Each cell shows \gls{scp} $\to$ SO-SCP, with the relative change between the two means in brackets (\%).
        Audit columns report WCovGap in percentage points, and Output reports set size.
    }
    \label{tab:cr_transfer_classification}
    \resizebox{\textwidth}{!}{%
        \begin{tabular}{lccccc}
\toprule
 & $K=10$ & $K=25$ & $K=50$ & SOM & Output \\
\midrule
CIFAR-10 & \shortstack{$3.65\to 2.90$\\$[-20.5\%]$} & \shortstack{$4.93\to 4.54$\\$[-7.9\%]$} & \shortstack{$5.51\to 5.22$\\$[-5.3\%]$} & \shortstack{$6.95\to 6.78$\\$[-2.4\%]$} & \shortstack{$1.42\to 1.46$\\$[+2.8\%]$} \\
MNIST & \shortstack{$5.09\to 3.37$\\$[-33.8\%]$} & \shortstack{$8.28\to 6.84$\\$[-17.4\%]$} & \shortstack{$10.11\to 9.08$\\$[-10.2\%]$} & \shortstack{$13.15\to 12.46$\\$[-5.2\%]$} & \shortstack{$0.90\to 0.91$\\$[+1.1\%]$} \\
Covertype & \shortstack{$2.30\to 1.49$\\$[-35.2\%]$} & \shortstack{$2.79\to 1.95$\\$[-30.1\%]$} & \shortstack{$3.42\to 2.73$\\$[-20.2\%]$} & \shortstack{$3.52\to 2.92$\\$[-17.0\%]$} & \shortstack{$1.20\to 1.22$\\$[+1.7\%]$} \\
CIFAR-100 & \shortstack{$2.86\to 2.13$\\$[-25.5\%]$} & \shortstack{$3.22\to 2.58$\\$[-19.9\%]$} & \shortstack{$3.85\to 3.32$\\$[-13.8\%]$} & \shortstack{$4.70\to 4.27$\\$[-9.1\%]$} & \shortstack{$14.25\to 14.67$\\$[+2.9\%]$} \\
ImageNet-1k & \shortstack{$4.64\to 3.24$\\$[-30.2\%]$} & \shortstack{$5.70\to 4.56$\\$[-20.0\%]$} & \shortstack{$6.36\to 5.04$\\$[-20.8\%]$} & \shortstack{$6.31\to 5.19$\\$[-17.7\%]$} & \shortstack{$1.48\to 1.59$\\$[+7.4\%]$} \\
\bottomrule
\end{tabular}
    }
\end{table}

Every classification dataset improves on all four partitions in Table~\ref{tab:cr_transfer_classification}.
The shared \gls{som} reductions range from $2.4\%$ to $17.7\%$ and average $10.3\%$.
All fifteen external comparisons favor SO-SCP, including reductions of at least $5.3\%$ at $K=50$.
ImageNet-1k retains a reduction close to $20\%$ at both $K=25$ and $K=50$.
The set-size changes match Table~\ref{tab:cr_main_classification}, since changing the audit does not alter predictions.

Together, Tables~\ref{tab:cr_transfer_regression} and~\ref{tab:cr_transfer_classification} record improvements in all $30$ external comparisons.
For the eight datasets with finite Regime~2 outputs, the gains also persist on the shared \gls{som}, supporting transfer across both audit geometries.

\subsubsection{Diagnostics without feature-space partitions}
\label{app:partition_free_diagnostics}

Section~\ref{app:partition_transfer_diagnostics} evaluates coverage on fixed groups in the input representation.
Now, we assess the same Regime~2 predictions using held-out \gls{wsc} target error, \gls{ssc} gap, and $L_1$ \gls{ert}, which need no feature-space partition.

Held-out \gls{wsc} averages coverage estimates from ten random half splits of the test observations per seed.
Each search scans $500$ directions for the lowest-coverage window containing $10\%$ of the search half and evaluates its slab on the other half.
Target error is the absolute deviation of this average from $0.90$.
The \gls{ssc} gap equally averages absolute coverage deviations across at most ten width or set-size strata.
The $L_1$ \gls{ert} diagnostic estimates coverage with a histogram gradient boosting classifier using five folds and computes excess risk under the $L_1$ miscoverage loss.
Smaller \gls{wsc} target error and \gls{ssc} gap are better, while $L_1$ \gls{ert} is better closer to zero.

\paragraph{Regression.}
Table~\ref{tab:cr_conditional_regression} reports diagnostic means and relative changes.
Figure~\ref{fig:cr_conditional_regression} shows the paired differences for individual seeds.

\begin{table}[ht]
    \centering
    \caption{
        Regression diagnostics without feature-space partitions, with the absolute residual score and Regime~2, averaged across ten seeds.
        $\Delta$ gives the difference between the SO-SCP and \gls{scp} means, with the relative change in parentheses.
        Lower values are bold.
        Auto MPG has no supported held-out \gls{wsc} estimate.
    }
    \label{tab:cr_conditional_regression}
    \resizebox{\textwidth}{!}{%
        \begin{tabular}{l*{9}{c}}
\toprule
 & \multicolumn{3}{c}{WSC error ($\downarrow$)} & \multicolumn{3}{c}{SSC gap ($\downarrow$)} & \multicolumn{3}{c}{$L_1$-ERT ($\rightarrow 0$)} \\
 & SCP & SO-SCP & $\Delta$ & SCP & SO-SCP & $\Delta$ & SCP & SO-SCP & $\Delta$ \\
\midrule
Bike & $0.1064$ & $\mathbf{0.0716}$ & $-0.0348\;(-32.7\%)$ & $0.0851$ & $\mathbf{0.0434}$ & $-0.0417\;(-49.0\%)$ & $0.0621$ & $\mathbf{0.0539}$ & $-0.0082\;(-13.2\%)$ \\
Bio & $0.0360$ & $\mathbf{0.0193}$ & $-0.0167\;(-46.4\%)$ & $0.0414$ & $\mathbf{0.0150}$ & $-0.0264\;(-63.8\%)$ & $0.0407$ & $\mathbf{0.0345}$ & $-0.0062\;(-15.2\%)$ \\
California Housing & $0.0925$ & $\mathbf{0.0849}$ & $-0.0076\;(-8.2\%)$ & $0.0863$ & $\mathbf{0.0215}$ & $-0.0648\;(-75.1\%)$ & $\mathbf{0.0549}$ & $0.0558$ & $+0.0009\;(+1.6\%)$ \\
Concrete & $0.0969$ & $\mathbf{0.0378}$ & $-0.0591\;(-61.0\%)$ & $0.1048$ & $\mathbf{0.0833}$ & $-0.0215\;(-20.5\%)$ & $0.0305$ & $\mathbf{0.0272}$ & $-0.0033\;(-10.8\%)$ \\
Auto MPG & \textit{n/a} & \textit{n/a} & \textit{n/a} & $0.1320$ & $\mathbf{0.1158}$ & $-0.0162\;(-12.3\%)$ & $0.0457$ & $\mathbf{0.0280}$ & $-0.0177\;(-38.7\%)$ \\
\bottomrule
\end{tabular}
    }
\end{table}

\begin{figure}[ht]
    \centering
    \includegraphics[width=0.85\textwidth]{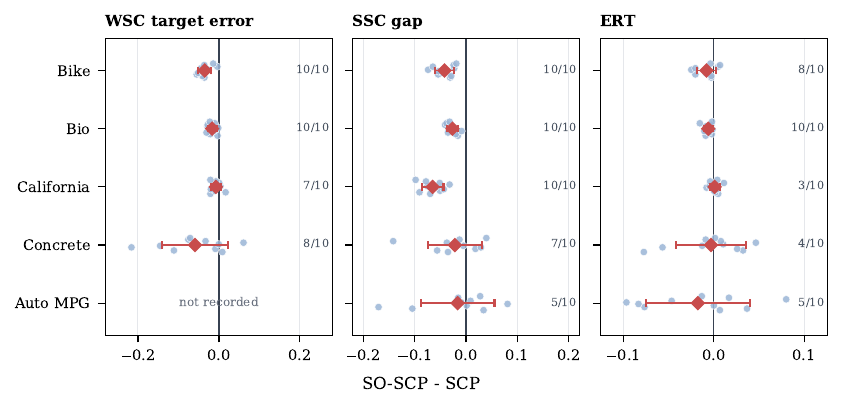}
    \caption{
        Paired diagnostic differences for Table~\ref{tab:cr_conditional_regression}, computed as SO-SCP minus \gls{scp}.
        Dots are seeds, diamonds are means, and bars show one standard deviation on either side.
        Counts give seeds with a negative difference.
    }
    \label{fig:cr_conditional_regression}
\end{figure}

Table~\ref{tab:cr_conditional_regression} shows lower \gls{wsc} target error and \gls{ssc} gap on every recorded regression dataset.
Across the available datasets, the mean relative reductions under SO-SCP are $37.1\%$ for held-out \gls{wsc} target error, $44.1\%$ for the \gls{ssc} gap, and $15.3\%$ for $L_1$ \gls{ert}.
On Bio, \gls{wsc} target error falls by $46.4\%$ and the \gls{ssc} gap by $63.8\%$.
California Housing has the largest \gls{ssc} reduction, $75.1\%$, but is the only regression dataset whose $L_1$ \gls{ert} increases, from $0.0549$ to $0.0558$.

Auto MPG has no held-out \gls{wsc} estimate because, in every seed and for both methods, at least one selected slab contains no evaluation observations.
Concrete and Auto MPG also retain the unbounded Regime~2 intervals from Table~\ref{tab:cr_calibrator_tradeoff_regression}, so their diagnostic reductions do not establish improved interval efficiency.

\paragraph{Classification.}
Table~\ref{tab:cr_conditional_classification} and Figure~\ref{fig:cr_conditional_classification} report the same diagnostics for \gls{lac}.

\begin{table}[ht]
    \centering
    \caption{
        Classification diagnostics without feature-space partitions, with \gls{lac} and Regime~2, averaged across ten seeds.
        $\Delta$ gives the difference between the SO-SCP and \gls{scp} means, with the relative change in parentheses.
        Lower values are bold.
    }
    \label{tab:cr_conditional_classification}
    \resizebox{\textwidth}{!}{%
        \begin{tabular}{l*{9}{c}}
\toprule
 & \multicolumn{3}{c}{WSC error ($\downarrow$)} & \multicolumn{3}{c}{SSC gap ($\downarrow$)} & \multicolumn{3}{c}{$L_1$-ERT ($\rightarrow 0$)} \\
 & SCP & SO-SCP & $\Delta$ & SCP & SO-SCP & $\Delta$ & SCP & SO-SCP & $\Delta$ \\
\midrule
CIFAR-10 & $0.0707$ & $\mathbf{0.0658}$ & $-0.0049\;(-6.9\%)$ & $\mathbf{0.0492}$ & $0.0936$ & $+0.0444\;(+90.2\%)$ & $0.0663$ & $\mathbf{0.0639}$ & $-0.0024\;(-3.6\%)$ \\
MNIST & $0.3337$ & $\mathbf{0.2980}$ & $-0.0357\;(-10.7\%)$ & $0.4998$ & $\mathbf{0.4998}$ & $+0.0000\;(+0.0\%)$ & $0.1698$ & $\mathbf{0.1613}$ & $-0.0085\;(-5.0\%)$ \\
Covertype & $0.0636$ & $\mathbf{0.0327}$ & $-0.0309\;(-48.6\%)$ & $\mathbf{0.2140}$ & $0.2234$ & $+0.0094\;(+4.4\%)$ & $0.0465$ & $\mathbf{0.0442}$ & $-0.0023\;(-4.9\%)$ \\
CIFAR-100 & $0.0322$ & $\mathbf{0.0203}$ & $-0.0119\;(-37.0\%)$ & $\mathbf{0.0275}$ & $0.0305$ & $+0.0030\;(+10.9\%)$ & $0.0268$ & $\mathbf{0.0245}$ & $-0.0023\;(-8.6\%)$ \\
ImageNet-1k & $0.1593$ & $\mathbf{0.1361}$ & $-0.0232\;(-14.6\%)$ & $0.1811$ & $\mathbf{0.1746}$ & $-0.0065\;(-3.6\%)$ & $0.0789$ & $\mathbf{0.0692}$ & $-0.0097\;(-12.3\%)$ \\
\bottomrule
\end{tabular}
    }
\end{table}

\begin{figure}[ht]
    \centering
    \includegraphics[width=0.85\textwidth]{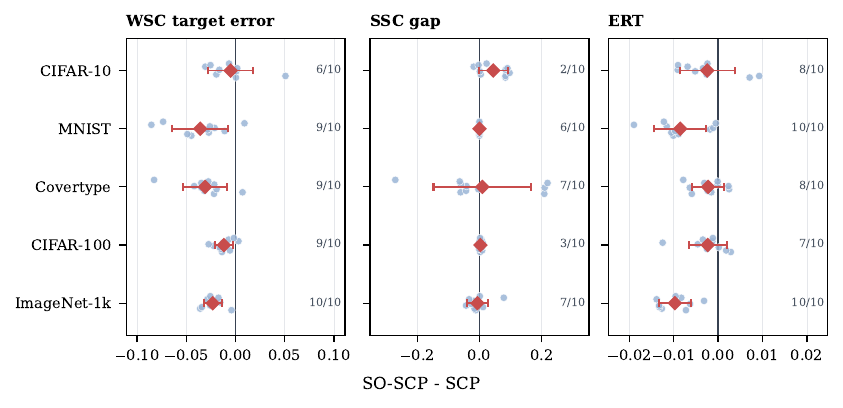}
    \caption{
        Paired diagnostic differences for Table~\ref{tab:cr_conditional_classification}, computed as SO-SCP minus \gls{scp}.
        Dots are seeds, diamonds are means, and bars show one standard deviation on either side.
        Counts give seeds with a negative difference.
    }
    \label{fig:cr_conditional_classification}
\end{figure}

Held-out \gls{wsc} target error and $L_1$ \gls{ert} decrease on all five datasets in Table~\ref{tab:cr_conditional_classification}.
The \gls{wsc} reductions reach $48.6\%$ on Covertype and $37.0\%$ on CIFAR-100.
On MNIST, held-out \gls{wsc} target error decreases from $0.3337$ to $0.2980$ but remains high.

The classification \gls{ssc} result is mixed.
Table~\ref{tab:cr_conditional_classification} shows an unchanged \gls{ssc} gap on MNIST, a $3.6\%$ decrease on ImageNet-1k, and increases on the other three datasets.
CIFAR-10 has the largest increase, from $0.0492$ to $0.0936$, or $90.2\%$.
SO-SCP returns empty sets for $2$ to $81$ of the $10{,}000$ test images in five seeds, while \gls{scp} returns none.
These empty predictions form a size group with zero coverage, to which \gls{ssc} assigns the same weight as every other observed size group.
On CIFAR-10, the mean \gls{ssc} increase is $0.0875$ in the five seeds where SO-SCP produces empty sets, compared with $0.0013$ in the other five.
As a result, the rise reflects both occasional empty predictions and the equal weighting of their size group.

On MNIST, both methods return either zero or one label.
The empty-set group contributes a target error of $0.90$, while the almost perfectly covered singleton group contributes an error close to $0.10$.
Because \gls{ssc} weights the two groups equally, regardless of how many sets are empty, the gap is close to $0.50$ for both methods.
Weighting \gls{ssc} groups by their number of observations could reduce the influence of rare set sizes and lower the large values observed on MNIST and Covertype.

Across Tables~\ref{tab:cr_conditional_regression} and~\ref{tab:cr_conditional_classification}, held-out \gls{wsc} target error improves on all nine recorded datasets and $L_1$ \gls{ert} on nine of ten.
The relative changes average $-29.6\%$ and $-11.1\%$, respectively.
These diagnostics support the coverage improvements under alternative audits in Section~\ref{app:partition_transfer_diagnostics}, while the classification \gls{ssc} gap is mixed.

\subsubsection{Sensitivity and small calibration buffers}
\label{app:sensitivity_small_buffers}

\paragraph{Map resolution and retrieval radius.}
The Regime~2 failures in Table~\ref{tab:cr_calibrator_tradeoff_regression} motivate examining how map resolution and retrieval radius affect local calibration support.
Figure~\ref{fig:cr_sensitivity} varies these choices on California Housing while keeping the absolute residual score and external audit partitions fixed.

\begin{figure}[ht]
    \centering
    \includegraphics[width=\textwidth]{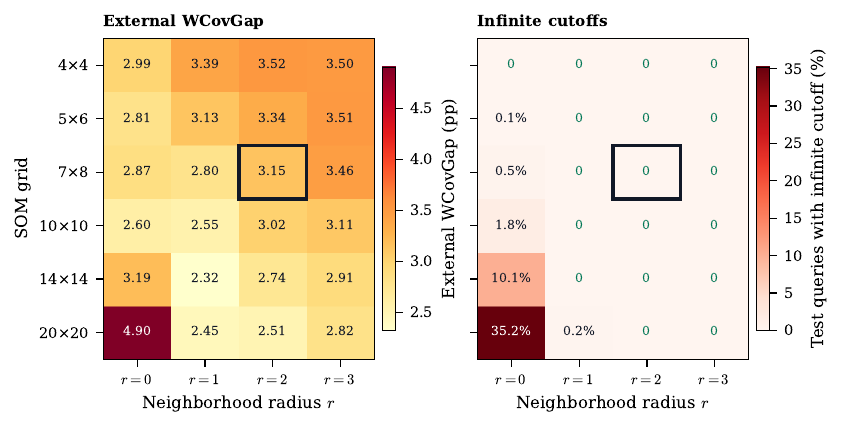}
    \caption{
        California Housing sensitivity across six \gls{som} grids and four retrieval radii, averaged over ten seeds.
        Panels show external WCovGap in percentage points at $K=25$ and the percentage of queries with infinite cutoffs.
        The outline marks the headline $7\times8$ grid with radius $2$.
    }
    \label{fig:cr_sensitivity}
\end{figure}

With radius $0$, each query uses one cell, and the infinite-cutoff rate in Figure~\ref{fig:cr_sensitivity} rises from $0\%$ on the $4\times4$ grid to $35.2\%$ on the $20\times20$ grid.
Radius $1$ leaves no infinite cutoffs through the $14\times14$ grid and only $0.2\%$ on the $20\times20$ grid.
Radii $2$ and $3$ produce no infinite cutoffs on any tested grid, showing how pooling neighboring cells can restore calibration support.

Among the $18$ configurations with no infinite cutoff, WCovGap ranges from $2.32$ to $3.52$ points.
Fifteen improve on the pooled \gls{scp} value of $3.50$ points on the same audit partitions.
The headline configuration yields $3.15$ points, while the lowest gap occurs at $14\times14$ with radius $1$.
Consequently, the gain is present across multiple configurations and does not depend on selecting the best result in the sweep.
This sensitivity experiment covers one dataset, so it does not establish a single preferred grid or radius across tasks.

\paragraph{Calibration set size.}
The small-buffer behavior on Concrete and Auto MPG in Table~\ref{tab:cr_calibrator_tradeoff_regression} becomes more acute as calibration size decreases.
Figure~\ref{fig:cr_regime3} compares the fixed neighborhoods of Regime~2 with the training-based enlargement of Regime~3 at four calibration sizes on Concrete and three on Auto MPG.
Within each seed and size, all methods use the same calibration observations.

\begin{figure}[ht]
    \centering
    \includegraphics[width=0.99\textwidth]{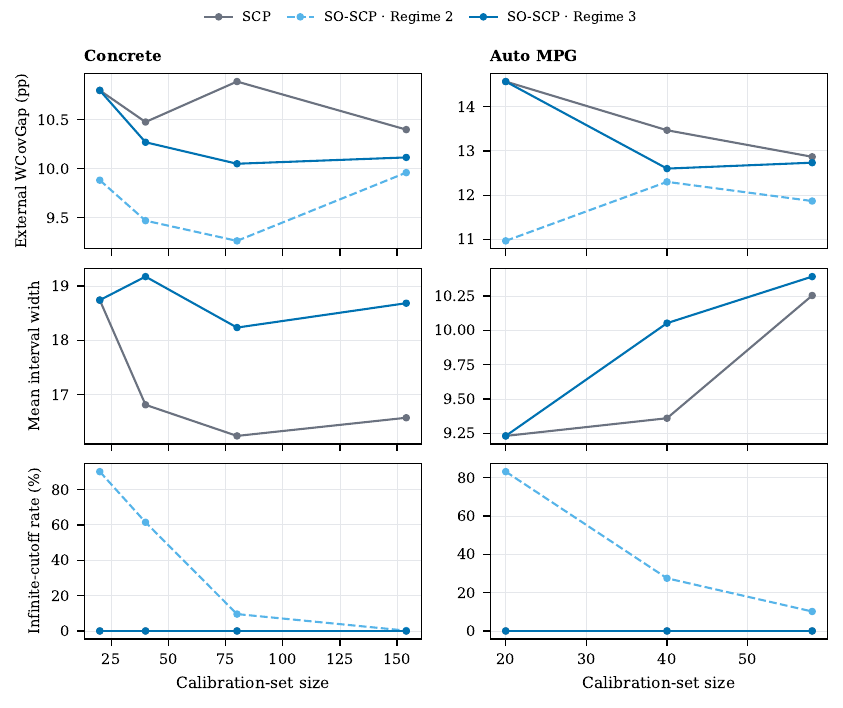}
    \caption{
        Calibration-size effects on Concrete and Auto MPG with the absolute residual score, averaged over ten seeds.
        Rows show external WCovGap at $K=25$, mean interval width, and the infinite-cutoff rate.
        Regime~3 uses routing from training data with $b=19$.
        Regime~2 has no finite mean-width curve.
    }
    \label{fig:cr_regime3}
\end{figure}

At $20$ calibration observations, Regime~2 returns infinite cutoffs for $90.1\%$ of Concrete queries and $83.2\%$ of Auto MPG queries.
Its mean interval width is infinite at every tested size, including the full calibration splits.
As a result, the lower Regime~2 WCovGap at $20$ observations accompanies mostly unbounded intervals.

Regime~3 produces finite cutoffs at every tested size on both datasets in Figure~\ref{fig:cr_regime3}.
The coverage means across ten seeds range from $0.897$ to $0.926$.
For each seed and planned calibration size, the rule selects the smallest global enlargement budget whose training-based projected buffer size reaches $b=19$ in every reference cell.
It fixes this budget before calibration.
At $20$ observations, the selected buffers recover pooling on both datasets.
As calibration size increases to the full split, the mean budget decreases from $26.0$ to $2.2$ additional cells on Concrete and from $16.0$ to $7.4$ on Auto MPG, allowing more local calibration.

At every tested size above $20$, Regime~3 lowers mean WCovGap relative to pooled \gls{scp}.
For example, the gap falls from $10.89$ to $10.05$ points with $80$ observations on Concrete and from $13.47$ to $12.60$ points with $40$ observations on Auto MPG.
These reductions accompany wider intervals, consistent with the small-sample trade-off in Table~\ref{tab:cr_main_regression}.
Thus, Regime~3 addresses the observed lack of local support, which also produces unbounded intervals for several competing calibrators in Table~\ref{tab:cr_calibrator_tradeoff_regression}.
The absence of infinite cutoffs remains an empirical finding.
The projected target $b=19$ does not guarantee that a realized buffer contains $19$ observations or that every cutoff is finite.

\clearpage
\subsection{Spatial diagnostics on the \glsentryshort{som}}
\label{app:spatial_diagnostics}

We examine how the coverage and output-size trade-offs in Section~\ref{app:detailed_results} arise across the fitted \gls{som}.
All maps use seed~$42$ and the task's base score, the absolute residual for regression and \gls{lac} for classification.
Unless explicitly averaged across ten seeds, all numbers describe the displayed fit.
SO-SCP uses Regime~2 on all ten datasets, including Concrete and Auto MPG.
Each calibrator retains its own prediction rule, and its outputs are evaluated on the same \gls{som} used by SO-SCP.
Grid coordinates start at $(0,0)$ in the bottom left.
This shared grid describes the spatial behavior of the methods, while Section~\ref{app:partition_transfer_diagnostics} evaluates the same predictions on separate external partitions.

First, we examine map geometry, occupancy, and calibration buffer sizes (Section~\ref{app:map_geometry_occupancy_buffers}).
Then, we describe predictor difficulty and score tails (Section~\ref{app:predictor_difficulty_score_tails}) and compare coverage across calibrators (Section~\ref{app:coverage_calibration_methods}).
We examine the corresponding thresholds and output sizes (Section~\ref{app:thresholds_output_size}), while empirical \gls{ks} diagnostics compare score distributions in each cell and its retrieved buffer (Section~\ref{app:ks_diagnostics}).
Finally, we isolate SO-SCP's changes in coverage and output size (Section~\ref{app:cellwise_coverage_output_changes}) and illustrate these local adjustments on Covertype (Section~\ref{app:worked_example_covertype}).

\subsubsection{Map geometry, occupancy, and buffer size}
\label{app:map_geometry_occupancy_buffers}

The U-matrix summarizes distances between neighboring prototypes in the fixed input representation.
Hit maps count the calibration and test observations assigned to each cell.
Under i.i.d. sampling, calibration and test observations should occupy the cells in similar proportions.
Small calibration buffers can make cutoffs less stable, while few test observations can make cellwise coverage estimates noisy.

\paragraph{Regression.}
Figure~\ref{fig:cr_maps_regression_geometry} shows the geometry and occupancy of the five regression maps.

\begin{figure}[ht]
    \centering
    \includegraphics[width=0.8\textwidth]{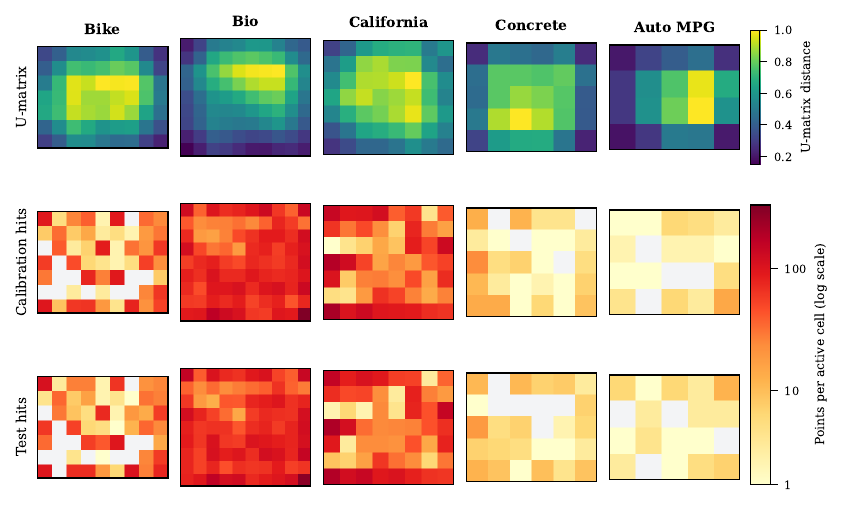}
    \caption{
        Regression \gls{som} geometry and occupancy for seed~$42$, with one dataset per column.
        Rows show summed distances to adjacent prototypes, normalized by the map maximum, then calibration and test counts on a shared logarithmic color scale with raw counts on the colorbar.
        Gray hit-map cells contain no observations from the displayed split.
    }
    \label{fig:cr_maps_regression_geometry}
\end{figure}

The counts in Figure~\ref{fig:cr_maps_regression_geometry} vary considerably across cells, especially on Concrete and Auto MPG.
Under Definition~\ref{def:conformal-q}, fewer than nine calibration scores, including an empty buffer, force an infinite cutoff at $\alpha=0.1$.
Regime~2 increases the available sample by pooling scores over a fixed neighborhood.
Its buffer size is determined by the fitted map, retrieval radius, and calibration inputs, without using calibration labels.
Figure~\ref{fig:cr_buffers_regression} reports one buffer size per cell containing test observations.

\begin{figure}[ht]
    \centering
    \includegraphics[width=0.95\textwidth]{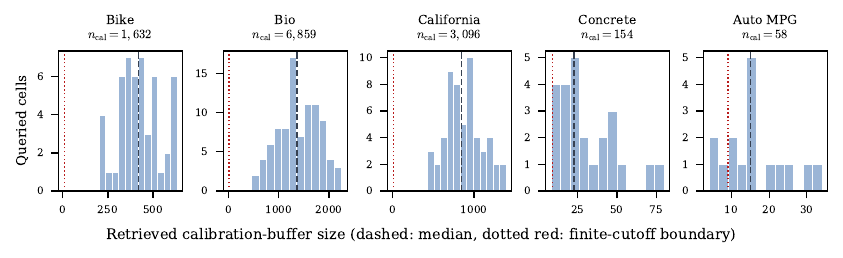}
    \caption{
        Regime~2 buffer sizes across regression cells containing test observations.
        Dashed lines mark medians.
        Dotted red lines mark nine scores, the minimum for a finite base-score cutoff at $\alpha=0.1$.
    }
    \label{fig:cr_buffers_regression}
\end{figure}

Every retrieved buffer on Bike Sharing, Bio, and California Housing contains at least $200$ scores.
Concrete and Auto MPG have median buffer sizes of $22.5$ and $15$.
The retrieved buffers of one of the $24$ occupied Concrete cells and three of the $16$ occupied Auto MPG cells contain fewer than nine scores.
Consequently, Regime~2 uses infinite cutoffs for queries in these four cells.

\paragraph{Classification.}
Figure~\ref{fig:cr_maps_classification_geometry} gives the corresponding geometry and occupancy maps for classification.

\begin{figure}[ht]
    \centering
    \includegraphics[width=0.8\textwidth]{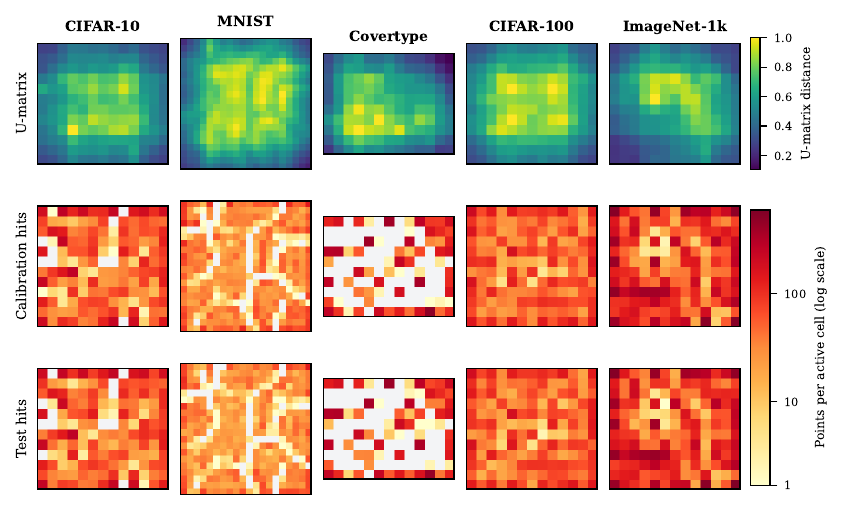}
    \caption{
        Classification \gls{som} geometry and occupancy for seed~$42$.
        Rows, scales, and empty-cell conventions match Figure~\ref{fig:cr_maps_regression_geometry}.
    }
    \label{fig:cr_maps_classification_geometry}
\end{figure}

On MNIST, sparse and empty cells often follow the bright U-matrix ridges, where neighboring prototypes are more separated.
On Covertype, cells with hundreds of observations can lie next to empty cells.
These patterns make calibration within individual cells uneven even when the full calibration split is large.
Figure~\ref{fig:cr_buffers_classification} shows the support obtained by neighborhood pooling.

\begin{figure}[ht]
    \centering
    \includegraphics[width=0.95\textwidth]{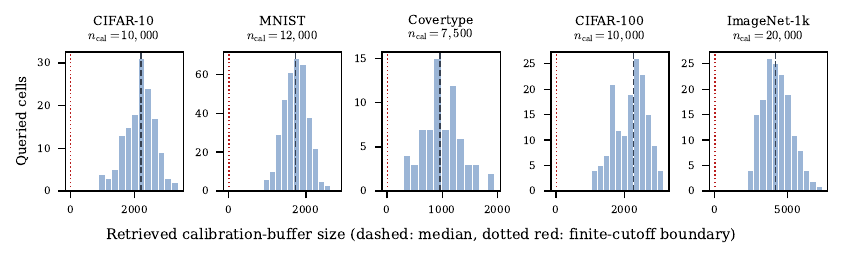}
    \caption{
        Regime~2 buffer sizes across classification cells containing test observations.
        Reference lines match Figure~\ref{fig:cr_buffers_regression}.
    }
    \label{fig:cr_buffers_classification}
\end{figure}

Every classification buffer contains at least $292$ scores, comfortably above the boundary of nine.
Therefore, pooling supplies finite cutoffs even where individual cells have few calibration observations.
Coverage within a cell also depends on how closely the pooled score distribution represents that cell.

\subsubsection{Predictor difficulty and score tails}
\label{app:predictor_difficulty_score_tails}

For each cell $k$ containing test observations, let $\widehat q^{\mathrm{test}}_{0.9,k}$ denote the empirical $90$th percentile of the base scores evaluated at the observed test responses or true labels.
This quantity describes the local score tail and is never used for calibration.

\paragraph{Regression.}
Figure~\ref{fig:maps_regression_predictor_score} compares mean prediction error with the upper tail of the absolute residual score.

\begin{figure}[ht]
    \centering
    \includegraphics[width=0.89\textwidth]{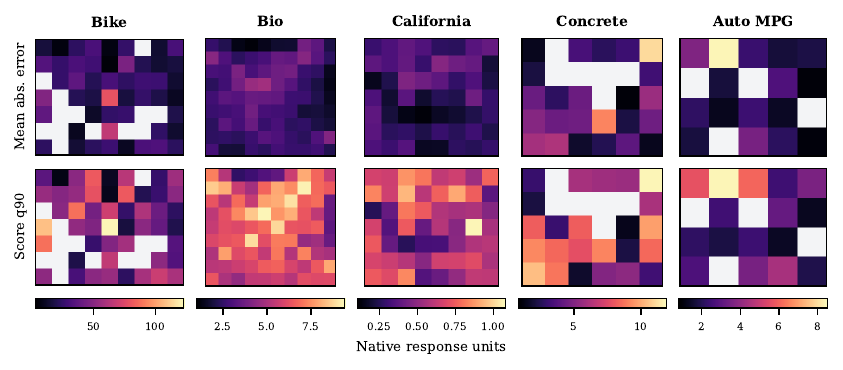}
    \caption{
        Regression predictor diagnostics for seed~$42$.
        Rows show mean absolute error and the empirical score percentile $\widehat q^{\mathrm{test}}_{0.9,k}$.
        Each dataset shares one scale in response units across both rows.
        Gray cells contain no test observations.
    }
    \label{fig:maps_regression_predictor_score}
\end{figure}

The mean describes typical error, while the $90$th percentile identifies the larger residuals that matter for the coverage target.
The upper tail varies across the grid, forming a broad band on Bio and reaching high values along the upper edge of California Housing.
A single \gls{scp} cutoff cannot follow this variation.
It can be conservative where residuals are small and insufficient where they are large, which motivates adjusting the cutoff locally.

\paragraph{Classification.}
Figure~\ref{fig:maps_classification_predictor_score} compares top-1 error with the upper tail of the \gls{lac} score.

\begin{figure}[ht]
    \centering
    \includegraphics[width=0.89\textwidth]{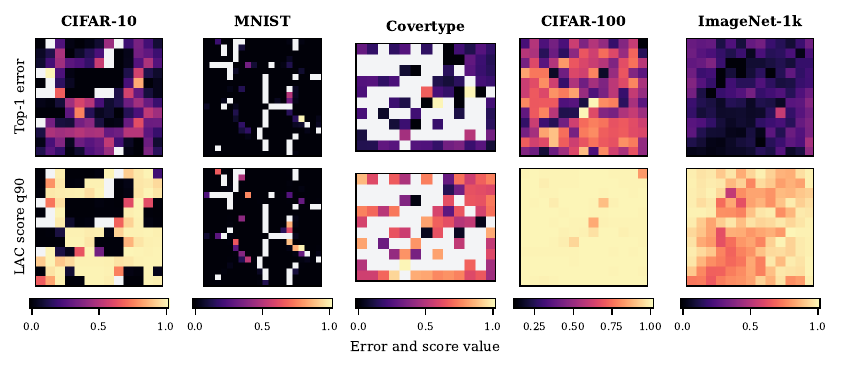}
    \caption{
        Classification predictor diagnostics for seed~$42$.
        Rows show top-1 error and the empirical \gls{lac} percentile $\widehat q^{\mathrm{test}}_{0.9,k}$.
        Each dataset shares one scale across both rows.
        Gray cells contain no test observations.
    }
    \label{fig:maps_classification_predictor_score}
\end{figure}

Top-1 error records whether the most probable label is correct, whereas the \gls{lac} score $1-\widehat p_y(x)$ depends on the probability assigned to the true label.
Even a correct prediction can have a large score if its probability is spread across several labels.
On MNIST, large score percentiles are concentrated in sparsely populated cells.
CIFAR-100 has percentiles close to one across almost the entire grid, consistent with low probabilities assigned to the true labels and a top-1 accuracy of $0.49$.
ImageNet-1k shows broader spatial variation.
These patterns indicate that the cutoff needed for the target coverage can vary both within and across datasets.

\subsubsection{Coverage across calibration methods}
\label{app:coverage_calibration_methods}

We compare empirical coverage on the shared \gls{som} grid, displaying every cell with at least one test observation.
Coverage in a cell with one test observation is either zero or one, so the maps should be read alongside the test counts in Figures~\ref{fig:cr_maps_regression_geometry} and~\ref{fig:cr_maps_classification_geometry}.
The ten-seed comparisons in Section~\ref{app:calibration_tradeoffs} provide the aggregate results.

\paragraph{Regression.}
Figure~\ref{fig:maps_regression_calibrators} compares the five calibrators using the same absolute residual score.

\begin{figure}[ht]
    \centering
    \includegraphics[width=0.99\textwidth]{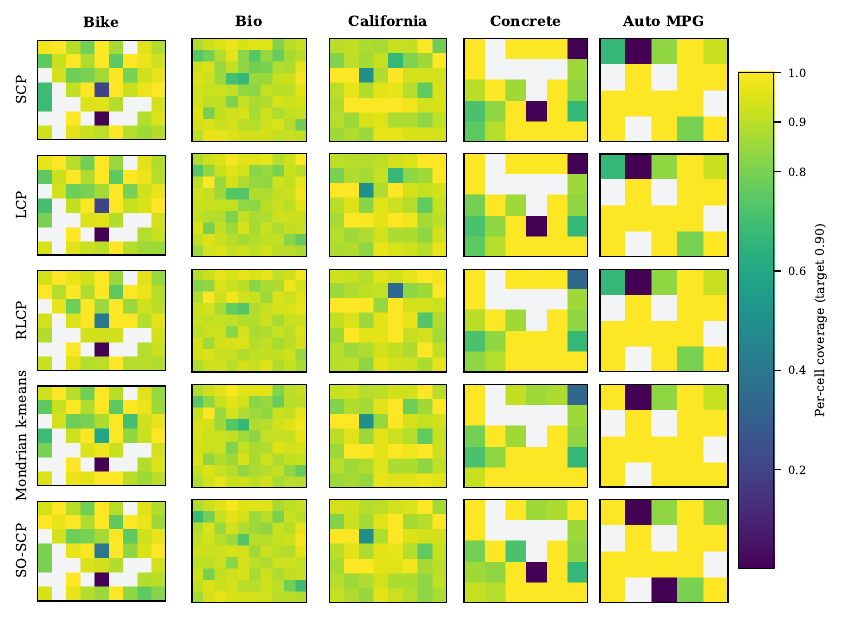}
    \caption{
        Regression coverage by cell with absolute residual score for seed~$42$.
        Rows show \gls{scp}, \gls{lcp}, \gls{rlcp}, Mondrian $k$-means, and SO-SCP, with datasets in columns.
        All panels share one scale with the $0.90$ target marked.
        Gray cells contain no test observations.
    }
    \label{fig:maps_regression_calibrators}
\end{figure}

Bio illustrates how local calibration changes coverage across the map.
Among the ten cells with \gls{scp} coverage below $0.80$, SO-SCP raises mean coverage from $0.757$ to $0.789$.
The number of cells below $0.80$ falls to six.
These gains accompany nearly unchanged marginal coverage, from $0.905$ to $0.906$, and mean interval width, from $12.95$ to $12.98$.
Therefore, SO-SCP improves coverage in these initially undercovered cells with little change in average output size.

\paragraph{Classification.}
Figure~\ref{fig:maps_classification_calibrators} extends the comparison to all seven classification calibrators.

\begin{figure}[ht]
    \centering
    \includegraphics[width=0.99\textwidth]{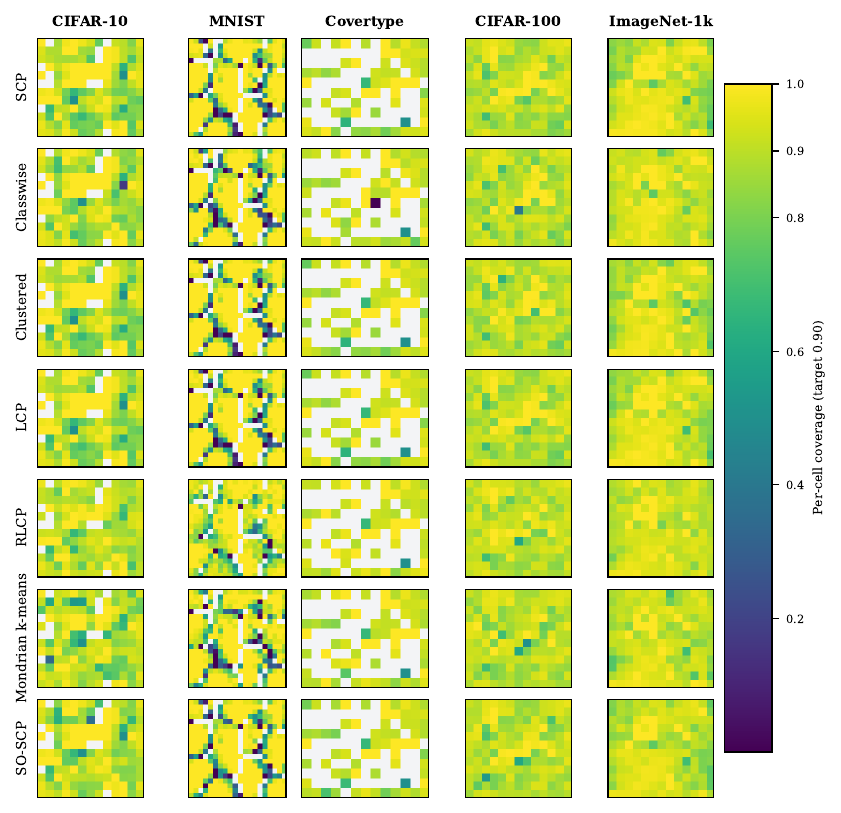}
    \caption{
        Classification coverage by cell with \gls{lac}, for seed~$42$.
        Rows show \gls{scp}, Classwise, Clustered, \gls{lcp}, \gls{rlcp}, Mondrian $k$-means, and SO-SCP.
        The shared scale, target marker, and empty-cell convention match Figure~\ref{fig:maps_regression_calibrators}.
    }
    \label{fig:maps_classification_calibrators}
\end{figure}

On MNIST, undercoverage follows the sparse bands in Figure~\ref{fig:cr_maps_classification_geometry} and the larger score percentiles in Figure~\ref{fig:maps_classification_predictor_score}.
Clustered and \gls{lcp} largely reproduce the \gls{scp} pattern.
SO-SCP raises coverage along parts of these bands, while \gls{rlcp} raises it more broadly.
The larger \gls{rlcp} changes must be considered alongside its output-size costs in Table~\ref{tab:cr_calibrator_tradeoff_classification}.
The goal is to reduce local coverage gaps with targeted corrections while keeping prediction sets informative.

\subsubsection{Conformal thresholds and output size}
\label{app:thresholds_output_size}

Now, we examine the cutoffs and output sizes behind the coverage patterns in Section~\ref{app:coverage_calibration_methods}.
Each map reports averages over the test queries assigned to each \gls{som} cell.
For cutoffs and interval widths, a single infinite value makes the cell average infinite, even if the other values are finite.

\paragraph{Regression.}
Figure~\ref{fig:maps_regression_cutoffs} compares the scalar cutoffs used with the absolute residual score.

\begin{figure}[ht]
    \centering
    \includegraphics[width=0.85\textwidth]{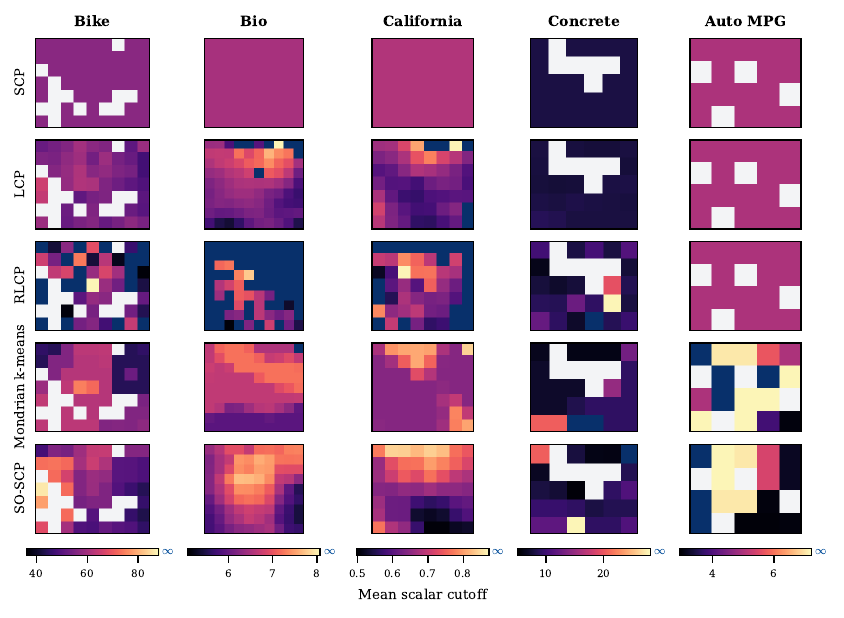}
    \caption{
        Mean regression cutoff by cell for seed~$42$, with rows matching Figure~\ref{fig:maps_regression_calibrators}.
        Each dataset shares one linear scale in score units across methods.
        Blue cells contain at least one infinite cutoff.
        Gray cells contain no test observations.
    }
    \label{fig:maps_regression_cutoffs}
\end{figure}

The \gls{scp} cutoff is constant, while SO-SCP varies it across cells.
On Bio and California Housing, higher SO-SCP cutoffs broadly follow the larger score percentiles in Figure~\ref{fig:maps_regression_predictor_score}.
Mondrian $k$-means produces more abrupt changes between its groups, while \gls{lcp} varies smoothly on the three large datasets.
On these three datasets, the \gls{rlcp} maps are irregular and interrupted by cells with infinite mean cutoffs.
On Auto MPG, \gls{lcp} and \gls{rlcp} use exactly the pooled cutoff at every query in this fit, so they provide no local adjustment.
Across ten seeds, neither improves WCovGap on Auto MPG in Table~\ref{tab:cr_calibrator_tradeoff_regression}: \gls{rlcp} matches \gls{scp}, and \gls{lcp} raises it by $0.17$ points.

For the absolute residual score, the interval $[\widehat\mu(x)-\hat q(x),\widehat\mu(x)+\hat q(x)]$ has width $2\hat q(x)$.
Figure~\ref{fig:maps_regression_size} shows the resulting interval widths and identifies unbounded outputs.

\begin{figure}[H]
    \centering
    \includegraphics[width=0.85\textwidth]{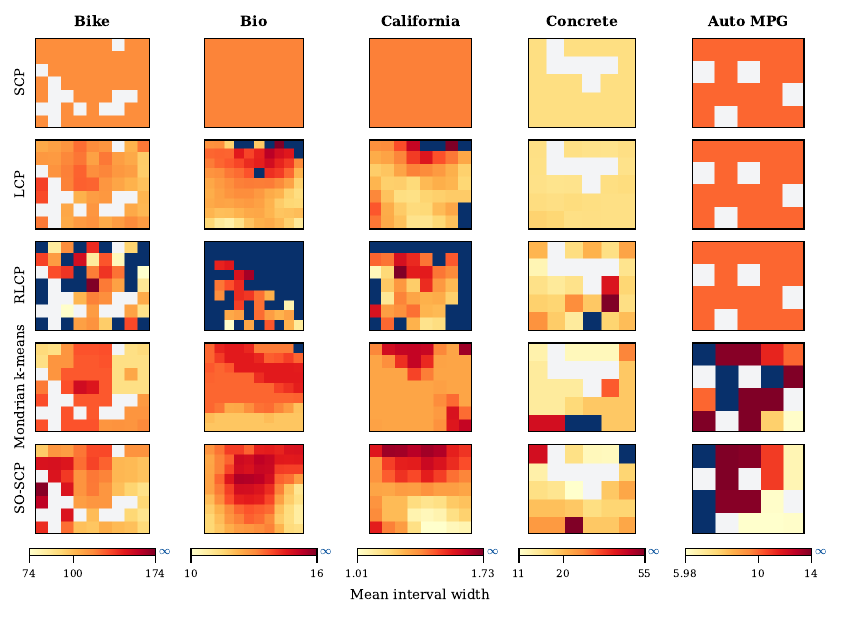}
    \caption{
        Mean regression interval width by cell for seed~$42$.
        Rows match Figure~\ref{fig:maps_regression_calibrators}.
        Each dataset shares one logarithmic color scale across methods, labeled in target units.
        Blue cells contain at least one unbounded interval.
        Gray cells contain no test observations.
    }
    \label{fig:maps_regression_size}
\end{figure}

SO-SCP keeps every interval finite on Bike Sharing, Bio, and California Housing.
By comparison, \gls{rlcp} returns unbounded intervals in $18$ of $50$, $64$ of $90$, and $21$ of $56$ occupied cells, respectively.
\gls{lcp} also produces unbounded intervals in $7$ Bio cells and $5$ California Housing cells.
On California Housing, $19$ of the $21$ affected \gls{rlcp} cells and all $5$ affected \gls{lcp} cells lie on the grid boundary.
This pattern is consistent with insufficient calibration weight for some queries in those regions.
SO-SCP retrieves hundreds of calibration scores even at the boundary, as established by Figure~\ref{fig:cr_buffers_regression}.

On Concrete and Auto MPG, SO-SCP's one and three cells with unbounded intervals are exactly the cells whose buffers fall below nine scores in Figure~\ref{fig:cr_buffers_regression}.
Mondrian $k$-means produces unbounded intervals in two and four cells, respectively.
Training-routed Regime~3 removes the observed infinite cutoffs in the calibration-size experiment of Figure~\ref{fig:cr_regime3}.

\paragraph{Classification.}
Figure~\ref{fig:maps_classification_cutoffs} compares methods with one scalar cutoff per query.
Classwise and Clustered are omitted because their cutoffs depend on the candidate label.

\begin{figure}[ht]
    \centering
    \includegraphics[width=0.96\textwidth]{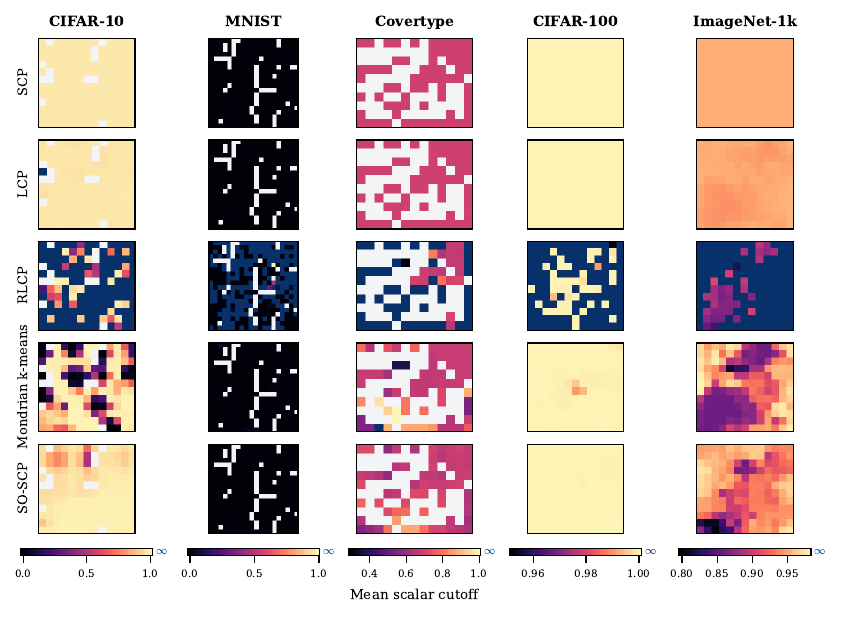}
    \caption{
        Mean scalar classification cutoff by cell for seed~$42$.
        Rows show \gls{scp}, \gls{lcp}, \gls{rlcp}, Mondrian $k$-means, and SO-SCP.
        Each dataset shares one linear scale across methods.
        Blue cells contain at least one infinite cutoff.
        Gray cells contain no test observations.
    }
    \label{fig:maps_classification_cutoffs}
\end{figure}

On ImageNet-1k, the higher SO-SCP cutoffs broadly follow the larger score percentiles in Figure~\ref{fig:maps_classification_predictor_score}.
\gls{lcp} remains close to the pooled cutoff on most datasets, consistent with its small gains in Table~\ref{tab:cr_calibrator_tradeoff_classification}.
For seed~$42$, $57\%$ to $78\%$ of occupied cells across the five datasets contain at least one infinite \gls{rlcp} cutoff, while SO-SCP has no infinite classification cutoffs.

With \gls{lac}, a label $c$ is retained when $1-\widehat p_c(x)\leq\hat q(x)$.
Classwise and Clustered replace $\hat q(x)$ with a cutoff indexed by $c$.
Consequently, set size depends on both the cutoffs and the full vector of predicted probabilities.
Figure~\ref{fig:maps_classification_size} reports these sizes for all seven methods.

\begin{figure}[ht]
    \centering
    \includegraphics[width=0.99\textwidth]{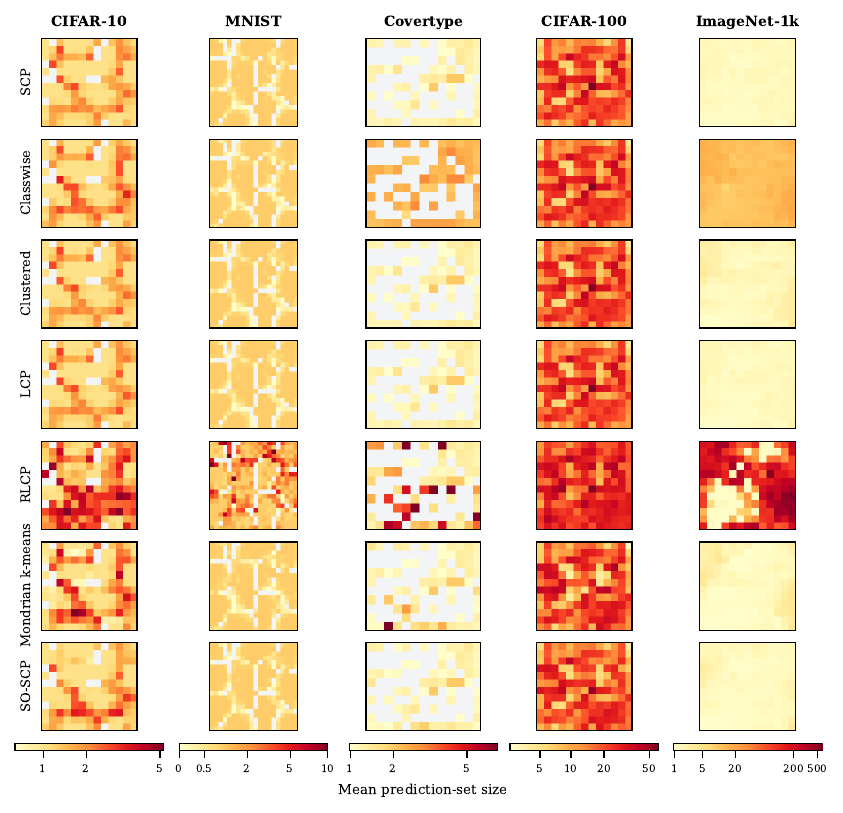}
    \caption{
        Mean classification set size by cell for seed~$42$.
        Rows match Figure~\ref{fig:maps_classification_calibrators}.
        Each dataset shares one logarithmic color scale across methods, with label counts on the colorbar.
        Gray cells contain no test observations.
    }
    \label{fig:maps_classification_size}
\end{figure}

An infinite scalar cutoff returns all labels, so it produces a finite but uninformative prediction set.
On ImageNet-1k, \gls{rlcp} returns all $1{,}000$ labels for $2{,}771$ of $20{,}000$ queries, raising its mean set size to $141.21$.
The corresponding means are $1.50$ for \gls{scp} and $1.55$ for SO-SCP.
Classwise sets are also enlarged throughout the map, with cell means between $7.3$ and $18.9$ labels.
For seed~$42$, SO-SCP increases the average prediction set size over test queries by at most $0.18$ labels relative to \gls{scp} on each of the five classification datasets.

Figures~\ref{fig:maps_regression_size} and~\ref{fig:maps_classification_size} show why coverage gains must be weighed against output size.
\gls{lcp} often remains close to pooled calibration, while \gls{rlcp}'s larger adjustments produce unbounded intervals or prediction sets containing every label for some queries.
SO-SCP combines local cutoff adjustments with mean output sizes close to \gls{scp} on the eight datasets with adequate Regime~2 buffers.

\clearpage
\subsubsection{Empirical \glsentryshort{ks} bridge diagnostics}
\label{app:ks_diagnostics}

For a fixed reference cell $k$ and neighborhood $A=N(k,r)$, Theorem~\ref{thm:regime2a} gives an exact coverage bound for the threshold computed from $A$, conditional on the query falling in $A$.
The cell-conditional bound in Theorem~\ref{thm:regime2b} pays the population \gls{ks} bias $\varepsilon_k(A)$ of Definition~\ref{def:ks-bias}, through Lemma~\ref{lem:bridge}.
We use calibration scores to examine the corresponding empirical discrepancies after calibration.
These diagnostics do not affect the cutoffs.

For every cell $k$, set $\widehat n_k\coloneqq|\mathcal I_k|$.
When $\widehat n_k>0$, define its empirical score \gls{cdf} by
\[
    \widehat F^{(k)}(t)
    \coloneqq
    \frac{1}{\widehat n_k}
    \sum_{i\in\mathcal I_k}\mathds{1}\{s_i\leq t\}.
\]
For cells $k$ and $j$ with calibration observations, let
\[
    \widehat d_{\mathrm{KS}}(k,j)
    \coloneqq
    \sup_{t\in\mathbb R}
    \left|\widehat F^{(k)}(t)-\widehat F^{(j)}(t)\right|.
\]
For $A=N(k,r)$ with $\widehat n_k>0$, the buffer size and pooled empirical \gls{cdf} are
\[
    \widehat n_A\coloneqq\sum_{j\in A}\widehat n_j,
    \qquad
    \widehat F^A(t)
    \coloneqq
    \sum_{j\in A:\widehat n_j>0}
    \frac{\widehat n_j}{\widehat n_A}\widehat F^{(j)}(t).
\]
Since $k\in A$, the denominator $\widehat n_A$ is positive.
We compare the weighted pairwise estimate $\widehat\varepsilon_k(A)$ with the distance $\widehat d_{\mathrm{mix}}(k,A)$ between the empirical score distributions in the central cell and its buffer:
\begin{align}
    \widehat\varepsilon_k(A)
     & \coloneqq
    \sum_{j\in A:\widehat n_j>0}
    \frac{\widehat n_j}{\widehat n_A}
    \widehat d_{\mathrm{KS}}(k,j),
    \\
    \widehat d_{\mathrm{mix}}(k,A)
     & \coloneqq
    \sup_{t\in\mathbb R}
    \left|\widehat F^{(k)}(t)-\widehat F^A(t)\right|.
\end{align}
The first quantity mirrors the population bridge term, while the second compares the central cell with the pooled buffer.
The triangle inequality gives $\widehat d_{\mathrm{mix}}(k,A)\leq\widehat\varepsilon_k(A)$.
Neither quantity is a confidence bound for the corresponding population discrepancy.

\paragraph{Regression.}
Figure~\ref{fig:maps_regression_ks} compares the two empirical distances using absolute residual scores.

\begin{figure}[ht]
    \centering
    \includegraphics[width=0.97\textwidth]{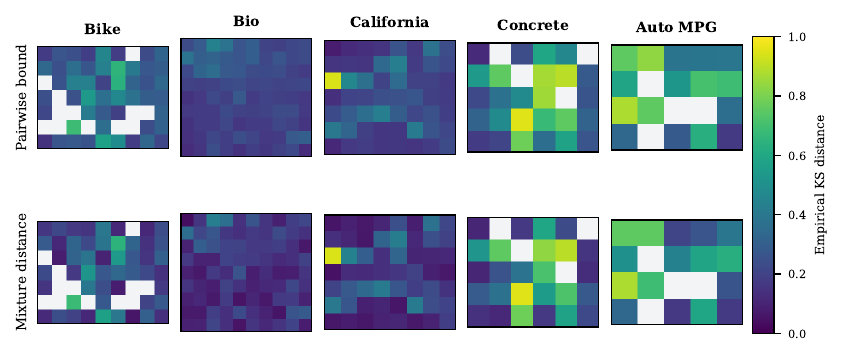}
    \caption{
        Regression \gls{ks} diagnostics from calibration scores for seed~$42$.
        Rows show the weighted pairwise estimate $\widehat\varepsilon_k(A)$ and the distance $\widehat d_{\mathrm{mix}}(k,A)$ on a shared $[0,1]$ scale.
        Gray cells contain no central-cell calibration observations.
    }
    \label{fig:maps_regression_ks}
\end{figure}

On Bio, every cell contains at least $17$ calibration scores.
Across these cells, the median weighted pairwise estimate $\widehat\varepsilon_k(A)$ is $0.209$, while the median distance $\widehat d_{\mathrm{mix}}(k,A)$ to the pooled buffer is $0.149$.
This empirical agreement complements the coverage improvements in Figure~\ref{fig:maps_regression_calibrators}.
The median of $\widehat d_{\mathrm{mix}}(k,A)$ is $0.137$ on California Housing and $0.235$ on Bike Sharing.

Large isolated values of $\widehat d_{\mathrm{mix}}(k,A)$ require more care.
California Housing's maximum of $0.945$ comes from a cell with one calibration score, although its retrieved buffer contains $984$.
The buffer supplies many scores for the cutoff, but the central-cell \gls{cdf} in this diagnostic still rests on one observation.
Concrete and Auto MPG have median $\widehat d_{\mathrm{mix}}(k,A)$ values of $0.330$ and $0.471$, with fewer than five central scores in $14$ of $26$ and $11$ of $16$ cells containing calibration observations.
Their empirical discrepancies can reflect both differences between score distributions and uncertainty from small samples.
Enlarging the retrieved buffer through Regime~3 does not increase the number of observations in the central cell, so it does not remove this source of diagnostic uncertainty.

\paragraph{Classification.}
Figure~\ref{fig:maps_classification_ks} reports the same diagnostics using \gls{lac} scores.

\begin{figure}[ht]
    \centering
    \includegraphics[width=0.99\textwidth]{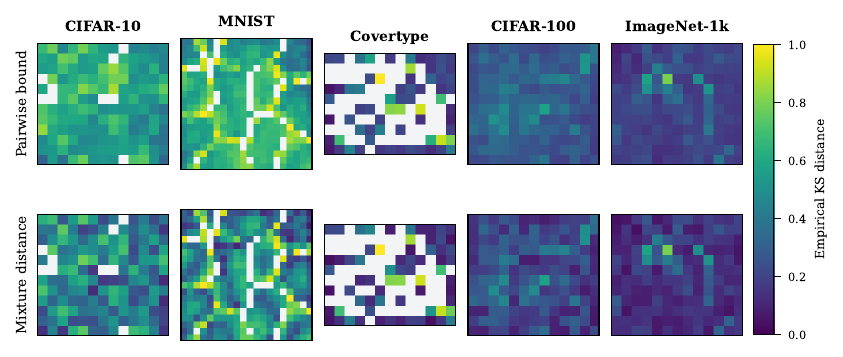}
    \caption{
        Classification \gls{ks} diagnostics from calibration scores for seed~$42$.
        Rows, scale, and empty-cell convention match Figure~\ref{fig:maps_regression_ks}.
    }
    \label{fig:maps_classification_ks}
\end{figure}

The lower row of Figure~\ref{fig:maps_classification_ks} reports $\widehat d_{\mathrm{mix}}(k,A)$, which compares the empirical calibration score distributions in a cell and its pooled buffer.
The percentile $\widehat q^{\mathrm{test}}_{0.9,k}$ in Figure~\ref{fig:maps_classification_predictor_score} instead describes the upper tail of test scores within the cell.
A high score percentile can coexist with a small $\widehat d_{\mathrm{mix}}(k,A)$ when the cell and its buffer have similar score distributions.

Across cells with calibration observations, the median of $\widehat d_{\mathrm{mix}}(k,A)$ is $0.425$ on CIFAR-10 and $0.475$ on MNIST.
The median calibration counts in these cells are $52$ and $31$, respectively.
On MNIST, among cells with both calibration and test observations, the median $\widehat d_{\mathrm{mix}}(k,A)$ is $0.637$ for cells with SO-SCP coverage below $0.90$ and $0.416$ for those at or above $0.90$.
Thus, on MNIST, SO-SCP undercoverage in Figure~\ref{fig:maps_classification_calibrators} is associated with larger empirical differences between the cell and buffer score distributions.

Covertype, CIFAR-100, and ImageNet-1k have lower median $\widehat d_{\mathrm{mix}}(k,A)$ values of $0.186$, $0.213$, and $0.139$.
CIFAR-100 combines test score percentiles close to one across most cells with smaller $\widehat d_{\mathrm{mix}}(k,A)$ values than CIFAR-10 and MNIST.
Therefore, its high scores coexist with closer agreement between the cell and buffer empirical distributions.
On Covertype, the five largest $\widehat d_{\mathrm{mix}}(k,A)$ values occur in cells with only one to four calibration scores.

On eight of the ten datasets, the cell with the largest $\widehat d_{\mathrm{mix}}(k,A)$ contains at most four calibration scores.
The corresponding cell contains $99$ scores on Bio and $63$ on CIFAR-10, the two exceptions.
Consequently, large $\widehat d_{\mathrm{mix}}(k,A)$ values should be interpreted alongside the calibration counts in Figures~\ref{fig:cr_maps_regression_geometry} and~\ref{fig:cr_maps_classification_geometry}, since estimates from sparse cells are sensitive to individual observations.

Buffer sizes indicate whether enough scores are available for a finite cutoff, while $\widehat d_{\mathrm{mix}}(k,A)$ measures agreement between the cell and buffer empirical score distributions.
A small empirical distance does not certify cell-conditional coverage, and a large one does not invalidate the exact guarantee for a fixed retrieved set.

\subsubsection{Cellwise coverage and output changes}
\label{app:cellwise_coverage_output_changes}

Figure~\ref{fig:cr_maps_effects} isolates SO-SCP's changes relative to \gls{scp} on the same test points and \gls{som} grid.
It pairs signed coverage differences with relative changes in mean output size.

\begin{figure}[ht]
    \centering
    \includegraphics[width=0.88\textwidth]{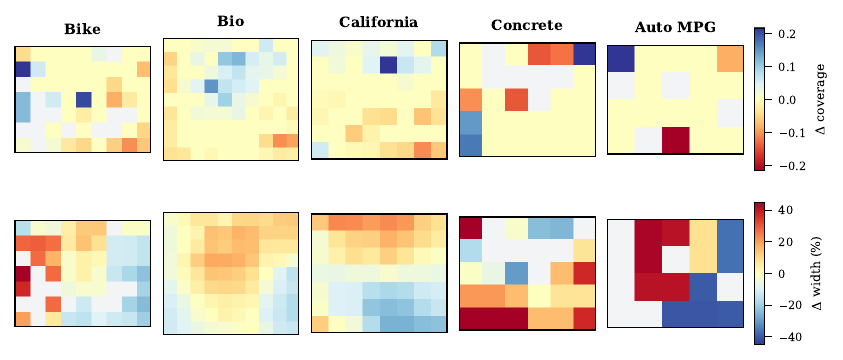}
    \par\smallskip
    \includegraphics[width=0.88\textwidth]{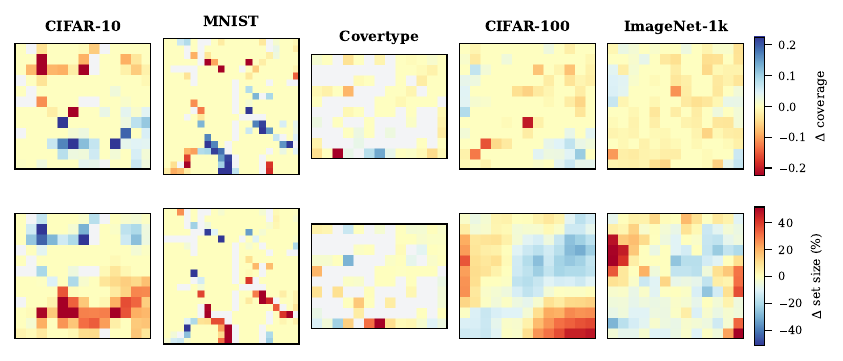}
    \caption{
        SO-SCP minus \gls{scp} cellwise effects for seed~$42$, with regression above classification.
        Each block shows coverage differences and relative changes in mean output size (\%).
        Blue indicates higher coverage or smaller outputs, red the reverse, and yellow little change.
        Gray denotes no test observations or an undefined relative change.
    }
    \label{fig:cr_maps_effects}
\end{figure}

A coverage increase can correct undercoverage or deepen overcoverage, depending on the cell's initial level.
The coverage rows should therefore be read against the $0.90$ target in Figures~\ref{fig:maps_regression_calibrators} and~\ref{fig:maps_classification_calibrators}.
The output-size rows show that SO-SCP enlarges predictions in some cells and shrinks them in others.

\paragraph{Regression.}
On the displayed \gls{som} partition, WCovGap decreases by $0.59$ points on Bio and $0.61$ points on California Housing, while Bike Sharing is unchanged.
These improvements accompany the finite, spatially varying widths in Figure~\ref{fig:maps_regression_size}.
WCovGap decreases by $1.94$ points on Concrete and increases by $1.00$ point on Auto MPG under the sparse Regime~2 buffers.
Both have unbounded intervals in this fit, although fewer cells are affected than under Mondrian $k$-means on each dataset (Section~\ref{app:thresholds_output_size}).

\paragraph{Classification.}
SO-SCP lowers WCovGap on four of the five datasets in this fit, with a reduction of $0.96$ points on ImageNet-1k.
Its local increases and decreases in set size retain the small average outputs documented in Figure~\ref{fig:maps_classification_size}.
CIFAR-10 is the exception, with WCovGap increasing by $0.18$ points.
Across ten seeds, its shared-\gls{som} WCovGap instead falls by $2.4\%$ (Table~\ref{tab:cr_transfer_classification}), so the displayed fit differs from the average result.

Across the eight datasets with finite Regime~2 outputs, the ten-seed results in Tables~\ref{tab:cr_transfer_regression} and~\ref{tab:cr_transfer_classification} give a mean WCovGap reduction of $10.6\%$ on the shared \gls{som}, for a $2.0\%$ mean output-size increase.
Those tables also show improvements on separate external partitions, connecting the spatial adjustments to gains beyond the displayed audit grid.

\subsubsection{Worked example on Covertype}
\label{app:worked_example_covertype}

Covertype illustrates how the changes in Figure~\ref{fig:cr_maps_effects} reduce coverage error on both sides of the target.
Figure~\ref{fig:cr_worked_example} compares the two coverage maps with the corresponding changes in prediction set size.

\begin{figure}[ht]
    \centering
    \includegraphics[width=\textwidth]{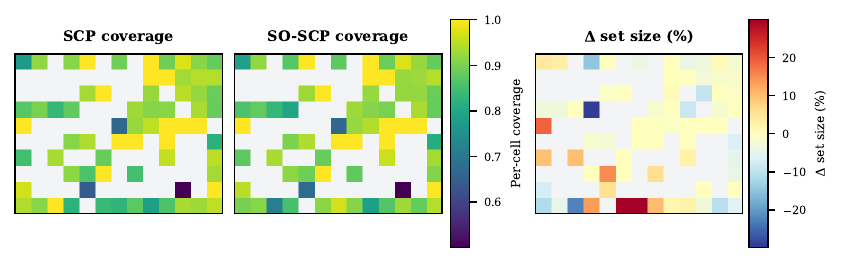}
    \caption{
        Covertype with \gls{lac} and Regime~2, for seed~$42$.
        Panels show \gls{scp} coverage, SO-SCP coverage, and the relative change in mean set size.
        Both coverage panels share one color scale.
        Set-size colors use the colormap of Figure~\ref{fig:cr_maps_effects}, and gray cells follow its convention.
    }
    \label{fig:cr_worked_example}
\end{figure}

On this \gls{som} partition, WCovGap falls from $3.53$ to $3.04$ points.
This reduction of $0.49$ points comprises $0.18$ points from less undercoverage and $0.31$ points from less overcoverage.
As a result, SO-SCP brings coverage closer to the target through adjustments in both directions.

WCovGap weights each cell's absolute coverage error by its share of test observations.
Table~\ref{tab:cr_worked_example} selects the three largest contributions to the reduction among cells initially below the target and the three largest among cells initially above it.
Coordinates are given as row and column, starting at the bottom left of the $10\times13$ grid.

\begin{table}[ht]
    \centering
    \caption{
        Covertype cells contributing most to the WCovGap reduction in Figure~\ref{fig:cr_worked_example}, with three selected on each side of the $0.90$ target under \gls{scp}.
        Support counts test observations.
        Coverage and mean set size are shown for both methods, with relative changes.
    }
    \label{tab:cr_worked_example}
    \begin{tabular}{lcccccccc}
\toprule
 & \multicolumn{2}{c}{} & \multicolumn{3}{c}{Coverage} & \multicolumn{3}{c}{Mean set size} \\
 & SCP side & Support & SCP & SO-SCP & $\Delta$ (\%) & SCP & SO-SCP & $\Delta$ (\%) \\
\midrule
$(0, 5)$ & $<0.90$ & $232$ & $0.836$ & $0.914$ & $+9.3\%$ & $1.237$ & $1.612$ & $+30.3\%$ \\
$(3, 0)$ & $<0.90$ & $263$ & $0.848$ & $0.875$ & $+3.2\%$ & $1.122$ & $1.228$ & $+9.4\%$ \\
$(2, 7)$ & $<0.90$ & $153$ & $0.850$ & $0.876$ & $+3.1\%$ & $1.105$ & $1.170$ & $+5.9\%$ \\
$(0, 0)$ & $>0.90$ & $244$ & $0.943$ & $0.898$ & $-4.8\%$ & $1.193$ & $1.061$ & $-11.1\%$ \\
$(2, 12)$ & $>0.90$ & $150$ & $0.933$ & $0.907$ & $-2.8\%$ & $1.160$ & $1.113$ & $-4.1\%$ \\
$(0, 10)$ & $>0.90$ & $414$ & $0.935$ & $0.928$ & $-0.7\%$ & $1.309$ & $1.271$ & $-2.9\%$ \\
\bottomrule
\end{tabular}

\end{table}

In cell $(0,5)$, coverage rises from $0.836$ to $0.914$ as mean set size increases from $1.237$ to $1.612$ labels.
In cell $(0,0)$, coverage moves from $0.943$ to $0.898$ as mean set size decreases from $1.193$ to $1.061$ labels.
Both adjustments bring coverage closer to $0.90$.
Across the full test set, marginal coverage changes from $0.902$ to $0.901$, and mean set size from $1.227$ to $1.228$.
The local enlargements and reductions nearly balance, so WCovGap falls on this map with little change in average set size.

The ten-seed external $K=25$ comparison in Table~\ref{tab:cr_calibrator_tradeoff_classification} gives SO-SCP a $0.84$-point WCovGap reduction on Covertype for $1.7\%$ larger sets.
Mondrian $k$-means and \gls{rlcp} gain $1.13$ and $1.21$ points, but increase set size by $4.2\%$ and $36.7\%$, respectively.
SO-SCP captures most of their coverage-gap reduction at a smaller output cost.
The spatial diagnostics show how it achieves this trade-off: cutoffs vary across cells, allowing coverage improvements through both larger and smaller predictions.

\end{document}